\documentclass{article}

\usepackage{microtype}
\usepackage{graphicx}
\usepackage[table,xcdraw]{xcolor}
\usepackage{booktabs} 

\usepackage{hyperref}

\usepackage[ruled,linesnumbered,vlined]{algorithm2e}

\usepackage[accepted]{icml2024}

\usepackage{amsmath}
\usepackage{amssymb}
\usepackage{mathtools}
\usepackage{amsthm}

\usepackage{kbordermatrix}

\usepackage{multirow}
\usepackage{dsfont}
\usepackage{amsmath}
\usepackage{amsthm}
\usepackage{amsfonts}
\usepackage{multicol}
\usepackage{breqn}
\usepackage[ruled,linesnumbered,vlined]{algorithm2e}

\usepackage{amsmath,amsfonts,bm}
  
\newcommand{\RNum}[1]{\uppercase\expandafter{\romannumeral #1\relax}}

\newcommand{\N}{\mathcal{N}}

\newcommand\restr[2]{{
  \left.\kern-\nulldelimiterspace 
  #1 
  \vphantom{\big|} 
  \right|_{#2} 
  }}

\DeclareMathOperator*{\argmin}{\arg\!\min}

\DeclareMathOperator*{\argmax}{\arg\!\max}

\newcommand{\cut}[1]{}

\newcommand{\removelatexerror}{\let\@latex@error\@gobble}

\def\1{\bm{1}}

\DeclareMathAlphabet{\mathsfit}{\encodingdefault}{\sfdefault}{m}{sl}
\SetMathAlphabet{\mathsfit}{bold}{\encodingdefault}{\sfdefault}{bx}{n}

\theoremstyle{definition}
\newtheorem{definition}{Definition}[section]

\usepackage{caption}
\usepackage{subcaption}
\usepackage{xcolor}
\usepackage{soul}
\usepackage{float}
\usepackage{tikz}

\newcommand{\CBP}{{\fontfamily{qcr}\selectfont{CBP}}}
\newcommand{\RandCBP}{{\fontfamily{qcr}\selectfont{RandCBP}}}
\newcommand{\TSPM}{{\fontfamily{qcr}\selectfont{TSPM}}}
\newcommand{\PMDMED}{{\fontfamily{qcr}\selectfont{PMDMED}}}
\newcommand{\BPMLeast}{{\fontfamily{qcr}\selectfont{BPM-Least}}}

\newcommand{\CBPside}{{\fontfamily{qcr}\selectfont{CBPside}}}
\newcommand{\CBPsidestar}{{\fontfamily{qcr}\selectfont{CBPside$^\star$}}}
\newcommand{\RandCBPside}{{\fontfamily{qcr}\selectfont{RandCBPside}}}
\newcommand{\RandCBPsidestar}{{\fontfamily{qcr}\selectfont{RandCBPside$^\star$}}}
\newcommand{\PGIDS}{{\fontfamily{qcr}\selectfont{PG-IDS}}}
\newcommand{\PGTS}{{\fontfamily{qcr}\selectfont{PG-TS}}}
\newcommand{\TSPMgaussian}{{\fontfamily{qcr}\selectfont{TSPM-Gaussian}}}

\newcommand{\STAP}{{\fontfamily{qcr}\selectfont{STAP}}}
\newcommand{\CESA}{{\fontfamily{qcr}\selectfont{CESA}}}
\newcommand{\IDSFW}{{\fontfamily{qcr}\selectfont{IDS-FW}}}

\newcommand{\explorecommit}{{\fontfamily{qcr}\selectfont{Explore-fully}}}
\newcommand{\crandcbp}{{\fontfamily{qcr}\selectfont{C-RandCBP}}}
\newcommand{\ccbp}{{\fontfamily{qcr}\selectfont{C-CBP}}}

\usepackage[capitalize,noabbrev]{cleveref}

\newtheorem{theorem}{Theorem}[section]
\newtheorem{lemma}[theorem]{Lemma}
\newtheorem{corollary}[theorem]{Corollary}
\newtheorem{remark}[theorem]{Remark}

\usepackage[textsize=tiny]{todonotes}

\icmltitlerunning{Randomized Confidence Bounds for Stochastic Partial Monitoring}

\begin{document}

\twocolumn[
\icmltitle{Randomized Confidence Bounds for Stochastic Partial Monitoring}

\begin{icmlauthorlist}
\icmlauthor{Maxime Heuillet}{ul,comp,mila,iid}
\icmlauthor{Ola Ahmad}{ul,comp}
\icmlauthor{Audrey Durand}{ul,mila,iid,cifar}
\end{icmlauthorlist}

\icmlaffiliation{cifar}{Canada CIFAR AI Chair}
\icmlaffiliation{ul}{Université Laval}
\icmlaffiliation{mila}{Mila - Québec AI Institute}
\icmlaffiliation{comp}{Thales Digital Solutions (cortAIx lab)}
\icmlaffiliation{iid}{Institut Intelligence et Données}

\icmlcorrespondingauthor{Maxime Heuillet}{maxime.heuillet.1@ulaval.ca}


\vskip 0.3in
]

\printAffiliationsAndNotice{\icmlEqualContribution} 

\begin{abstract}
The partial monitoring (PM) framework provides a theoretical formulation of sequential learning problems with incomplete feedback. 
On each round, a learning agent plays an action while the environment simultaneously chooses an outcome. 
The agent then observes a feedback signal that is only partially informative about the (unobserved) outcome. 
The agent leverages the received feedback signals to select actions that minimize the (unobserved) cumulative loss. 
In contextual PM, the outcomes depend on some side information that is observable by the agent before selecting the action on each round.
In this paper, we consider the contextual and non-contextual PM settings with stochastic outcomes.
We introduce a new class of strategies based on the randomization of deterministic confidence bounds, that extend regret guarantees to settings where existing stochastic strategies are not applicable.
Our experiments show that the proposed \RandCBP{} and \RandCBPsidestar{} strategies improve state-of-the-art baselines in PM games.
To encourage the adoption of the PM framework, we design a use case on the real-world problem of monitoring the error rate of any deployed classification system.
\end{abstract}

\section{Introduction}

Partial monitoring~\cite{bartok2014partial} is a framework tailored for online learning problems with partially informative feedback. 
A partial monitoring (PM) game is played between a learning agent and the environment over multiple rounds. 
At a given round, the agent selects an \textit{action} and the environment simultaneously selects an \textit{outcome}.
The agent then incurs an instant \textit{loss} and receives a \textit{feedback} signal that is partially informative about the outcome. 
The challenge is that the agent does not observe the loss. 
Nonetheless, its goal is to minimize the (unobserved) cumulative loss by carefully balancing between actions associated to informative feedback signals and small-loss actions, which captures the core exploration-exploitation trade-off.

The agent's performance is measured by the \textit{regret}, which corresponds to the excess loss associated with the selected action compared to the best action in hindsight. 
The cumulative regret scales linearly with the horizon $T$ if the agent fails to identify the best action.
In this work, we consider the \textit{stochastic setting} where outcomes are independent and identically distributed (i.i.d.) according to some (unknown) outcome distribution.
In this setting, \citet{bartok2011minimax} classified PM games into four categories based on achievable bounds on the cumulative regret: trivial games (no regret); \textit{easy} games with poly-logarithmic upper bounds in
$\tilde \Theta(\sqrt{T})$ ; \textit{hard} games with upper bounds in $\Theta(T^{2/3})$; and intractable games with lower bounds in $\Omega(T)$.
The well-known multi-armed bandit problem~\citep{auer2002nonstochastic} corresponds to an easy game. 
Additionally, many problems correspond to hard games, such as learning from costly expert advice~\citep{helmbold1997some}, dynamic pricing~\citep{kleinberg2003value}, and online monitoring~\citep{ginart2022mldemon}.

Deterministic PM strategies such as \CBP{}~\citep{bartokICML2012} and \PMDMED{}~\citep{komiyama2015regret} have sub-linear regret guarantees on both easy and hard games. However, these are consistently outperformed empirically by stochastic strategies like \BPMLeast{}~\citep{vanchinathan2014efficient} and \TSPM{}~\citep{tsuchiya2020analysis}, for which regret guarantees are unfortunately limited to easy games.

The \textit{contextual} PM setting is an extension where the outcome distribution is a function of some \textit{side information} (a \textit{context}) observed by the agent before selecting the action on each round.
Existing contextual PM strategies are fairly restrictive. The deterministic \CBPside{}~\citep{bartok2012CBPside}, a contextual extension of \CBP{}, is not applicable to hard games, which capture a valuable diversity of applications (see examples above). 
On the other hand, the stochastic \IDSFW{}~\citep{kirschner2023linear} has regret guarantees on both easy and hard games at the price of several drawbacks: it scales quadratically with the number of rounds; and it requires the set of contexts to be finite and known in advance, a restriction that often does not hold in practice.

The primary aim of this study is to bridge the gap between the theoretical regret guarantees of \CBP{}-based strategies and their empirical performance, which is dominated by stochastic approaches.
\citet{kveton2019garbage} and \citet{vaswani2019old} show that the confidence bounds used in ``optimistic in the face of uncertainty'' (OFU) strategies can be randomized to improve empirical performance, while maintaining theoretical guarantees. 
We therefore raise the following question:
\emph{\textbf{Can the randomization of confidence bounds also benefit non OFU-based strategies?}} 

\vspace{-1em}
\paragraph{Contributions} 
1) We focus on \CBP{}-based strategies in the PM setting, which instantiate a successive elimination exploration strategy. We show that it is possible to randomize \CBP{}-based strategies, and obtain sub-linear regret guarantees for the resulting randomized strategies.
2) In addition, we investigate the mechanics preventing the \CBPside{} applicability in hard games. The proposed \CBPsidestar{} successfully extends \CBP{} to hard contextual PM games. 
3) Our experiments show that the proposed randomized variants, namely \RandCBP{} and \RandCBPsidestar{}, improve state-of-the-art baselines in hard and easy PM games.
4) Currently, the PM field is predominantly theoretical and there is a notable scarcity of PM applications~\citep{singla2014contextual, kirschner2023linear}. To illustrate how the PM framework can benefit real world applications, we design a new use-case based on the real-world application of monitoring the error rate of any deployed classification system. 

\vspace{-0.5em}
\section{Preliminaries on Partial Monitoring}
\label{sec:background}

We consider finite PM games defined by $N$ actions available to the agent and $M$ outcomes available to the environment.
A game is characterized by a loss matrix $\textbf{L} \in [0,1]^{N \times M} $ and a feedback matrix $ \textbf{H} \in \Sigma^{N \times M}$. 
The feedback space $\Sigma$ is finite, arbitrary, and not necessarily numeric (e.g., it could be symbols). For convenience, we assume that the difference between greatest and lowest elements in the loss matrix is bounded by $1$, i.e. $\max( \textbf{L} ) - \min( \textbf{L} ) \leq 1$.

\vspace{-0.5em}
\subsection{Finite stochastic partial monitoring games}

A finite PM game is played over $T$ rounds between a learning agent and the environment.
The horizon $T$ is unknown to both the agent and the environment.
The matrices $\textbf{L}$ and $\textbf{H}$ are known. 
At each round $t=1, 2, \dots, T$, the environment samples an outcome $J_t$ from a distribution $p^\star \in \Delta_M$, where $\Delta_M$ is the probability simplex of dimension $M$ (column vector).
We refer to $p^\star$ as the \textit{outcome distribution} and
the outcomes are independent and identically distributed (i.i.d.) according to $p^\star$.
The agent then plays an action $I_t$. 
Following this, the agent observes a feedback $\textbf{H}[I_t, J_t]$ and incurs a deterministic loss $\textbf{L}[I_t, J_t]$, where $[i,j]$ denotes the element at row $i$ and column $j$.
We emphasize that the loss and the outcome are never revealed to the agent.

\paragraph{Non-contextual setting} 

The expected loss of action $i$ is noted $\ell_i =  L_i p^\star $, where the notation $L_i$ corresponds to the $i$-th row of matrix $\textbf{L}$.
The optimal action is given by $i^\star = \argmin_{1 \leq i \leq N} \ell_i$
and the performance of the agent can be evaluated using the cumulative regret (to minimize):
\vspace{-1em}
\begin{equation}
    R(T) = \sum_{t=1}^{T} (L_{I_t} - L_{i^\star}) p^\star .
\label{eq:noncontextual_regret}
\vspace{-0.5em}
\end{equation}

\paragraph{Contextual setting}
In the non-contextual setting, the optimal action does not depend on side information (context).
In the contextual setting~\citep{bartok2012CBPside}, also known as PM with side information, the optimal action depends on side information (context).
Let $p^\star(x)$ denote the outcome distribution as a function of a context $x \in \mathcal X$, with $\mathcal X$ denoting the unknown and possibly continuous context space. 
The optimal action in context $x$ minimizes the expected loss in that context: $i^\star(x) = \argmin_{1 \leq i \leq N} L_i  p^\star(x) $.
The agent aims to minimize the cumulative contextual regret:
\vspace{-0.5em}
\begin{equation}
R(T) = \sum_{t=1}^T ( L_{I_t}  - L_{i^\star(x_t) } ) p^\star(x_t) ,
\label{eq:contextual_regret}
\end{equation}
where $L_{ i^\star(x_t) }$ is the loss vector of the optimal action in $x_t$.

\vspace{-0.5em}
\paragraph{Other relevant settings}
We focus on stochastic settings where the outcome distribution is fixed over time, which differs from \citet{piccolboni2001discrete,cesa2006prediction,lattimore2019exploration,lattimore2022rustichini,tsuchiya2022bestofbothwords} who assume the outcome distribution may change over time. We also assume finite actions and feedback spaces, unlike
\citet{kirschner2020information} who focus on settings with continuous action and feedback spaces. Finally, we consider a contextual setting where the context space $\mathcal X$ is unknown and can be continuous, whereas
\citet{kirschner2023linear} assume that $\mathcal X$ is finite and known in advance. 

\vspace{-0.5em}
\subsection{Structure of partial monitoring games}

The optimality and informativeness of actions are respectively defined by loss matrix $\textbf{L}$ and feedback matrix $\textbf{H}$.
\begin{definition}[Cell decomposition, \citet{bartokICML2012}]
    The \textit{cell} $\mathcal C_i $ of action $i$ is defined as the subspace in the probability simplex $\Delta_M$  where action $i$ is optimal. Formally, 
    $\mathcal C_i =  \{ p \in \Delta_M, j \in \{1,...,N\}, (L_i-L_j) p \rangle \leq 0 \}.$
    \label{def:cell}
\end{definition}

Based on the cell, one can tell that an action $i$ is: (i) \textit{dominated} if $\mathcal C_i = \emptyset$ (i.e. there is no outcome distribution s.t. the action would be optimal); (ii) \textit{degenerate} if the action is not dominated and there exist action $i'$ such that $\mathcal C_i \subsetneq \mathcal C_{i'}$ (i.e. actions $i$ and $i'$ are duplicates, and therefore, both are jointly optimal under some outcome distributions); (iii) \textit{Pareto-optimal} if the action is neither dominated nor degenerate. The set of Pareto-optimal actions is denoted $\mathcal P$. 

Let $\sigma_i$ denote the number of unique feedback symbols on row $i$ of $\textbf{H}$.
Let $s_1, ..., s_{\sigma_i} \in \Sigma$ be an enumeration of the unique feedback symbols induced by action $i$ (i.e. symbols in row $H_i$), sorted by order of appearance (columns) in $H_i$. 
\begin{definition}[Signal matrix, \citet{bartokICML2012}]
    Given action $i$, the elements of \textit{signal matrix} $S_i \in \{0,1\}^{\sigma_i \times M}$ are defined as $S_i[u,v] = \mathds{1}_{ \{ \textbf{H}[i,v] = s_{u}  \} } $.
\end{definition}
The signal matrix $S_i$ is binary and it can be thought as a one-hot encoding of the unique feedback symbols induced by action $i$.  

The signal matrices verify the important relation $\pi^\star_i = S_i\, p^\star  \in \Delta_{\sigma_i}$, where $\pi^\star_i$ (respectively $\pi^\star_i(x)$ in the contextual setting) is the distribution over the feedback symbols induced by action $i$.
\vspace{-0.5em}
\paragraph{Difference between easy and hard games} A PM game is \textit{easy} if it suffices to play Pareto-optimal actions to minimize the regret. In hard games, minimizing the regret requires to play all actions in the game (including dominated and degenerate actions). Formal definitions of easy and hard games are provided in Appendix \ref{appendix:games}.

\section{Towards a randomized CBP}

\CBP{}~\citep{bartokICML2012} currently stands out as the only strategy offering regret guarantees in both easy and hard games for non-contextual PM, and a practical extension in the contextual setting.
In terms of empirical performance, \CBP{} is outperformed by stochastic PM strategies. 
Similar limitations have been identified in the bandits setting for deterministic Upper Confidence Bound (UCB) strategies~\citep{chapelle2011empirical}. 
Randomizing UCB-based strategies has proven to be helpful~\citep{vaswani2019old, kveton2019garbage} for improving empirical performance while preserving the theoretical analysis.
Here, we extend these ideas to the class of successive elimination strategies~\citep{even2002pac} to which CBP belongs.
Algorithm~\ref{alg:RandCBP} jointly displays the pseudo-codes of \CBP{} and the proposed \RandCBP{}. Differences are highlighted in \textcolor{purple}{purple}. 
Implementation details are reported in Appendix~\ref{appendix:details}.

\begin{algorithm}[t]
\label{CBP}
\SetKwInOut{Input}{input}
\Input{ $\mathcal P, \mathcal N, \alpha, \eta_a, f(\cdot), \color{purple}{K, \sigma, \varepsilon}  $ }
\DontPrintSemicolon
\caption{\CBP{} \citep{bartokICML2012} and \texttt{\textcolor{purple}{RandCPB}} }
\label{alg:RandCBP}

$\#$ Notation $e(\cdot)$ is a $\sigma_{I_t}$-dimensional one-hot encoding. \;

\For{$t = 1, 2, \dots, N$} { 
     Play action $I_t = t$ (play each action once)   \; 
     
     Observe feedback $\textbf{H}[I_t, J_t]$  \; 
     
     Init $n_{I_t}(N) = 1$, $\nu_{I_t}(N) = e(\textbf{H}[I_t, J_t])$
}
\For{$t > N$} {
    Initialize $\mathcal U(t) \gets \{\}$
    
    \For{ each neighbor pair $\{i,j\} \in \mathcal N$} {
    
     $\hat \delta_{ij}(t) \gets \sum_{a \in V_{ij}} v_{ija}^\top \frac{\nu_a(t-1) }{n_a(t-1) } $  \;
     
    \color{purple}{ $B \gets  \sqrt{  \alpha log(t)   } $ } \;
    
    \textcolor{purple}{Sample $Z_{ijt}$ with Algorithm \ref{alg:randomization_procedure} (Appendix \ref{appendix:pseudo_code})}  \;
    
    $\color{purple}{ c'_{ij}(t) \gets  \sum_{a \in V_{ij}} \| v_{ija} \|_{\infty} Z_{ijt} \sqrt{  \frac{ 1 }{ n_a } } }$ \;
    
    \uIf{$| \hat \delta_{ij}(t) | >$ \st{ $c_{ij}(t)$ }  $\color{purple}{ c'_{ij}(t) } $ }{
         Add $\{i,j\}$ to $\mathcal U(t)$ \; }
         
    }
    Compute $D(t)$ based on $\mathcal U(t)$ \;
    
    Get $\mathcal P(t)$ and $\mathcal N(t)$ given $\mathcal P$, $\mathcal N$ and $D(t)$ \;
    
    $\mathcal N^{+}(t) \gets \bigcup_{ {ij} \in \mathcal N(t) }  N^{+}_{ij} $ \;
    
    $\mathcal V(t) \gets \bigcup_{ {ij} \in \mathcal N(t) } V_{ij} $ \;
    
    $\mathcal R(t) \gets \{ a \in N : n_a(t-1) \leq \eta_a f(t) \}$   \;
    
    $\mathcal S(t) \gets \mathcal P(t) \cup \mathcal N^{+}(t) \cup ( \mathcal V(t)  \cap \mathcal R(t)  ) $ \;
    
    Select action $I_t = \argmax_{a\in \mathcal S(t)} \frac{W_{a}^2}{n_a(t-1)}$ \;
    
    Observe feedback $\textbf{H}[I_t, J_t]$  \; 
    
    $n_i(t) \leftarrow n_i(t-1) + \mathbb I[i = I_t], \forall i$ \; 
    
    $\nu_i(t) \leftarrow \nu_i(t-1) + \mathbb I[i = I_t] e(\boldsymbol H[I_t, J_t]), \forall i$  \;  } 
    
\end{algorithm}

\subsection{The \CBP{} strategy}
\label{sec:CBP}
 
Recall that the unknown parameter of the game is the outcome distribution $p^\star \in \Delta_M$.
The expected loss difference between two actions $i$ and $j$ is defined as
\begin{align}
    \delta_{i,j} =& (L_i - L_j) p^\star  \label{eq:loss_diff_p_star} = \ell_i - \ell_j.
\end{align}
The sign of the expected loss indicates which action is better: action $j$ is better than action $i$ when  $\delta_{i,j}>0$.

\begin{definition}[Neighbor pairs, \citet{bartokICML2012}]
Two Pareto-optimal actions $i$ and $j$ are \textit{neighbors} if $\mathcal C_i \cap \mathcal C_j$ is an ($M-2$)-dimensional polytope. 
The set of all neighbor pairs is denoted $\mathcal N$.  
\label{def:neighbor}
\end{definition}
\vspace{-0.5em}
Two actions are neighbors when they cannot be jointly optimal under a given outcome distribution. 
\CBP{} exploits that it suffices to compute $\delta_{i,j}$ for the pairs in $\mathcal N$, instead of computing $\delta_{i,j}$ for all the action pairs in the game.

\vspace{-0.5em}
\paragraph{Successive elimination confidence bounds }
\CBP{} computes expected loss difference estimates for all action pairs in $\mathcal N$.
Pairs with low confidence estimates are then eliminated based on a successive elimination~\cite{even2002pac} criteria. Consider Eq. \ref{eq:loss_diff_p_star}: any estimate $\hat \ell_i(t)$ of the expected loss of action $i$ at round $t$ admits an upper confidence bound, denoted $\text{UCB}_i(t) = \hat \ell_i(t) + c_i(t)$ and a lower confidence bound, denoted $\text{LCB}_i(t) = \hat \ell_i(t) - c_i(t)$, where $c_i(t)$ is a confidence width that holds with some probability. 
One can tell with confidence that action $j$ has a lower expected loss (i.e. is confidently better) than action $i$ if $\text{UCB}_j(t)$ is strictly lower than $\text{LCB}_i(t)$:
\begin{align}
    \text{UCB}_j(t) < \text{LCB}_i(t) & \Leftrightarrow \hat \ell_j(t) + c_j(t) < \hat \ell_i(t) - c_i(t) \nonumber \\
                                    & \Leftrightarrow |\hat \delta_{ij}(t)| > c_{i,j}(t),
                                    \label{eqn:ucb_lcb}
\end{align}
where $\hat \delta_{i,j}(t) = \hat \ell_i(t) - \hat \ell_j(t)$ and $c_{i,j}(t) = c_i(t) + c_j(t)$.

At each round, action pairs $\{i,j\}$ that do not satisfy the elimination criteria of Eq. \ref{eqn:ucb_lcb} correspond to low confidence estimates that are eliminated by \CBP{}.
Action pairs that satisfy the criteria are added to a set of high confidence pairs, denoted $\mathcal U(t)$.
The set $\mathcal U(t)$ is then used to compute a sub-space of the probability simplex, denoted $D(t)$, that summarizes the current knowledge of \CBP{} about the outcome distribution $p^\star$:
$  D(t) = \{ p \in \Delta_M, (i,j) \in \mathcal U(t), 
        \texttt{sign}(\hat \delta_{i,j}) (L_i-L_j) p > 0 \}. $
The sub-space $D(t)$ gathers constraints based on signs of confident estimates, which inform on the relative quality of actions.

Unfortunately, one cannot empirically estimate the losses $\ell_i$ and $\ell_j$ to compute $\hat \delta_{i,j}(t)$ and $c_{i,j}(t)$. Indeed, that would require to estimate the outcome distribution $p^\star$, but the agent never observes the outcomes. To address this challenge, \CBP{} exploits a connection between the outcome and feedback distributions.

\vspace{-0.5em}
\paragraph{ Connecting outcome and feedback distributions }
\label{par:feedback_dist}

Computing $\hat \delta_{i,j}$ and $c_{i,j}$ in practice requires two definitions.
\begin{definition}[Observer set, \citet{bartokICML2012}]
The set $V_{i,j}$ associated with action $i$ and $j$ contains the actions required to verify the relation
$(L_i - L_j)^\top \in \oplus_{a \in V_{i,j}} \text{Im} (S_a^\top$), where $\oplus$ corresponds to the direct sum.
\label{def:observer_set}
\end{definition}

\begin{definition}[Observer vectors, \citet{bartokICML2012}]
The observer vector of the action pair $\{i,j\}$ with respect to action $a$ in observer set $V_{i,j}$, denoted $v_{ija} \in \mathbb R^{\sigma_a}$, is selected to satisfy the relation $(L_i - L_j)^\top = \sum_{a \in V_{i,j} } S_a^\top v_{ija}$.
\label{def:observer_vector}
\end{definition}

The set $V_{i,j}$ identifies all the actions that induce informative feedback signals about a loss difference.
Actions in $V_{i,j}$ allow $L_i - L_j$ to be expressed as a linear combination of their corresponding signal matrix images, with observer vectors as coefficients.  

From Def. \ref{def:observer_set} and Def. \ref{def:observer_vector}, one can express the expected loss difference $\delta_{i,j}$ as a function of the feedback distributions $\pi_a^\star$ associated with every action $a$ in $V_{i,j}$:
\vspace*{-0.5em}
\begin{align*}
    \delta_{i,j} =  \langle L_i - L_j , p^\star \rangle 
                =  \sum_{a \in V_{i,j} } v_{ija}^\top S_a p^\star 
                =  \sum_{a \in V_{i,j} } v_{ija}^\top \pi_a^\star,\label{eq:feedback_dist} 
\vspace{-0.5em}
\end{align*}
where we used Eq.~\ref{eq:loss_diff_p_star} and $\pi^\star_a = S_a\, p^\star$ for action $a$. As a result, \CBP{} computes $\hat \delta_{i,j}(t)$ using the estimates $\hat \pi_a(t) = \frac{\nu_a(t)}{n_a(t)}$, where the vector $\nu_a(t) \in \mathbb N^{\sigma_a}$ contains the counts of observations for each unique feedback symbol observable with action $a$ up to time $t$ and $n_a(t)$ is the number of times that action $a$ was played up to time $t$. 
The confidence bound over the estimate $\hat \delta_{i,j}(t)$ is defined as~\citep{bartokICML2012}:
\vspace{-1em}
\begin{equation}
c_{i,j}(t) = \sum_{ a\in V_{i,j} } \| v_{ija} \|_{\infty} \sqrt{ \frac{\alpha \log(t)}{n_a(t)} }, \label{eq:real_CI}
\vspace{-1em}
\end{equation}

s.t. $\mathbb P[|\hat \delta_{ij}(t) - \delta_{ij}| \geq c_{ij}(t)] \leq 2 |V_{ij}| t^{1-2\alpha}$ where $\alpha>1$, and $|V_{ij}|$ is the size of the observer set $V_{ij}$.
Greater values of $\alpha$ inflate the width of the confidence bounds, inducing more elimination from the criteria in Eq. \ref{eqn:ucb_lcb}.

\vspace{-0.5em}
\paragraph{Exploration and exploitation in \CBP{}}
At round $t$, \CBP{} identifies plausible subsets of $\mathcal P$ and $\mathcal N$, denoted $\mathcal P(t)$ and $\mathcal N(t)$, based on the constrained probability space $D(t)$.
The set $\mathcal P(t)$ contains all Pareto-optimal actions $i\in \mathcal P$ whose cell $\mathcal C_i$ intersects with $D(t)$.
Similarly, the set $\mathcal N(t)$ contains all neighbor pairs $\{i, j\} \in \mathcal N$ whose cell intersection $\mathcal C_i \cap \mathcal C_j$ also intersects with $D(t)$.
When $\mathcal P(t)$ contains only one action, the set $\mathcal N(t)$ is automatically empty and therefore \CBP{} exploits.
When $\mathcal P(t)$ contains more than one action, $\mathcal N(t)$ is not empty and \CBP{} needs to explore. 
The following definitions characterize exploration:

\begin{definition}[Underplayed actions, \citet{bartokICML2012}]
The set $\mathcal R(t) = \{  a = 1, \dots, N: n_a(t)\leq \eta_a f(t) \}$ contains actions that are underplayed according to a play rate function $f(t)$ and a constant $\eta_a > 0$.
\label{def:rarely_sampled_actions}
\end{definition}

\begin{definition}[Neighbor action set, \citet{bartokICML2012}]
    The neighbor action set of a neighbor pair $\{i,j\}$ is defined as 
    $N^+_{i,j} = \{ k = 1, \dots, N: \mathcal C_i \cap \mathcal C_j \subseteq \mathcal C_k\}.$
    Note that $N^+_{i,j}$ naturally contains $i$ and $j$. If $N^+_{i,j}$ contains another action $k$, then $\mathcal C_k = \mathcal C_i$ or $\mathcal C_k = \mathcal C_j$ or $\mathcal C_k = \mathcal C_i \cap \mathcal C_j$.
\label{def:neighbor_action_set}
\end{definition}

Based on $\mathcal N(t)$ and Def. \ref{def:neighbor_action_set}, \CBP{} computes $\mathcal N^+(t) = \bigcup_{ i,j \in \mathcal N(t) } N^+_{i,j}$ for the neighbor pairs. Similarly, \CBP{} computes $\mathcal V(t) = \bigcup_{ i,j \in \mathcal N(t) } V_{i,j}$ for the observer actions.

The final set of actions considered by \CBP{}, denoted $\mathcal S(t)$, contains potentially optimal actions ($\mathcal P(t) \cup \mathcal N^+(t)$) and informative underplayed actions ($\mathcal V(t) \cup \mathcal R(t)$). 
\CBP{} selects the action with the smallest action count, i.e. $I_t = \argmax_{a \in \mathcal S(t)}\frac{W_{a}^2}{n_a(t)}$, weighted by $W_{a} = \max_{\{i,j\} \in \mathcal N} \| v_{ija} \|_{\infty}$. 

\subsection{Instantiating \RandCBP{} }
\label{subsec:randcbp}

We now introduce \RandCBP{}, a randomized counterpart of the \CBP{} strategy. 
The main idea behind \RandCBP{} is to replace deterministic confidence intervals (Eq.~\ref{eq:real_CI}) by randomized confidence intervals: 
\begin{equation*}
    c'_{i,j}(t) = \sum_{ a\in V_{i,j} } \| v_{ija} \|_{\infty}  \frac{Z_{ijt} }{\sqrt{n_a(t)}},
\vspace{-1em}
\end{equation*}

where $Z_{ijt}$ is sampled for each action pair $\{i,j\}$ from a discrete probability distribution supported over $K$ bins in the interval $[A, B]$.
Note that the \CBP{} strategy corresponds to the specific case of $K = 1$ and $A=B=\sqrt{\alpha \log(t)}$.

\vspace{-0.5em}
\paragraph{Randomization procedure} 
\label{par:rand_procedure} 
Let $\rho_1 = A,\dots,\rho_K = B$ denote $K$ equally spaced values, and $p_k$ denote the probability of sampling the value $\rho_k$, with $k=1,\dots,K$.
The probabilities assigned to the remaining $K-1$ points are shaped according to the positive side of a discretized Gaussian distribution centered at 0. Formally, for $k \leq K-1$, let $\bar p_k := \exp(-\rho_k^2 / 2 \sigma^2)$. 
Then, the $p_k$ corresponds to the normalized probabilities, that is, $p_k := (1 - \varepsilon) \bar p_k / (\sum_k \bar p_k)$.
The above distribution from which the $Z_{ijt}$ are sampled is a truncated (between $A$ and $B$) and discretized (into $K$ points) Gaussian distribution with tunable hyper-parameters $\varepsilon$, $\sigma > 0$, and $K$.
A pseudo-code of the randomization procedure is provided in Appendix \ref{appendix:details}. 

This randomization procedure was introduced by \citet{vaswani2019old} for randomizing Upper Confidence Bound (UCB) strategies in the bandit setting.
We generalize these ideas to the broader PM setting, where confidence bounds articulate a successive elimination criterion. This requires to define the randomized confidence bounds on quantities estimated for each action pair $\{i,j\}$ in $\mathcal N$. We will now show how this mechanism can be considered seamlessly in the theoretical analysis of \CBP{}, allowing to maintain the regret guarantees with \RandCBP{}. 

\vspace{-0.5em}
\paragraph{Regret analysis}
The analysis, follows the structure of \CBP's analysis \citep{bartok2011minimax}, and involves upper bounding the expected number of times the confidence bounds succeed and fail, as detailed in Appendix \ref{subsection:boundingA} and \ref{subsection:boundingB}.
Inspired by \citet{kveton2019garbage}, we leverage that the probability that the randomized SE criteria fails becomes negligible over time. 
For the success case, we adapt lemmas from \citet{bartokICML2012} by observing that the randomized bounds are always upper bounded by their deterministic counterparts. Detailed analysis is reported in Appendix \ref{appendix:regretRandCBP}. 

\begin{theorem}
Consider the interval $[A,B]$, with $B = \sqrt{\alpha \log(t) }$ and $A\leq0$. 
Set the randomization over $K$ bins with a probability $\epsilon$ on the tail and a standard deviation $\sigma$. 
Set $f(t) = \alpha^{1/3} t^{2/3} \log(t)^{1/3}$, $\eta_a = W_a^{2/3}$ and $\alpha>1$. 
On easy games, \RandCBP{} achieves
\begin{align*}
    \mathbb{E}[ R_T ] \leq N \left[ 2 (1 + \frac{1}{2\alpha-2})  | \mathcal V| + 1 \right] + \sum_{ k = 1 }^N \delta_k + \\ \sum_{ k = 1, \delta_k>0 }^N 4 W_k^2 \frac{g_k^2}{\delta_k}\alpha \log(T),
\end{align*}
with $\mathcal V = \bigcup_{i,j \in \mathcal N} V_{ij}$ and $g_k$ being game dependent constants.
On hard games, assuming positive constants $C_1$ and $C_2$, \RandCBP{} achieves
$$ \mathbb E[ R_T ] \leq C_1 N + C_2 T^{2/3} \log^{1/3}(T) .$$  
\label{thm:randcbp}
\vspace{-2.5em}
\end{theorem}
Similarly to the guarantees of \CBP{} \citep{bartokICML2012}, the bound on easy games is problem dependent while the bound on hard games is problem independent. 
In both cases, the expected regret of \RandCBP{} grows at the same rate as \CBP{} on the horizon $T$. 
We will see in the experiments that \RandCBP{} empirically outperforms \CBP{}.

\vspace{-0.5em}
\section{The contextual setting}
\label{sec:contextual_setting}

\CBPside{}~\citep{bartok2012CBPside} extends \CBP{} to the linear and logistic contextual PM settings.
The regret guarantee of \CBPside{} is currently restricted to easy games.
An empirical study suggests that a sub-linear regret guarantee should be achievable on hard games as well~\citep{lienert2013exploiting}.
However, the exploration strategy of \CBPside{} prevents both applicability and derivation of a regret guarantee on hard games.
We introduce \CBPsidestar{}, an extension of \CBPside{} with an exploration strategy based on pseudo-counts applicable on easy and hard games.
We then propose \RandCBPsidestar{}, a stochastic counterpart of \CBPsidestar{}, that enjoys regret guarantees on both easy and hard games in the linear setting, while empirically outperforming its deterministic counterpart.
Pseudo-codes of \CBPsidestar{} and \RandCBPsidestar{} are reported in Appendix A.

\subsection{The linear \CBPsidestar{} strategy}

Recall that under the contextual PM setting, $p^\star(x)$ denotes the outcome distribution given $d$-dimensional contexts $x \in \mathcal X\subseteq \mathbb R^d$. In the linear setting, $p^\star(x) = \theta x  $, where $\theta \in\mathbb R^{M\times d}$ is an unknown parameter matrix.
Similarly to the non-contextual setting, it is not possible to estimate the outcome distribution directly (see Section~\ref{par:feedback_dist}). 
Consequently, \CBPsidestar{} exploits the connection between outcome and feedback distributions. In the contextual setting, the feedback distribution is $\pi_i^\star(x) = S_i p^\star(x) \in \Delta_{\sigma_i}$ for all actions $i \in \{1, \dots, N\}$ .
If we denote $\theta_i =S_i  \theta  \in \mathbb R^{\sigma_i\times d}$ as the per-action unknown parameter of the regression, then the contextual feedback distribution is $\pi_i^\star(x) =  \theta_i x $. 

\CBPsidestar{} estimates $\theta_i$ with a ridge estimator defined as $\hat \theta_i(t) = Y_{i,t} X_{i,t}^\top (\lambda I_d + X_{i,t} X_{i,t}^\top)^{-1}, $ where $X_{it} \in \mathbb R^{d \times t}$ is the history of contexts, $Y_{it} \in \{0,1\}^{\sigma_i \times t}$ is the history of one-hot-encoded feedback symbols for action $i$, and $I_d$ the $d$-dimensional identity matrix.
The following confidence bound on $\hat \delta_{i,j}(x)$ holds with probability $1/t^2$:
\vspace{-1em}
\begin{align}
c_{i,j}(x) = \displaystyle \sum_{a\in V_{ij}} \|v_{ija}\|_2  \nonumber \\ 
\times \sigma_a \left(\sqrt{ d \log(t) + 2 \log(1/t^2) } +   \sigma_a \right)   \| x \|_{G_{a,t}^{-1}}, 
\label{eq:contextual_CI}  
\end{align}
where $G_{a,t}= \lambda I_d + X_{at} X_{at}^\top$ is the Gram matrix and $\| x \|^2_{G^{-1}_{a,t}} = x^\top G^{-1}_{a,t} x$ is the weighted 2-norm.

\begin{remark}
    The confidence bound of \CBPsidestar{} (Eq. \ref{eq:contextual_CI}) corrects the bound $c_{i,j}(x) \propto d (\sqrt{d \log(t) \dots })$ used in \citet{bartok2012CBPside, lienert2013exploiting}. 
    We multiply by $\sigma_a$ instead of $d$ to correctly instantiate Theorem 3 in \citet{bartok2012CBPside} over matrix traces.
    The corrected confidence bound is less conservative as $\sigma_a$ is usually smaller than $d$.
\end{remark}

\vspace{-1em}
\paragraph{Exploration based on pseudo-counts}
\CBPsidestar{} is based on a new definition of underplayed actions (Def.~\ref{def:rarely_sampled_actions}) suitable for the contextual setting. 
The definition enables the applicability of \CBPsidestar{} to hard games, and the regret analysis in both easy and hard games.
In the non-contextual setting, underplayed actions are based on the number of times that action $a$ was played up to time~$t$, i.e. $n_a(t)$. 
A natural extension would consist in counting the number of times action $a$ was played in context $x_t$. 
Unfortunately, given that contexts are usually sampled from a continuous domain $\mathcal X$, each context is typically encountered only once over a game, making such counters irrelevant. 

\begin{definition}[Underplayed actions (contextual case) ]
At round $t$, the set $ \mathcal R(x_t) = \{  a = 1, \dots, N: 1/ \| x_t \|^2_{G_{a,t}^{-1}} <   \eta_a f(t)   \}$ contains actions that are underplayed at context $x_t$ given a play rate function $f(t)$ and a constant $\eta_a > 0$. 
\label{def:rarely_sampled_actions_contextual}
\end{definition}

The quantity $1 / \| x \|^2_{G^{-1}_{a,t}}$  is a \textit{pseudo-count} of the number of selections of action $a$ at a given context $x$. 
In the specific case of orthogonal contexts sampled from the finite set of $d$-dimensional one-hot vectors, $1/ \| x \|^2_{G^{-1}_{a,t}}$ corresponds to the exact number of selections of action $a$ in context $x$. When contexts are not orthogonal, the pseudo-count 
increases proportionally to the frequency of the action being played
in similar contexts.
\vspace{-0.5em}
\subsection{Instantiating linear \RandCBPsidestar{} }

The randomized counterpart of \CBPsidestar{}, namely \RandCBPsidestar{}, relies on randomized confidence intervals defined at time $t$ for a pair $\{i,j\}$, defined as 
\begin{equation*}
    c_{i,j}'(x) = \sum_{a\in V_{ij}} ||v_{ija}||_2 \sigma_a ( Z_{ijt} + \sigma_a )  || x||_{G_{a,t}^{-1}},
\vspace{-1em}
\end{equation*}
where $Z_{ijt}$ is a random variable bounded in $[A, B]$ and follows the randomization procedure presented in Section~\ref{par:rand_procedure}.

\vspace{-0.5em}
\paragraph{Regret analysis}
We leverage the non-contextual analysis of \CBP{} \citep{bartokICML2012}. We adapt the analysis to the contextual case by introducing pseudo-counts and simplifying the obtained expressions with the Cauchy-Schwartz inequality. The contextual confidence bounds are simplified by considering the total number of feedback symbols in the game. To upper bound the expected number of times the confidence bound fails, we consider some worst-case probability over all actions of the game. 

\begin{theorem}
Consider the interval $[A,B]$, with $B = \sqrt{ d \log(t) + 2 \log(1/t^2) } $ and $A\leq0$.
Set the randomization over $K$ bins with a probability $\epsilon$ on the tail and a standard deviation $\sigma$. 
Let $f(t) = \alpha^{1/3} t^{2/3} \log(t)^{1/3}$, $\eta_a = W_a^{2/3}$ and $\alpha>1$.
Assume $\| x_t\|_2 \leq E$ and positive constants $C_1$, $C_2$, $C_3$, and $C_4$. On easy games, \RandCBPsidestar{} achieves:
\begin{equation*}
    \mathbb{E}[ R_T ]\leq C_1 N + C_2 N d \sqrt{T} \log(T)
\end{equation*}
and, on hard games, \RandCBPsidestar{} achieves:
\begin{equation*}
    \mathbb{E}[ R_T ]\leq C_3 N + C_4 \sqrt{d} \log(T)^{1/3} T^{2/3} .
\end{equation*}
\label{thm:contextual_upperbound}
\vspace{-2.5em}
\end{theorem}
The detailed analysis is reported in Appendix \ref{appendix:regret_CBPside}.
Similarly to \CBPside{},  the guarantee of \RandCBPsidestar{} is problem independent and grows at the same rate on easy games. 
However, \RandCBPsidestar{} presents a new problem dependent guarantee on hard games.
Additionally, since \CBPsidestar{} is a special case of \RandCBPsidestar{} with $A=B$ and $K=1$, our bounds for \RandCBPsidestar{} also hold for \CBPsidestar{}.
We will see in the experiments that \RandCBPsidestar{} empirically outperforms \CBPsidestar{} on easy and hard games.

\vspace{-0.5em}
\paragraph{ Applicability to the logistic setting }

While the focus of this manuscript is on the linear setting, \CBPsidestar{} and \RandCBPsidestar{} (and their guarantees) can be extended to the logistic setting by considering the logistic estimator and confidence bounds defined in \citet{bartokICML2012}. 

\vspace{-0.5em}
\section{Numerical experiments}
\label{sec:experiments}

We conduct experiments to validate the empirical performance of \RandCBP{} and \RandCBPsidestar{} on the well-known Apple Tasting (AT)~\citep{helmbold2000apple} and Label Efficient (LE)~\citep{helmbold1997some} games. 
AT is a two actions and two outcomes easy game:
\renewcommand{\arraystretch}{0.8}
\vspace{-1.5em}
$$\bf{L}=\kbordermatrix{
    & \text{} & \text{}\\
      \text{action 1}  & 1 & 0 \\
      \text{action 2} & 0 & 1}, \notag  
      \bf{H}= \kbordermatrix{
    &\text{} & \text{}\\
     \text{} & \bot & \bot \\
     \text{} & \wedge & \odot }.$$
LE is a hard game with three actions and two outcomes:
\vspace{-1.5em}
$$\bf{L}=\kbordermatrix{ & \text{} & \text{}\\
        \text{action 1} & 1 & 1\\
        \text{action 2} & 0 & 1\\
        \text{action 3} & 1 & 0}, \notag 
        \bf{H}= \kbordermatrix{
        & \text{ } & \text{ }\\
        \text{} & \bot & \odot\\
        \text{} & \wedge & \wedge\\
        \text{} & \wedge & \wedge}.$$
For reproducibility, we provide in Appendix \ref{appendix:games} a detailed analysis of both games. 
\renewcommand{\arraystretch}{1}

\vspace{-0.5em}
\subsection{Evaluation of \RandCBP{}}

Since both AT and LE admit binary outcomes, the outcome distribution corresponds to $p^\star=[p, 1 - p]$ with $p\in [0, 1]$.
We consider \textit{imbalanced} and \textit{balanced} instances.
Imbalanced instances, where $p \sim \mathcal U_{[0, 0.2] \cup [0.8, 1]}$, are usually solved faster since outcomes have lower variance.
Balanced instances, where $p \sim \mathcal U_{[0.4,0.6]}$, require more exploration to estimate $p^{\star}$ with confidence. 
This leads to four cases: imbalanced/balanced AT and imbalanced/balanced LE.  
For each of the four cases, we run the experiment $96$ times on a $T = 20$k horizon.  
We consider the deterministic \PMDMED{} and \CBP{} as baselines, as well as the stochastic \BPMLeast, \TSPM{}, and \TSPMgaussian{} (in the settings where they have a guarantee).
Implementation and hyper-parameters details are reported in Appendix~\ref{appendix:details}.
We measure the performance with the average non-contextual cumulative regret (Eq.~\ref{eq:noncontextual_regret}) and the win-count (number of times a strategy achieves the lowest cumulative regret at end of the game). 
We perform a one sided Welch's t-test to asses if the cumulative regret of \RandCBP{} at the end of the game is significantly lower than the baselines'.

\begin{figure}
     \centering
     \begin{subfigure}[b]{0.23\textwidth}
         \centering
         \includegraphics[width=\textwidth]{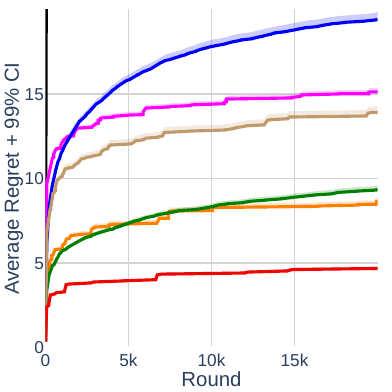}
         \caption{Imbalanced AT}
         \label{fig:easyFF}
     \end{subfigure}
     \hfill
     \begin{subfigure}[b]{0.23\textwidth}
         \centering
         \includegraphics[width=\textwidth]{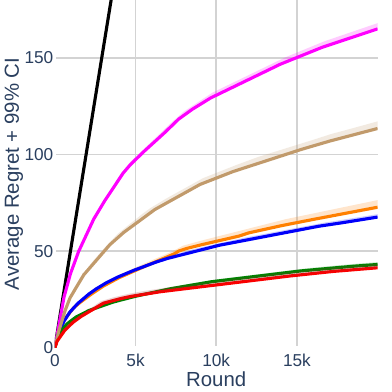}
         \caption{Balanced AT}
         \label{fig:hardFF}
     \end{subfigure}
     \hfill
     \begin{subfigure}[b]{0.23\textwidth}
         \centering
         \includegraphics[width=\textwidth]{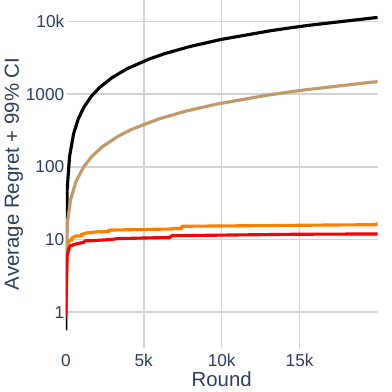}
         \caption{Imbalanced LE }
         \label{fig:easyEF}
     \end{subfigure}
     \hfill
     \begin{subfigure}[b]{0.23\textwidth}
         \centering
         \includegraphics[width=\textwidth]{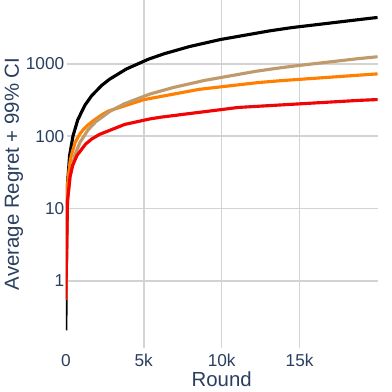}
         \caption{Balanced LE}
         \label{fig:hardEF}
     \end{subfigure}
     \begin{subfigure}[b]{0.45 \textwidth}
         \centering
         \vspace{0.5em}
         \includegraphics[width=\textwidth]{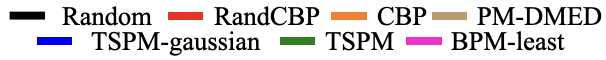}
     \end{subfigure}
     \vspace{-0.5em}
     \caption{Average regret (with $99\%$ confidence interval above) on non-contextual AT and LE games. }
    \label{fig:without_side_info}
    \vspace{-2em}
\end{figure}

\vspace{-1em}

\paragraph{Results}
Figure~\ref{fig:without_side_info} shows the non-contextual cumulative regret for each strategy over the four configurations considered. Numeric details are reported in Table~\ref{tab:numericdetail_noncontextual_AT} and \ref{tab:numericdetail_noncontextual_LE} of Appendix~\ref{appendix:additional_exps}.
In all four cases, \RandCBP{} is the best strategy in terms of average regret.
\RandCBP{} achieves a regret significantly lower (p-value$<0.01$) than all baselines in three settings (Figures~\ref{fig:easyFF}, \ref{fig:easyEF}, and~\ref{fig:hardEF}) out of the four considered. 
In the balanced AT game (Figure~\ref{fig:hardFF}), \RandCBP{} is not statistically different from \CBP{} (p-value=0.055) and \TSPM{} (p-value=0.854). For \CBP{}, this can be attributed to its high variance (std=138). For \TSPM{}, we observe from the win-count that \RandCBP{} achieves lowest regret 37 times, against 13 for \TSPM{}.
Performance similarity between \RandCBP{} and \TSPM{} reflects the theoretical connections between randomizing confidence bounds and Thompson Sampling~\citep{vaswani2019old}, on which \TSPM{} is based.

\vspace{-0.5em}
\subsection{Evaluation of \RandCBPsidestar{} }

Here, the outcome distribution is a linear function of $10$-dimensional contexts (sampled uniformly in $[0,1]^{10}$)   and a fixed unknown parameter $\theta \in \mathbb R^{M \times 10}$ with all values set at $0.1$. From the uniform context distribution, we have $0.5$ as mean for each context feature. Therefore, the resulting outcome distributions are more \textit{balanced}.
We run the experiment $96$ times over a $T = 20$k horizon.
We report the contextual cumulative regret (Eq.~\ref{eq:contextual_regret}), the win-count and Welch's t-test.
The only PM baseline in this setting is \CBPsidestar{}.
We therefore resort to baselines that only apply on specific games. We consider \PGIDS{} \citep{grant2021bayesianappletasting}, \PGTS{} \citep{grant2021bayesianappletasting}, and \STAP{} \citep{helmbold2000apple} for the AT game, and \CESA{} \citep{cesa2006prediction} for the LE game.
Implementation and hyper-parameters details are reported in Appendix \ref{appendix:details}. 

\vspace{-1em}
\paragraph{Results}
Figure~\ref{fig:with_linear_info} shows the contextual cumulative regret for each strategy on AT and LE, with dotted-lines indicating game-specific baselines.
Numeric details are reported in Table \ref{tab:numericdetail_contextual} (Appendix~\ref{appendix:additional_exps}).
Over the horizon $T=20$k, \RandCBPsidestar{} achieves the best regret performance in both settings and significantly (p-value$<0.01$ in AT and LE) improves over \CBPsidestar{}, \STAP{}, and \CESA{}.
In the AT game (Figure \ref{fig:AT_with_linear_info}), \PGIDS{} achieves the lowest regret on the truncated horizon $T=7.5$k. 
However, \PGIDS{} and \PGTS{} scale in cubic time with the number of contexts due to the necessity to sample and invert matrices at each round, whereas \CBPsidestar{} and \RandCBPsidestar{} enjoy lower complexity thanks to the Sherman-Morison update \citep{sherman_morrison}, making them usable on long horizons. 
We emphasize that \PGIDS{}, \PGTS{}, \STAP{}, and \CESA{} are \textit{game-specific}, unlike \CBPsidestar{} and \RandCBPsidestar{}. 

\begin{figure}
  \centering
  \begin{subfigure}{0.23\textwidth}
    \centering
    \includegraphics[width=\linewidth]{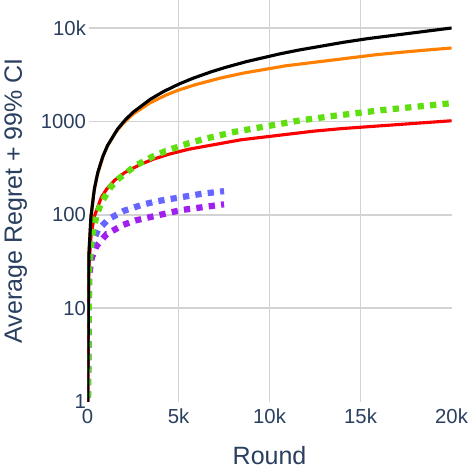}
    \caption{Apple Tasting (AT)}
    \label{fig:AT_with_linear_info}
  \end{subfigure}
  \begin{subfigure}{0.229\textwidth}
    \centering
    \includegraphics[width=\linewidth]{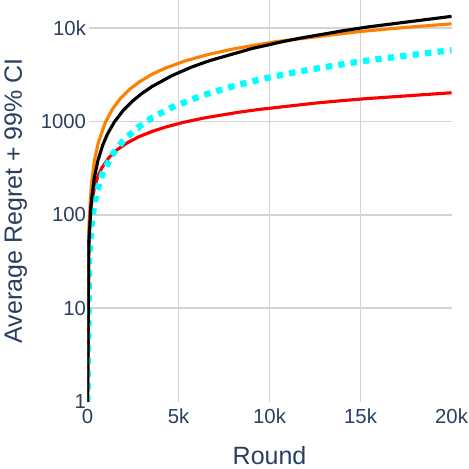}
    \caption{Label Efficient (LE)}
    \label{fig:LE_with_linear_info}
  \end{subfigure}
  \begin{subfigure}{0.4\textwidth}
    \centering
    \vspace{0.5em}
    \includegraphics[width=\linewidth]{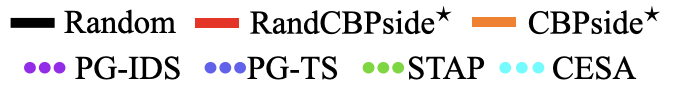}
  \end{subfigure}
  \vspace{-0.75em}
  \caption{Average regret (with $99\%$ confidence interval above) on contextual games (10-$d$ contexts).  }
    \label{fig:with_linear_info}
    \vspace{-1.5em}
\end{figure}

\vspace{-0.5em}
\section{Use-case: adaptive monitoring of a deployed black-box classifier}
\label{sec:case_study}

Partial monitoring has a reputation for being a complex framework due to its generality \cite{kirschner2023linear}, which can hinder its adoption in real-world problems.
Documented applied studies of PM do not emphasize on how to employ the framework towards an application~\citep{singla2014contextual, kirschner2023linear}.
Here, we show how to formulate a real-world application as a PM problem, to encourage future applied research. 

We consider the problem of cost-efficiently verifying the prediction error rate of a deployed black-box classifier.
We assume a streaming setting where, at each round, the classifier receives an input and outputs probabilities to $C$ classes. 
The index of the highest probability determines the \textit{predicted class}. 
Each of the $C$ predicted classes has an error rate $p_c$.  
The goal is to identify which predicted classes have an error rate greater than a tolerance threshold $\tau \in [0,1]$ while minimizing the number of verifications. 
In contrast to \citet{kossen2021active}, who require a verification budget to be specified, 
our approach assumes no prior knowledge regarding the number of required verifications.

\vspace{-0.5em}
\paragraph{Problem formulation}

Everytime class $c$ is predicted, a binary outcome is generated: either the classifier mispredicted (0) or not (1). 
Thus, the outcome distribution is $p^\star_c = [p_c, 1-p_c]$ where $p_c$ denotes the error rate we aim to estimate. 
We design a PM game, that we name $\tau$-detection game, to estimate the outcome distribution over multiple rounds for a predicted class $c$:
\vspace{-0.75em}
\begin{equation*}
    \bf{L}=\kbordermatrix{
    & \text{error} & \text{no error}\\
      \text{verify}  & 1 & 1 \\
      \text{pass} & 1/\tau &  0 }, \notag
    \bf{H}= \kbordermatrix{
    &\text{error} & \text{no error}\\
     \text{verify} & \wedge & \odot  \\
     \text{pass} & \bot & \bot  }.
\end{equation*}

After each classifier prediction, the PM agent can either require a verification (observation of the true class) or not (pass). 
The loss matrix is designed such that  the optimal action is to pass when $p_c <\tau$ and to verify when $p_c \geq \tau$.
The ``verify'' action is informative about the error rate $p_c$, but it has a fixed cost no matter the outcome.  
For reproducibility, we provide in Appendix \ref{appendix:detection_game} the analysis of this game.

\begin{figure}
  \centering
  \begin{subfigure}{0.23\textwidth}
    \centering
    \includegraphics[width=\linewidth]{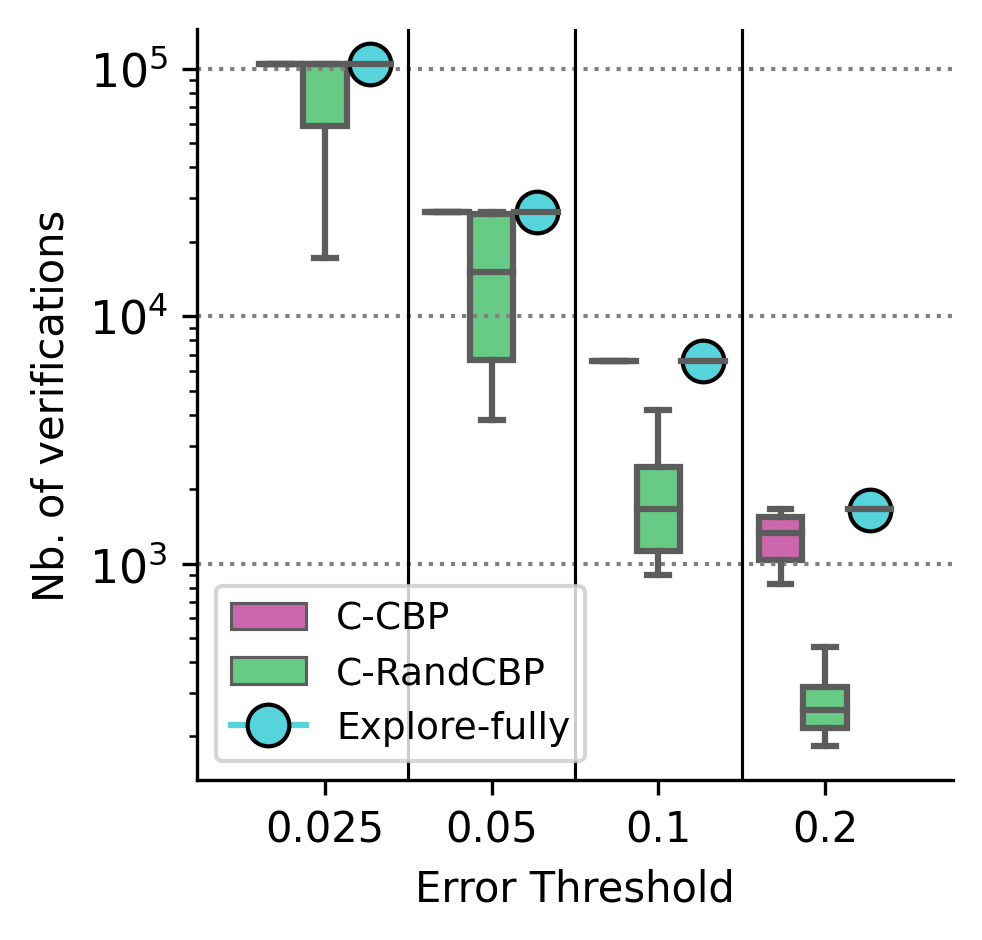}
    \caption{Case 1}
  \end{subfigure}
  \begin{subfigure}{0.23\textwidth}
    \centering
    \includegraphics[width=\linewidth]{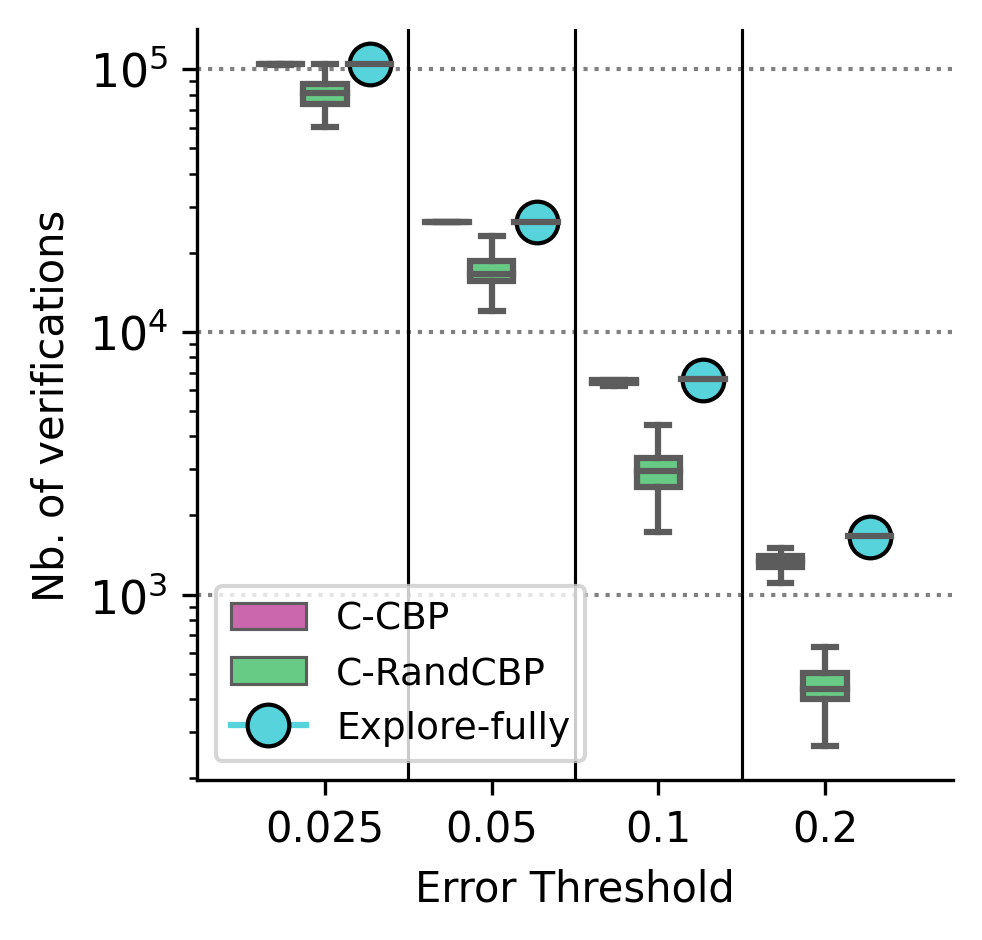}
    \caption{Case 2}
  \end{subfigure}
  \caption{Cost-efficiency of the considered approaches (quartiles of number of verifications).  }
  \label{figure:use_case}
  \vspace{-1em}
\end{figure}

\vspace{-0.5em}
\paragraph{Experiment setup}

We simulate a variety of black-box classifiers by randomly generating confusion matrices with a global error rate lower than $10\%$ (the black-box would not be deployed otherwise).
Prediction errors from the black-box can be \textit{uniformly} distributed across the classes or \textit{non-uniformly} distributed. 
In addition, the distribution of the true classes in the stream can be \textit{balanced} or \textit{imbalanced}.
We obtain four configurations (uniform/balanced, uniform/imbalanced, non-uniform/balanced, non-uniform/imbalanced). We consider the two opposite configurations: i) balanced true classes with uniform black-box errors (case 1), and ii) imbalanced true classes with non-uniform black-box errors (case 2).

We run the experiment $96$ times. We consider four error tolerance thresholds $\tau \in \{ 0.025, 0.05, 0.1, 0.2\}$ and  
a classification task with $C=10$ classes. We measure the mean and median f1-score to assess how accurate a given approach is at identifying predicted classes whose error rate exceeds $\tau$, and the underlying average number of verifications used by each approach.
For validation and comparison purposes, we consider a maximum number of verifications that one is willing to spend to estimate accurately each error rate. The maximum number of verifications is derived from Wald's confidence intervals formula (more details in Appendix \ref{appendix:case_study}).
The non-adaptive \explorecommit{} baseline consumes entirely the maximum number of verifications. 
We compare \explorecommit{} against the adaptive strategies \crandcbp{} and \ccbp{}, which consist of $C$ instances of \RandCBP{} (resp. \CBP{}) that play the $\tau$-detection game.

\vspace{-0.5em}
\paragraph{Results} Tables \ref{tab:case1} and \ref{tab:case2} (reported in Appendix \ref{appendix:additional_exps}) show that \crandcbp{}, \ccbp{} and \explorecommit{} all have an average f1-score is within the same range, indicating that the three strategies are equally effective in identifying predicted classes that exceed the threshold $\tau$. 
In case 1, for the smallest threshold $\tau = 0.025$, \explorecommit{} and \ccbp{} have an average f1-score of $0.96\pm0.15$ and \crandcbp{} of $0.95\pm0.16$. For $\tau=0.2$, the average f1-score is equal to $1.0$ for all strategies. Similar tendencies are observed for the other cases.
Figure \ref{figure:use_case} shows that the number of verifications consumed by \crandcbp{} to achieve the task is consistently lower than the one of \explorecommit{} and \ccbp{}.
In case 1, \crandcbp{} reduces the verification cost by $15\%$ for a small error threshold ($\tau = 0.025$) and by $73\%$ for $\tau = 0.2$, relatively to \explorecommit{}.
In case 2, \crandcbp{} reduces the verification cost by $18\%$ for a small error threshold ($\tau = 0.025$) and by $62\%$ at $\tau = 0.2$, relatively to \explorecommit{}. 

\vspace{-0.5em}
\section{Conclusion}

This work extends randomization techniques~\cite{kveton2019garbage,vaswani2019old} designed for OFU-based methods in the bandit setting to successive elimination strategies in the more general partial monitoring framework. We show that it is possible to randomize \CBP{}-based strategies~\cite{bartokICML2012,bartok2012CBPside}, allowing to maintain the regret guarantees while improving significantly their empirical performance.
In the contextual PM setting, we propose a correction to the seminal \CBPside{}; the resulting \CBPsidestar{} is the first strategy to enjoy regret guarantees on both easy and hard contextual games.
Our proposed \RandCBP{} and \RandCBPsidestar{} establish state-of-the-art performance in multiple settings while maintaining regret guarantees.
To further bridge the gap between theory and practice, we present a use case on the real-world problem of monitoring the error rate of deployed classifiers. 
Future research may consist in obtaining tighter regret bounds for \RandCBP{} and \RandCBPsidestar{}.
Obtaining lower bounds in the contextual setting is another possible future research.



\bibliography{Bib.bib}

@InProceedings{ginart2022mldemon,
  title = 	 { MLDemon:Deployment Monitoring for Machine Learning Systems },
  author =       {Ginart, Tony and Jinye Zhang, Martin and Zou, James},
  booktitle = 	 {Proceedings of The 25th International Conference on Artificial Intelligence and Statistics},
  pages = 	 {3962--3997},
  year = 	 {2022},
  editor = 	 {Camps-Valls, Gustau and Ruiz, Francisco J. R. and Valera, Isabel},
  volume = 	 {151},
  series = 	 {Proceedings of Machine Learning Research},
  month = 	 {28--30 Mar},
  publisher =    {PMLR},
  url = 	 {https://proceedings.mlr.press/v151/ginart22a.html}
}

@inproceedings{bartok2011minimax,
  author       = {G{\'{a}}bor Bart{\'{o}}k and
                  D{\'{a}}vid P{\'{a}}l and
                  Csaba Szepesv{\'{a}}ri},
  editor       = {Sham M. Kakade and
                  Ulrike von Luxburg},
  title        = {Minimax Regret of Finite Partial-Monitoring Games in Stochastic Environments},
  booktitle    = {{COLT} 2011 - The 24th Annual Conference on Learning Theory, June
                  9-11, 2011, Budapest, Hungary},
  series       = {{JMLR} Proceedings},
  pages        = {133--154},
  publisher    = {JMLR.org},
  year         = {2011},
  url          = {http://proceedings.mlr.press/v19/bartok11a/bartok11a.pdf},
  bibsource    = {dblp computer science bibliography, https://dblp.org}
}

@book{cesa2006prediction,
  title={Prediction, learning, and games},
  author={Cesa-Bianchi, Nicolo and Lugosi, G{\'a}bor},
  year={2006},
  publisher={Cambridge university press}
}

@mastersthesis{lienert2013exploiting,
  title={Exploiting side information in partial monitoring games: An empirical study of the cbp-side algorithm with applications to procurement},
  author={Lienert, Ian},
  year={2013},
  school={Eidgen{\"o}ssische Technische Hochschule Z{\"u}rich, Department of Computer Science}
}

@inproceedings{piccolboni2001discrete,
  title={Discrete Prediction Games with Arbitrary Feedback and Loss},
  author={Antonio Piccolboni and Christian Schindelhauer},
  booktitle={COLT/EuroCOLT},
  year={2001},
  url={https://api.semanticscholar.org/CorpusID:269616606}
}

@inproceedings{tsuchiya2020analysis,
  author       = {Taira Tsuchiya and
                  Junya Honda and
                  Masashi Sugiyama},
  editor       = {Hugo Larochelle and
                  Marc'Aurelio Ranzato and
                  Raia Hadsell and
                  Maria{-}Florina Balcan and
                  Hsuan{-}Tien Lin},
  title        = {Analysis and Design of Thompson Sampling for Stochastic Partial Monitoring},
  booktitle    = {Advances in Neural Information Processing Systems 33: Annual Conference
                  on Neural Information Processing Systems 2020, NeurIPS 2020, December
                  6-12, 2020, virtual},
  year         = {2020},
  url          = {https://proceedings.neurips.cc/paper/2020/hash/649d45bf179296e31731adfd4df25588-Abstract.html},
  bibsource    = {dblp computer science bibliography, https://dblp.org}
}

@inproceedings{tsuchiya2022bestofbothwords,
  author       = {Taira Tsuchiya and
                  Shinji Ito and
                  Junya Honda},
  editor       = {Shipra Agrawal and
                  Francesco Orabona},
  title        = {Best-of-Both-Worlds Algorithms for Partial Monitoring},
  booktitle    = {International Conference on Algorithmic Learning Theory, February
                  20-23, 2023, Singapore},
  series       = {Proceedings of Machine Learning Research},
  pages        = {1484--1515},
  publisher    = {{PMLR}},
  year         = {2023},
  url          = {https://proceedings.mlr.press/v201/tsuchiya23a.html},
  bibsource    = {dblp computer science bibliography, https://dblp.org}
}

@article{grant2021bayesianappletasting,
  author       = {James A. Grant and
                  David S. Leslie},
  title        = {Apple Tasting Revisited: Bayesian Approaches to Partially Monitored
                  Online Binary Classification},
  journal      = {CoRR},
  volume       = {abs/2109.14412},
  year         = {2021},
  url          = {https://arxiv.org/abs/2109.14412},
  eprinttype   = {arXiv},
  eprint       = {2109.14412},
  bibsource    = {dblp computer science bibliography, https://dblp.org}
}

@InProceedings{lattimore2022rustichini,
  title = 	 {Minimax Regret for Partial Monitoring: Infinite Outcomes and Rustichini’s Regret},
  author =       {Lattimore, Tor},
  booktitle = 	 {Proceedings of Thirty Fifth Conference on Learning Theory},
  pages = 	 {1547--1575},
  year = 	 {2022},
  editor = 	 {Loh, Po-Ling and Raginsky, Maxim},
  volume = 	 {178},
  series = 	 {Proceedings of Machine Learning Research},
  month = 	 {02--05 Jul},
  publisher =    {PMLR},
  url = 	 {https://proceedings.mlr.press/v178/lattimore22a.html}
}

@article{helmbold2000apple,
title = {Apple tasting},
journal = {Information and Computation},
volume = {161},
number = {2},
pages = {85-139},
year = {2000},
issn = {0890-5401},
doi = {https://doi.org/10.1006/inco.2000.2870},
url = {https://www.sciencedirect.com/science/article/pii/S0890540100928700},
author = {David P. Helmbold and Nicholas Littlestone and Philip M. Long}
}

@article{auer2002nonstochastic,
  author       = {Peter Auer and
                  Nicol{\`{o}} Cesa{-}Bianchi and
                  Yoav Freund and
                  Robert E. Schapire},
  title        = {The Nonstochastic Multiarmed Bandit Problem},
  journal      = {{SIAM} J. Comput.},
  volume       = {32},
  number       = {1},
  pages        = {48--77},
  year         = {2002},
  url          = {https://doi.org/10.1137/S0097539701398375},
  doi          = {10.1137/S0097539701398375},
  bibsource    = {dblp computer science bibliography, https://dblp.org}
}

@InProceedings{vaswani2019old,
  title = 	 {Old Dog Learns New Tricks: Randomized UCB for Bandit Problems},
  author =       {Vaswani, Sharan and Mehrabian, Abbas and Durand, Audrey and Kveton, Branislav},
  booktitle = 	 {Proceedings of the Twenty Third International Conference on Artificial Intelligence and Statistics},
  pages = 	 {1988--1998},
  year = 	 {2020},
  editor = 	 {Chiappa, Silvia and Calandra, Roberto},
  volume = 	 {108},
  series = 	 {Proceedings of Machine Learning Research},
  month = 	 {26--28 Aug},
  publisher =    {PMLR},
  url = 	 {https://proceedings.mlr.press/v108/vaswani20a.html}
}

@inproceedings{even2002pac,
author = {Even-Dar, Eyal and Mannor, Shie and Mansour, Yishay},
title = {PAC Bounds for Multi-armed Bandit and Markov Decision Processes},
year = {2002},
isbn = {354043836X},
publisher = {Springer-Verlag},
address = {Berlin, Heidelberg},
booktitle = {Proceedings of the 15th Annual Conference on Computational Learning Theory},
pages = {255–270},
numpages = {16},
series = {COLT '02}
}

@inproceedings{chapelle2011empirical,
 author = {Chapelle, Olivier and Li, Lihong},
 booktitle = {Advances in Neural Information Processing Systems},
 editor = {J. Shawe-Taylor and R. Zemel and P. Bartlett and F. Pereira and K. Weinberger},
 pages = {},
 publisher = {Curran Associates, Inc.},
 title = {An Empirical Evaluation of Thompson Sampling},
 url = {https://proceedings.neurips.cc/paper_files/paper/2011/file/e53a0a2978c28872a4505bdb51db06dc-Paper.pdf},
 volume = {24},
 year = {2011}
}

@misc{gurobi,
  author = {{Gurobi Optimization, LLC}},
  title = {{Gurobi Optimizer Reference Manual}},
  year = 2023,
  url = "https://www.gurobi.com"
}

@article{mitchell2011pulp,
  title={PuLP: a linear programming toolkit for python},
  author={Mitchell, Stuart and OSullivan, Michael and Dunning, Iain},
  journal={The University of Auckland, Auckland, New Zealand},
  volume={65},
  year={2011}
}

@inproceedings{vaswani2017attention,
  author       = {Ashish Vaswani and
                  Noam Shazeer and
                  Niki Parmar and
                  Jakob Uszkoreit and
                  Llion Jones and
                  Aidan N. Gomez and
                  Lukasz Kaiser and
                  Illia Polosukhin},
  editor       = {Isabelle Guyon and
                  Ulrike von Luxburg and
                  Samy Bengio and
                  Hanna M. Wallach and
                  Rob Fergus and
                  S. V. N. Vishwanathan and
                  Roman Garnett},
  title        = {Attention is All you Need},
  booktitle    = {Advances in Neural Information Processing Systems 30: Annual Conference
                  on Neural Information Processing Systems 2017, December 4-9, 2017,
                  Long Beach, CA, {USA}},
  pages        = {5998--6008},
  year         = {2017},
  url          = {https://proceedings.neurips.cc/paper/2017/hash/3f5ee243547dee91fbd053c1c4a845aa-Abstract.html},
  bibsource    = {dblp computer science bibliography, https://dblp.org}
}

@article{cifar,
  title={Learning multiple layers of features from tiny images},
  author={Krizhevsky, Alex and others},
  year={2009}
}

@inproceedings{helmbold1997some,
author = {Helmbold, David and Panizza, Sandra},
title = {Some label efficient learning results},
year = {1997},
isbn = {0897918916},
publisher = {Association for Computing Machinery},
address = {New York, NY, USA},
url = {https://doi.org/10.1145/267460.267502},
doi = {10.1145/267460.267502},
booktitle = {Proceedings of the Tenth Annual Conference on Computational Learning Theory},
pages = {218–230},
numpages = {13},
location = {Nashville, Tennessee, USA},
series = {COLT '97}
}

@InProceedings{kirschner2020information,
  title = 	 {Information Directed Sampling for Linear Partial Monitoring},
  author =       {Kirschner, Johannes and Lattimore, Tor and Krause, Andreas},
  booktitle = 	 {Proceedings of Thirty Third Conference on Learning Theory},
  pages = 	 {2328--2369},
  year = 	 {2020},
  editor = 	 {Abernethy, Jacob and Agarwal, Shivani},
  volume = 	 {125},
  series = 	 {Proceedings of Machine Learning Research},
  month = 	 {09--12 Jul},
  publisher =    {PMLR},
  url = 	 {https://proceedings.mlr.press/v125/kirschner20a.html}
}

@InProceedings{lattimore2019exploration,
  title = 	 {Exploration by Optimisation in Partial Monitoring},
  author =       {Lattimore, Tor and Szepesv{\'a}ri, Csaba},
  booktitle = 	 {Proceedings of Thirty Third Conference on Learning Theory},
  pages = 	 {2488--2515},
  year = 	 {2020},
  editor = 	 {Abernethy, Jacob and Agarwal, Shivani},
  volume = 	 {125},
  series = 	 {Proceedings of Machine Learning Research},
  month = 	 {09--12 Jul},
  publisher =    {PMLR},
  url = 	 {https://proceedings.mlr.press/v125/lattimore20a.html}
}

@inproceedings{komiyama2015regret,
author = {Komiyama, Junpei and Honda, Junya and Nakagawa, Hiroshi},
title = {Regret lower bound and optimal algorithm in finite stochastic partial monitoring},
year = {2015},
publisher = {MIT Press},
address = {Cambridge, MA, USA},
booktitle = {Proceedings of the 29th International Conference on Neural Information Processing Systems - Volume 1},
pages = {1792–1800},
numpages = {9},
location = {Montreal, Canada},
series = {NIPS'15}
}

@inproceedings{vanchinathan2014efficient,
  author       = {Hastagiri P. Vanchinathan and
                  G{\'{a}}bor Bart{\'{o}}k and
                  Andreas Krause},
  editor       = {Zoubin Ghahramani and
                  Max Welling and
                  Corinna Cortes and
                  Neil D. Lawrence and
                  Kilian Q. Weinberger},
  title        = {Efficient Partial Monitoring with Prior Information},
  booktitle    = {Advances in Neural Information Processing Systems 27: Annual Conference
                  on Neural Information Processing Systems 2014, December 8-13 2014,
                  Montreal, Quebec, Canada},
  pages        = {1691--1699},
  year         = {2014},
  url          = {https://proceedings.neurips.cc/paper/2014/hash/0a113ef6b61820daa5611c870ed8d5ee-Abstract.html},
  bibsource    = {dblp computer science bibliography, https://dblp.org}
}

@InProceedings{bartok2012CBPside,
author="Bart{\'o}k, G{\'a}bor
and Szepesv{\'a}ri, Csaba",
editor="Bshouty, Nader H.
and Stoltz, Gilles
and Vayatis, Nicolas
and Zeugmann, Thomas",
title="Partial Monitoring with Side Information",
booktitle="Algorithmic Learning Theory",
year="2012",
publisher="Springer Berlin Heidelberg",
address="Berlin, Heidelberg",
pages="305--319",
isbn="978-3-642-34106-9"
}

@inproceedings{bartokICML2012,
  author       = {G{\'{a}}bor Bart{\'{o}}k and
                  Navid Zolghadr and
                  Csaba Szepesv{\'{a}}ri},
  title        = {An adaptive algorithm for finite stochastic partial monitoring},
  booktitle    = {Proceedings of the 29th International Conference on Machine Learning,
                  {ICML} 2012, Edinburgh, Scotland, UK, June 26 - July 1, 2012},
  publisher    = {icml.cc / Omnipress},
  year         = {2012},
  url          = {http://icml.cc/2012/papers/846.pdf},
  bibsource    = {dblp computer science bibliography, https://dblp.org}
}

@article{bartok2014partial,
author = {Bart\'{o}k, G\'{a}bor and Foster, Dean P. and P\'{a}l, D\'{a}vid and Rakhlin, Alexander and Szepesv\'{a}ri, Csaba},
title = {Partial Monitoring—Classification, Regret Bounds, and Algorithms},
year = {2014},
issue_date = {November 2014},
publisher = {INFORMS},
address = {Linthicum, MD, USA},
volume = {39},
number = {4},
issn = {0364-765X},
url = {https://doi.org/10.1287/moor.2014.0663},
doi = {10.1287/moor.2014.0663},
journal = {Math. Oper. Res.},
month = nov,
pages = {967–997},
numpages = {31}
}

@inproceedings{abbasi2011improved,
  author       = {Yasin Abbasi{-}Yadkori and
                  D{\'{a}}vid P{\'{a}}l and
                  Csaba Szepesv{\'{a}}ri},
  editor       = {John Shawe{-}Taylor and
                  Richard S. Zemel and
                  Peter L. Bartlett and
                  Fernando C. N. Pereira and
                  Kilian Q. Weinberger},
  title        = {Improved Algorithms for Linear Stochastic Bandits},
  booktitle    = {Advances in Neural Information Processing Systems 24: 25th Annual
                  Conference on Neural Information Processing Systems 2011. Proceedings
                  of a meeting held 12-14 December 2011, Granada, Spain},
  pages        = {2312--2320},
  year         = {2011},
  url          = {https://proceedings.neurips.cc/paper/2011/hash/e1d5be1c7f2f456670de3d53c7b54f4a-Abstract.html},
  bibsource    = {dblp computer science bibliography, https://dblp.org}
}

@inproceedings{kleinberg2003value,
  title={The value of knowing a demand curve: Bounds on regret for online posted-price auctions},
  author={Kleinberg, Robert and others},
  booktitle={44th Annual IEEE Symposium on Foundations of Computer Science, 2003. Proceedings.},
  pages={594--605},
  year={2003},
  organization={IEEE}
}

@inproceedings{kveton2019garbage,
  author       = {Branislav Kveton and
                  Csaba Szepesv{\'{a}}ri and
                  Sharan Vaswani and
                  Zheng Wen and
                  Tor Lattimore and
                  Mohammad Ghavamzadeh},
  editor       = {Kamalika Chaudhuri and
                  Ruslan Salakhutdinov},
  title        = {Garbage In, Reward Out: Bootstrapping Exploration in Multi-Armed Bandits},
  booktitle    = {Proceedings of the 36th International Conference on Machine Learning,
                  {ICML} 2019, 9-15 June 2019, Long Beach, California, {USA}},
  series       = {Proceedings of Machine Learning Research},
  pages        = {3601--3610},
  publisher    = {{PMLR}},
  year         = {2019},
  url          = {http://proceedings.mlr.press/v97/kveton19a.html},
  bibsource    = {dblp computer science bibliography, https://dblp.org}
}

@inproceedings{singla2014contextual,
  author       = {Adish Singla and
                  Ian Lienert and
                  G{\'{a}}bor Bart{\'{o}}k and
                  Andreas Krause},
  editor       = {Jeffrey P. Bigham and
                  David C. Parkes},
  title        = {Contextual Procurement in Online Crowdsourcing Markets},
  booktitle    = {Proceedings of the Second {AAAI} Conference on Human Computation and
                  Crowdsourcing, {HCOMP} 2014, November 2-4, 2014, Pittsburgh, Pennsylvania,
                  {USA}},
  pages        = {58--59},
  publisher    = {{AAAI}},
  year         = {2014},
  url          = {https://doi.org/10.1609/hcomp.v2i1.13195},
  doi          = {10.1609/HCOMP.V2I1.13195},
  bibsource    = {dblp computer science bibliography, https://dblp.org}
}

@inproceedings{kossen2021active,
  author       = {Jannik Kossen and
                  Sebastian Farquhar and
                  Yarin Gal and
                  Tom Rainforth},
  editor       = {Marina Meila and
                  Tong Zhang},
  title        = {Active Testing: Sample-Efficient Model Evaluation},
  booktitle    = {Proceedings of the 38th International Conference on Machine Learning,
                  {ICML} 2021, 18-24 July 2021, Virtual Event},
  series       = {Proceedings of Machine Learning Research},
  pages        = {5753--5763},
  publisher    = {{PMLR}},
  year         = {2021},
  url          = {http://proceedings.mlr.press/v139/kossen21a.html},
  bibsource    = {dblp computer science bibliography, https://dblp.org}
}

@article{sherman_morrison,
  title={Adjustment of an Inverse Matrix Corresponding to a Change in One Element of a Given Matrix},
  author={Jack Sherman and Winifred J. Morrison},
  journal={Annals of Mathematical Statistics},
  year={1950},
  volume={21},
  pages={124-127},
  url={https://api.semanticscholar.org/CorpusID:123460064}
}

@article{kirschner2023linear,
author = {Kirschner, Johannes and Lattimore, Tor and Krause, Andreas},
title = {Linear partial monitoring for sequential decision making: algorithms, regret bounds and applications},
year = {2023},
issue_date = {January 2023},
publisher = {JMLR.org},
volume = {24},
number = {1},
issn = {1532-4435},
journal = {J. Mach. Learn. Res.},
month = jan,
articleno = {346},
numpages = {45}
}
\bibliographystyle{icml2024}

\appendix 
\onecolumn 
\section{Implementation details for CBP-based strategies}
\label{appendix:details}

\subsection{Pseudo-code of the randomization procedure}

In Section \ref{par:rand_procedure} we described textually the randomization procedure used in \RandCBP{} and \RandCBPsidestar{}. Algorithm \ref{alg:randomization_procedure} provides the pseudo-code for the described randomization procedure. The parameter $K$ corresponds to the number of bins in the discretized distribution in $[A,B]$; $\sigma$ corresponds to the variance of this distribution and $\epsilon$ to the probability of sampling the value $B$. 

\subsection{Pseudo-code for \CBPsidestar{} and \RandCBPsidestar{}
\label{appendix:pseudo_code} }

Algorithm~\ref{alg:RandCBPside} provides the pseudo-code of \CBPside{} as defined by \citet{lienert2013exploiting} and our proposed \RandCBPsidestar{}. Differences are highlighted in \textcolor{purple}{purple}. The strategies are instantiated with the set of Pareto optimal actions $\mathcal P$ (see Definition \ref{def:cell}), the set of neighbor pairs $\mathcal N$ (see Definition \ref{def:neighbor}), parameters $\eta_a$ for each action, the exploration parameter $\alpha>1$ and the decaying exploring rate $f(t)$. 

\begin{remark}
    Obtaining $\mathcal{P}(t)$ and $\mathcal{N}(t)$ at each round entails solving a computationally expensive optimization problem with evolving constraints. However, by caching the various half-spaces collected over time, the encountered problems can be buffered, significantly enhancing the overall computational complexity of the approach. In practice, Gurobi \citep{gurobi} or PULP \citep{mitchell2011pulp} can be used to solve the optimization problems.
\end{remark}

\begin{remark}
    In the contextual scenario, the update process of the inverse Gram matrix $G_{a,t}$ of action $a$ at time $t$ within \CBPside{} and \RandCBPsidestar{} can be efficiently implemented using the Sherman-Morrison update \citep{sherman_morrison} instead of relying on a costly matrix inversion operation.    
\end{remark}

\begin{algorithm}
\caption{Randomization Procedure}
\label{alg:randomization_procedure}

\SetKwInOut{Input}{Input}
\SetKwInOut{Output}{Output}

\Input{$A, B, K, \varepsilon, \sigma$}
\Output{Z}

\BlankLine

Initialize an array $\rho$ of size $K$ with equally spaced values in $[A,B]$\;

\For{$k \gets 1$ \KwTo $K$}{
    Calculate $\bar{p}_k$ using $\bar{p}_k = \exp\left(-\frac{\rho_k^2}{2 \sigma^2}\right)$\;
}

Initialize an array $p$ of size $K$\;

\For{$k \gets 1$ \KwTo $K-1$}{
    Calculate $p_k$ using $p_k = \frac{(1 - \varepsilon) \bar{p}_k}{\sum_k \bar{p}_k}$\;
}

Sample Z in $\rho$ with probabilities $p$ \;

\end{algorithm}

\begin{algorithm}
\SetKwInOut{Input}{input}
\Input{ $\mathcal P, \mathcal N, \alpha, f(\cdot), \eta_a, \color{purple}{K, \sigma, \varepsilon}  $ }
\DontPrintSemicolon
\caption{\CBPside{}
 \citep{lienert2013exploiting} and  \texttt{\textcolor{purple}{RandCPBside$^\star$}} }
\label{alg:RandCBPside}

$\#$ Notation $e(\cdot)$ is a $\sigma_{I_t}$-dimensional one-hot encoding. \;

Initialize $G_{a,0} = \lambda I_d$ for all $a = \{1, \dots, N\}$ \;

\For{$t = 1, 2, \dots, N$} {
     Receive side-information $x_t$ \;
     
     Play action $I_t = t$ \;
     
     Observe feedback $\textbf{H}[I_t, J_t]$ \;

    $X_{i, t} =X_{i, t-1} \cup \{x_t\} $ if $I_t=i$ else $X_{i, t} =X_{i, t-1}, \forall i$ \;

    $Y_{I_t, t} = Y_{I_t, t-1} \cup \{e(\textbf{H}[I_t, J_t])\}$ if $I_t=i$ else $Y_{i, t} =Y_{i, t-1}, \forall i$ \;

    Compute $G^{-1}_{I_t,t}$ (Sherman-Morrison update, \citet{sherman_morrison}) \;
    
    Update $\hat \theta_i(t) = Y_{i,t} X_{i,t}^\top (\lambda I_d + X_{i,t} X_{i,t}^\top)^{-1}$

     }
\For{$t > N$} {
    Initialize $\mathcal U(t) \gets \{\}$
    Receive side-information $x_t$ \;
    
    \For{$a = 1, \dots,N$} {
         $\hat \pi_a(x_t) = \hat \theta_a x_t$ \;

         \st{  $w_a(x_t) = \sigma_a \left(\sqrt{ d \log(t) + 2 \log(1/t^2) } + \sigma_a \right) \| x_t \|_{G_{a,t}^{-1}} $ }
         
         \textcolor{purple}{ $B = \sqrt{ d \log(t) + 2 \log(1/t^2) }$ } \;
         
         \textcolor{purple}{Sample $Z_{a,t}$, according to Algorithm \ref{alg:randomization_procedure} } \;
         
         \textcolor{purple}{ $w'_a(x_t) = \sigma_a \left(Z_{a,t} + \sigma_a \right) \| x_t \|_{G_{a,t}^{-1}} $  } \;
         
         }
         
    \For{each neighbor pair $\{i,j\} \in \mathcal N$} {

         $\hat \delta_{i,j}(x_t) = \sum_{a \in V_{i,j} } v_{ija}^\top \hat \pi_a(x_t)$ \;
         
         \st{ $c_{i,j}(x_t) \gets \sum_{a\in V_{i,j}} \|v_{ija}\|_2 w_a(x_t)$ } \;
          
          $\color{purple}{ c'_{ij}(x_t) \gets  \sum_{a \in V_{i,j}} \| v_{ija} \|_2 w'_{a}(x_t)  }$  \;
        
        \uIf{$| \hat \delta_{i,j}(x_t) | >$ \st{ $c_{i,j}(x_t)$ }  $\color{purple}{ c'_{i,j}(x_t) } $ }{

         Add pair $\{i,j\}$ to $\mathcal U(t)$ \; }

    }     
    Compute $D(t)$ based on $\mathcal U(t)$ \;
    
    Get $\mathcal P(t)$ and $\mathcal N(t)$ from $\mathcal P$, $\mathcal N$ and $D(t)$ \;
    
    $\mathcal N^{+}(t) \gets \bigcup_{ {ij} \in \mathcal N(t) }  N^{+}_{ij} $ \;
    
    $\mathcal V(t) \gets \bigcup_{ {ij} \in \mathcal N(t) } V_{ij} $ \;
    
    \st{ $\mathcal R(t) \gets \{ a \in \{1, \dots, N\}  : n_a(t) \leq \eta_a f(t) \} $ } \;
    
    $\color{purple}{\mathcal R(x_t) \gets \{ a \in \{1, \dots, N\} : 1/ \| x_t \|^2_{G_{a,t}^{-1}} <   \eta_a f(t)  \} }$  \;
    
    $\mathcal S(t) \gets \mathcal P(t) \cup \mathcal N^{+}(t) \cup ( \mathcal V(t)  \cap \mathcal  R(x_t)   ) $ \;
    
    \st{Play $I_t = \argmax_{ a \in \mathcal S(t) } W_{a} w_a(x_t) $ } \;
    
    \textcolor{purple}{Play} $\color{purple}{ I_t = \argmax_{ a \in \mathcal S(t) } W_{a}  w'_a(x_t) }$ \label{line:decision_rule_RandCBPside}\;
    
    Observe feedback $\textbf{H}[I_t, J_t]$ \;

    $X_{i, t} =X_{i, t-1} \cup \{x_t\} $ if $I_t=i$ else $X_{i, t} =X_{i, t-1}$ \;

    $Y_{I_t, t} = Y_{I_t, t-1} \cup \{e(\textbf{H}[I_t, J_t])\}$ if $I_t=i$ else $Y_{i, t} =Y_{i, t-1}$ \;

    Compute $G^{-1}_{I_t,t}$ (Sherman-Morrison update, \citet{sherman_morrison}) \;
    
    Update $\hat \theta_i(t) = Y_{i,t} X_{i,t}^\top (\lambda I_d + X_{i,t} X_{i,t}^\top)^{-1}$
    }
\end{algorithm}

\section{Partial Monitoring Games}
\label{appendix:games}

In this Appendix, we analyse the Apple Tasting \citep{helmbold2000apple},  Label Efficient \citep{helmbold1997some}, and $\tau$-detection games presented in the main paper. The analysis is necessary to implement partial monitoring environments and strategies based on these games.

\subsection{Characterizing a partial monitoring game}

A game is easy or hard depending on whether it verifies the \textit{global observability} or \textit{local observability} condition. \textit{Easy} games refer to games that are locally observable while \textit{hard} games verify the global observability condition but are not locally observable. 
    
\begin{definition}[Global observability, \citet{piccolboni2001discrete}]
    A partial-monitoring game with $\textbf{L}$ and $\textbf{H}$ admits the \textit{global observability} condition, if all pairs $\{i, j\}$ verify $L_i^\top - L_j^\top \in \oplus_{1 \leq a \leq N} \text{Im} (S_a^\top)$.
    \label{def:global_observability}
\end{definition}
    
\begin{definition}[Local observability, \citet{bartokICML2012}]
    A pair of neighbor actions $i, j$ is \textit{locally observable} if $L_i^\top - L_j^\top \in \oplus_{a \in N^+_{i,j}} \text{Im} (S_a^\top)$. 
    We denote by $\mathcal L \subset \mathcal N$ the set of locally observable pairs of actions (the pairs are unordered). 
    A game satisfies the local observability condition if every pair of neighbor actions is locally observable, i.e., if $\mathcal L = \mathcal N$.
    \label{def:local_observability}
\end{definition}

\begin{remark}
When a pair is locally observable, we have $V_{ij} = N^+_{ij}$. For non-locally observable pairs, $V_{ij} = \{ 1, \dots, N\}$ is always a valid set \citet{bartokICML2012}.    
\end{remark}

\subsection{Apple Tasting game}
\label{appendix:AT}
The Apple Tasting game is defined by the following loss and feedback matrices: 
$$ \bf{L}=\kbordermatrix{
    & \text{A} & \text{B}\\
      \text{action 1}  & 1 & 0 \\
      \text{action 2} & 0 & 1},\quad \bf{H}= \kbordermatrix{
    &\text{A} & \text{B}\\
     \text{action 1} & \bot & \bot \\
     \text{action 2} & \wedge & \odot }. $$

This game has two possible actions and $N=2$ actions and $M=2$ outcomes (denoted $A$ and $B$). 

\paragraph{Signal Matrices:} Signal matrices are such that $S_1 \in \{0,1\}^{1\times2}$ and $S_2 \in \{0,1\}^{2\times2}$. The matrices verify: 
$$ S_1 = \begin{bmatrix} 1 & 1   \end{bmatrix},  \quad
S_2 = \begin{bmatrix} 1 & 0 \\ 0 & 1  \end{bmatrix} $$

The outcome distribution is denoted $p^\star = [p_A, p_B]^\top$. 
\begin{itemize}
   \item $\pi^\star_1 =  S_1 p^\star  =  \begin{bmatrix} 1 & 1   \end{bmatrix}  \begin{bmatrix} p_A \\ p_B   \end{bmatrix}  = 1$, there is only one feedback symbol ($ \perp$) induced by action $1$ therefore the probability of seeing this feedback symbol is always 1.
    \item $\pi^\star_2 =  S_2  p^\star =  \begin{bmatrix} 1 & 0 \\ 0 & 1  \end{bmatrix} \begin{bmatrix} p_A \\ p_B   \end{bmatrix} = \begin{bmatrix} p_A \\ p_B   \end{bmatrix}$ , therefore, the probability of seeing feedback $\wedge$ is $p_A$ and the probability of seeing $\odot$ is $p_B$.
\end{itemize}

\paragraph{Cells:} This game has 2 actions, each associated to a sub-space of the probability simplex: 
\begin{itemize}
    \item For action 1, we have:  $\mathcal C_1 = \{ p \in \Delta_M, \forall j \in \{1, \dots, N\}, (L_1-L_j) p \leq 0 \}$. This probability space corresponds to the following constraints: 
    $$   \begin{bmatrix} L_1-L_1 \\ L_1-L_2  \end{bmatrix} p  =  \begin{bmatrix} 0 & 0 \\ 1 & -1 \end{bmatrix} p  \leq 0 $$ 
    The first constraint $ (L_1-L_1) p \leq 0$ is always verified. The second constraint $ (L_1-L_2) p \leq 0$ implies $p_A - p_B \leq 0$.
    \item For action 2, we have:  $\mathcal C_2 = \{ p \in \Delta_M, \forall j \in \{1, \dots, N\}, (L_2-L_j) p \leq 0 \}$. This probability space corresponds to the following constraints: 
    $$  \begin{bmatrix} L_2-L_1 \\ L_2-L_2  \end{bmatrix} p  =  \begin{bmatrix} -1 & 1 \\ 0 & 0 \end{bmatrix} p  \leq 0 $$
    The second constraint $  (L_2-L_2) p \leq 0$ is always verified. The first constraint $ (L_2-L_1) p \leq 0$ implies $p_B - p_A \leq 0$.
\end{itemize}
Action 1 is optimal when the outcome distribution verifies $p_A - p_B \leq 0$ whereas it is the opposite for action 2.

\paragraph{Pareto optimal actions:} The cell respective to each action is neither empty nor included one in another. Therefore, according to Definition \ref{def:cell}, both actions $1$ and $2$ are Pareto optimal, i.e. $\mathcal P = \{ 1, 2\}$

\paragraph{Neighbor actions:} The space corresponding to $\mathcal{C}_1\cap \mathcal{C}_2$ includes only one unique point, being $\begin{bmatrix} 0.5 & 0.5   \end{bmatrix}$. Therefore, $\dim(C_1 \cap C_2) = 0 = M-2$, which satisfies Definition \ref{def:neighbor}. This implies that actions 1 and 2 are neighbors, i.e. $\mathcal N =\{ \{1, 2 \}, \}$. 

\paragraph{Neighbor action set:} This set includes: $N^+_{12} = N^+_{21} = [1,2]$.

\paragraph{Observability of the game:} We need to calculate $\text{Im}(S_1)$ and $\text{Im}(S_2)$: \begin{itemize}
    \item For $\text{Im}(S_1^\top)$ we have:
 $  \begin{bmatrix} 1 \\ 1  \end{bmatrix} \begin{bmatrix} x & y  \end{bmatrix} =  x + y = \begin{bmatrix} 1   \end{bmatrix} (x + y)  $
    \item For $\text{Im}(S_2^\top)$ we have:
$\begin{bmatrix} 1 & 0 \\ 0 & 1  \end{bmatrix} \begin{bmatrix} x \\  y \end{bmatrix} = \begin{bmatrix} x\\  y \end{bmatrix} = \begin{bmatrix} 1 \\ 0  \end{bmatrix} x + \begin{bmatrix} 0 \\ 1  \end{bmatrix} y $
\end{itemize}

Resulting in: $ \text{Im}(S_1^\top) \bigoplus \text{Im}(S_2^\top) = \text{Span}( \begin{bmatrix} 1 \\  0 \end{bmatrix}, \begin{bmatrix} 0 \\ 1 \end{bmatrix}, [1])$.  \\
The action pair $\{1,2\}$ is locally observable because $L_1^\top-L_2^\top = \begin{bmatrix} 1 &  -1 \end{bmatrix}^\top$ can be expressed from the set of vectors included in $\text{Im}(S_1) \bigoplus \text{Im}(S_2)$. We can conclude that the game is globally and locally observable. Therefore, it can be classified as an \textit{easy game}.

\paragraph{Observer set:} The pair $\{1,2\}$ is locally observable. According to Definition \ref{def:local_observability}, we have: $V_{12} = N_{12}^+ = [1,2]$. The pair of actions $2$ and $1$ is also  locally observable therefore $V_{21} = N_{21}^+ = \{1,2\}$.

\paragraph{Observer vector:} For the pair of actions $1$ and $2$, we have to find $v_{ija}, a \in V_{ij}$ such that $L_1^\top - L_2^\top = \sum_{a \in V_{ij} } S_i^T v_{ija}$, according to Definition \ref{def:observer_vector}. Choosing and $v_{121} = 0$ and $v_{122}^\top = \begin{bmatrix} 1 & -1 \end{bmatrix}$ verifies the relation: 

\begin{equation}
    L_1^\top - L_2^\top =   \begin{bmatrix} 1 \\ -1 \end{bmatrix} =  \langle \begin{bmatrix} 1 \\  1 \end{bmatrix}, 0 \rangle + \begin{bmatrix} 1 & 0 \\ 0 & 1  \end{bmatrix} \begin{bmatrix} 1 \\ -1 \end{bmatrix}  
\end{equation}

It suffices to reproduce the same procedure for pair of actions $2$ and $1$.

\subsection{Label Efficient game}
\label{appendix:LE}

The Label Efficient game \citep{helmbold1997some} is defined by the following loss and feedback matrices: \\

$$\bf{L}=\kbordermatrix{ & \text{A} & \text{B}\\
        \text{action 1} & 1 & 1\\
        \text{action 2} & 1 & 0\\
        \text{action 3} & 0 & 1},\quad \bf{H}= \kbordermatrix{
        & \text{A} & \text{B}\\
        \text{action 1} & \bot & \odot\\
        \text{action 2} & \wedge & \wedge\\
        \text{action 3} & \wedge & \wedge}. $$

The game includes a set of $N=3$ possible actions and $M=2$ possible outcomes (denoted $A$ and $B$).

\paragraph{Signal Matrices:} The dimension of the signal matrices are such that $S_1 \in \{0,1\}^{2\times2}$, $S_2 \in \{0,1\}^{1\times2}$ and $S_3 \in \{0,1\}^{1\times2}$. The matrices verify:  \\
$$ S_1 = \begin{bmatrix} 1 & 0 \\ 0 & 1  \end{bmatrix} , \quad  
S_2 =  \begin{bmatrix} 1 & 1   \end{bmatrix}, \quad S_3 =  \begin{bmatrix} 1 & 1   \end{bmatrix} $$

The outcome distribution is noted $p^\star = [p_A, p_B]^\top$.

\paragraph{Cells:} Each action can be associated to a sub-space of the probability simplex noted \textit{cell} (see Definition \ref{def:cell}): 
\begin{itemize}
    \item For action 1, we have:  $\mathcal C_1 = \{ p \in  \Delta_M, \forall j \in \{1,\dots, N\}, (L_1-L_j) p \leq 0 \}$. 
    This probability space corresponds to the following constraints: 
    $$   \begin{bmatrix} L_1-L_1 \\ L_1-L_2 \\ L_1-L_3 \end{bmatrix}  p    =   \begin{bmatrix} 0 & 0 \\ 0 & 1 \\ 1 & 0 \end{bmatrix}  p    \leq 0 $$
    The first constraint $   (L_1-L_1) p    \leq 0$ is always verified. The second constraint $   (L_1-L_2) p    \leq 0$ implies $ p_B \leq 0$ and the third constraint $   (L_1-L_3) p    \leq 0$ implies $ p_A \leq 0$. 
    There exist no probability vector in $\Delta_M$ satisfying these three constraints at the same time. 
    \item For action 2, we have:  $\mathcal C_2 = \{ p \in  \Delta_M, \forall j \in \{1,\dots, N\}, (L_2-L_j) p \leq 0 \}$. This probability space corresponds to the following constraints: 
    $$   \begin{bmatrix} L_2-L_1 \\ L_2-L_2 \\ L_2-L_3 \end{bmatrix} p    =   \begin{bmatrix} 0 & -1 \\ 0 & 0 \\ 1 & -1 \end{bmatrix} p    \leq 0 $$
    The second constraint $   (L_2-L_2) p    \leq 0$) is always verified. The first constraint $   (L_2-L_1) p    \leq 0$ implies $-p_B \leq 0 \iff p_B \geq 0$. The third constraint $  (L_2-L_3) p    \leq 0$ implies $p_A-p_B \leq 0 \iff p_A \leq p_B$.
    \item For action 3, we have:  
    $\mathcal C_3 = \{ p \in  \Delta_M, \forall j \in \{1,\dots, N\}, (L_3-L_j) p \leq 0 \}$. This probability space corresponds to the following constraints: 
    $$  \begin{bmatrix} L_3-L_1 \\ L_3-L_2 \\ L_3-L_3  \end{bmatrix} p    =   \begin{bmatrix} -1 & 0 \\ -1 & 1 \\ 0 & 0 \end{bmatrix} p    \leq 0 $$
    The third constraint $  (L_3-L_3) p    \leq 0$ is always satisfied. The second constraint $  (L_3-L_1) p   \leq 0$ implies $-p_A+p_B \leq 0 \iff p_B \geq p_A$. The first constraint ($ (L_3-L_1) p \leq 0$) implies $-p_A \leq 0 \iff p_A \geq 0$.
\end{itemize}

\paragraph{Pareto optimal actions:} From the analysis of the cells, we have $\mathcal{C}_{1} = \emptyset$. Therefore, action 1 is dominated, according to Definition \ref{def:cell}. The remaining actions 2 and 3 are Pareto optimal because their respective cells are not included in one another, i.e. $\mathcal P = \{2,3\}$.

\paragraph{Neighboring actions:} In this paragraph, we will determine whether action 2 and 3 are a neighbor pair.

$$ \mathcal{C}_1 \cap \mathcal{C}_2 = \begin{cases}  p_2\geq 0 \\  p_1  \leq p_2 \\  p_2 \leq p_1 \\ p_1 \geq 0 \end{cases} $$
The only point in this vector space is $\begin{bmatrix} 0.5 & 0.5   \end{bmatrix}^\top$. Therefore, $\dim(\mathcal C_1 \cap \mathcal C_2) = 0 = M-2$ and the pair $\{2,3\}$ is a neighbor pair, i.e. $\mathcal N = \{ \{2,3\}, \}$.

\paragraph{Neighborhood action set:} This set is defined as $N^+_{ij} = \{ k \in \{ 1, \dots, N\}, \mathcal C_i \cap \mathcal C_j \subset C_k \}$. This yields: $N^+_{23} = N^+_{32} = [2,3]$ because the cell of action $1$ is empty.

\paragraph{Observability of the game:} 
We will determine whether this game is globally and/or locally observable (see Definitions \ref{def:local_observability} and \ref{def:global_observability}). 
We need to calculate $\text{\text{Im}}(S_1^\top)$, $\text{\text{Im}}(S_2^\top)$ and $\text{\text{Im}}(S_3^\top)$ according to the definition of local and global observability: 
\begin{itemize}
    \item For $\text{\text{Im}}(S_1^\top)$ we have:
$   \begin{bmatrix} 1 & 0 \\ 0 & 1  \end{bmatrix},  \begin{bmatrix} x \\  y \end{bmatrix}    =  \begin{bmatrix} x \\  y \end{bmatrix}  =  x \begin{bmatrix} 1 \\  0 \end{bmatrix}  + y  \begin{bmatrix} 0 \\ 1 \end{bmatrix}  $
    \item For $\text{\text{Im}}(S_2^\top) = \text{\text{Im}}(S_3^\top)$ we have:
$   \begin{bmatrix} 1 \\ 1  \end{bmatrix}, \begin{bmatrix} x & y  \end{bmatrix}    =  x + y = \begin{bmatrix} 1   \end{bmatrix} (x + y)  $
\end{itemize}

We have $ \bigoplus_{1 \leq i \leq N} \text{Im}(S_i^\top) = \text{Span}( \begin{bmatrix} 1 \\  0 \end{bmatrix}, \begin{bmatrix} 0 \\ 1 \end{bmatrix}, [1])$.
 
The pair $\{2,3\}$ is not locally observable because it is not possible to express $L_2^\top-L_3^\top$ from $ \bigoplus_{i \in N_{23}^+} \text{Im}(S_i^\top)= \text{Span}( [1])$. On the contrary, it is possible to express $L_2^\top-L_3^\top$ from $ \bigoplus_{1 \leq i \leq N} \text{Im}(S_i^\top) = \text{Span}( \begin{bmatrix} 1 \\  0 \end{bmatrix}, \begin{bmatrix} 0 \\ 1 \end{bmatrix}, [1])$. We can conclude that the game is not locally observable and that the pair $\{2,3\}$ is globally observable. Therefore, the Label Efficient game belongs to the class of hard games. 

\paragraph{Observer set:} The pair $\{2,3\}$ is not locally observable. According to Definition \ref{def:local_observability}, we have: $V_{23} = \{ 1, \dots, N\}$ same applies to $V_{32} = \{ 1, \dots, N\}$.

\paragraph{Observer vector:} For the pair $\{2,3\}$, we have to find $v_{ija}, a \in V_{ij}$ such that $L_2^\top - L_3^\top = \sum_{a \in V_{ij} } S_i^T v_{ija}$, according to Definition \ref{def:observer_vector}. Choosing and $v_{231}^\top = \begin{bmatrix} -1 & 1 \end{bmatrix}$, $v_{232} = 0$ and $v_{233} = 0$ verifies the relation: 

\begin{equation}
    L_2^\top - L_3^\top =   \begin{bmatrix} -1 \\ 1 \end{bmatrix} =    \begin{bmatrix} 1 & 0 \\ 0 & 1  \end{bmatrix} \begin{bmatrix} -1 \\ 1 \end{bmatrix}    +   \begin{bmatrix} 1 \\ 1 \end{bmatrix} 0     +   \begin{bmatrix} 1 \\ 1 \end{bmatrix} 0    
\end{equation}

It suffices to reproduce the same procedure for pair of actions $3$ and $2$.

\subsection{$\tau$-detection game}
\label{appendix:detection_game}
Let us consider the $\tau$-detection game, with $\tau \in ]0,1[$. The game is defined by the following loss and feedback matrices: 
$$ \bf{L}=\kbordermatrix{
    & \text{A} & \text{B}\\
      \text{action 1}  & 1 & 1 \\
      \text{action 2} & 1/\tau &  0 },\notag  \\
    \bf{H}= \kbordermatrix{
    &\text{A} & \text{B}\\
     \text{action 1} & \wedge & \odot  \\
     \text{action 2} & \bot & \bot  }.$$

This game includes a set of $N=2$ possible actions and $M=2$ possible outcomes (denoted $A$ and $B$).

\paragraph{Signal Matrices:} The dimension of the signal matrices are such that $S_1 \in \{0,1\}^{2\times2} $ and $S_2 \in \{0,1\}^{1\times2}$. The matrices verify: 
$$ S_1 = \begin{bmatrix} 1 & 0 \\ 0 & 1  \end{bmatrix} , \quad  
S_2 =  \begin{bmatrix} 1 & 1   \end{bmatrix} $$

Consider a general instance of the problem where the outcome distribution is $p^\star = [p_A, p_B]^\top$. 

\paragraph{Cells:} This game has two actions, each can be associated to a cell: 
\begin{itemize}
    \item For action 1, we have:  $\mathcal C_1 = \{ p \in  \Delta_M, \forall j \in \{ 1, \dots, N \}, (L_1-L_j) p \leq 0 \}$. 
    This probability space corresponds to the following constraints: 
    $$  \begin{bmatrix} L_1-L_1 \\ L_1-L_2  \end{bmatrix}  p  =  \begin{bmatrix} 0 & 0 \\ 1-1/\tau & 1 \end{bmatrix}  p  \leq 0 $$
    The first constraint $  (L_1-L_1) p  \leq 0$ is always verified. The second constraint $  (L_1-L_2) p  \leq 0$ implies $ p_A (2-1/\tau) \leq 1$. 
    \item For action 2, we have:  $\mathcal C_2 = \{ p \in  \Delta_M, \forall j \in \{ 1, \dots, N \}, (L_2-L_j) p \leq 0 \}$. This probability space corresponds to the following constraints: 
    $$  \begin{bmatrix} L_2-L_1 \\ L_2-L_2  \end{bmatrix} p  = \begin{bmatrix} 1/\tau -1 & -1 \\ 0 & 0 \end{bmatrix} p  \leq 0 $$
    The second constraint $ (L_2-L_2) p \leq 0$) is always verified. The first constraint $  (L_2-L_1) p \leq 0$ implies $p_A \leq \tau$.
\end{itemize}

\paragraph{Pareto optimal actions:} The cell respective to each action is neither empty nor included one in another. Therefore, according to Definition \ref{def:cell}, both actions $1$ and $2$ are Pareto optimal, i.e. $\mathcal P = \{ 1, 2\}$

\paragraph{Neighboring actions:} For values of $\tau \in ]0,1[$, $\mathcal{C}_1\cap \mathcal{C}_2 = \emptyset$. Therefore, $\dim(C_1 \cap C_2) = 0$, which satisfies the definition \ref{def:neighbor}. This implies that actions 1 and 2 are neighboring actions, i.e. $\mathcal N =\{ \{1, 2 \}, \}$.

\paragraph{Neighborhood action set:} This set is defined as $N^+_{ij} = \{ k  \in \{ 1, \dots, N \}, \mathcal C_i \cap \mathcal C_j \subset C_k \}$. This yields: $N^+_{12} = N^+_{21} = [1,2]$.

\paragraph{Observability of the game:} In this paragraph, we will determine whether this game is globally and/or locally observable (see Definitions \ref{def:local_observability} and \ref{def:global_observability}). We need to calculate $\text{Im}(S_1^\top)$ and $\text{Im}(S_2^\top)$, according to the definition of local and global observability. \begin{itemize}
    \item For $\text{Im}(S_1^\top)$ we have:
$  \begin{bmatrix} 1 & 0 \\ 0 & 1  \end{bmatrix}  \begin{bmatrix} x \\  y \end{bmatrix}  =  \begin{bmatrix} x \\  y \end{bmatrix}  =  x \begin{bmatrix} 1 \\  0 \end{bmatrix}  + y  \begin{bmatrix} 0 \\ 1 \end{bmatrix}  $
    \item For $\text{Im}(S_2^\top)$ we have:
 $  \begin{bmatrix} 1 & 1  \end{bmatrix} \begin{bmatrix} x \\ y  \end{bmatrix}  = x + y = (x+y) \times [1]  $
\end{itemize}
Resulting in: $ \text{Im}(S_1^\top) \bigoplus \text{Im}(S_2^\top) = \text{Span}( \begin{bmatrix} 1 \\  0 \end{bmatrix}, \begin{bmatrix} 0 \\ 1 \end{bmatrix}, [1])$.  \\
The action pair $\{ 1, 2 \}$ is locally observable because $L_1^\top-L_2^\top = \begin{bmatrix} 1-1/\tau & 1  \end{bmatrix}$ can be expressed from the set of basis vectors included in $\text{Im}(S_1) \bigoplus \text{Im}(S_2)$ (see Definition \ref{def:local_observability}). Since this also applies to the pair $\{2,1\}$, we can conclude that the game is globally and locally observable. Therefore, it can be classified as an \textit{easy game}.

\paragraph{Observer set:} The pair $\{1,2\}$ is locally observable. According to the definition \ref{def:observer_set}, we have: $V_{12} = N_{12}^+ = [1,2]$. The pair $\{2,1\}$ being also locally observable, we have $V_{21} = N_{21}^+ = \{1,2\}$.

\paragraph{Observer vector:} For the pair $\{1,2\}$, we have to find $v_{ija}, a \in V_{ij}$ such that $L_1^\top - L_2^\top = \sum_{a \in V_{ij} } S_i^T v_{ija}$. Choosing and $v_{121} = 0$ and $v_{122}^\top = \begin{bmatrix} 1 & -(1-b_{opt}) \end{bmatrix}$ verifies the relation: 

\begin{equation}
    L_1^\top - L_2^\top =   \begin{bmatrix} 1-1/\tau \\ 1  \end{bmatrix} =  \begin{bmatrix} 1 \\  1 \end{bmatrix} 0  + \begin{bmatrix} 1 & 0 \\ 0 & 1  \end{bmatrix} \begin{bmatrix} 1 \\ -(1-b_{opt}) \end{bmatrix}, 
\end{equation}
where $b_{opt}$ satisfies the constraint $1/b_{opt}-\tau = 0$

\section{Regret analysis of \RandCBP{} }
\label{appendix:regretRandCBP}
In this section, we provide an upper bound on the expected regret of \RandCBP{}. 
The incidence of randomization on \textit{Upper confidence bound} strategies was characterized by \citet{kveton2019garbage} and \citet{vaswani2019old}. 
CBP-based strategies belong instead to the class of \textit{Successive Elimination} strategies, which utilize both upper and lower confidence bounds.

Let $\delta_{i} = \max_{1 \leq j \leq N} \delta_{ij}$ be the sub-optimality gap between the expected loss of action $i$ and the optimal action. Similarly to \citet{bartok2012CBPside}, define $g_i$ as
\begin{equation}
g_i = \max_{ \mathcal P', \mathcal N' \in \Psi, i \in \mathcal P' } \min_{ \pi \in B_i(\mathcal N'), \pi = (i_0, \dots, i_r) } \sum_{s=1}^r | V_{i_{s-1},i_s} |   
\end{equation}
where $\Psi$ corresponds to the set of plausible configurations and $B_i(\mathcal N')$ the set of possible paths. The quantity $g_i$ is correlated with the number of actions $N$.

\subsection{Regret decomposition of RandCBP}
Assuming action $1$ is optimal:
\begin{align}
  \mathbb{E} [ R(T) ]  & = \mathbb{E} [ \sum_{ t = 1 }^T (L_{I_t} - L_{1}) p^\star   ]  \\
                       & = \sum_{ k = 1 }^N \mathbb{E} [ n_k(T) ] \delta_k 
\end{align}
\noindent
The goal is to bound $\mathbb{E} [ n_k(T) ]$. 
Define the event $\mathcal E_{t}$: "the confidence interval succeeds"\footnote{We reverse the notation used in \citet{bartokICML2012}.}. Formally,
$\mathcal E_{t} = \{ | \hat \delta_{i,j}(t) - \delta_{i,j} | \leq c_{i,j}(t) \}$.
The event $\mathcal E_{t}$ induces the following decomposition:
\begin{align}
  \mathbb{E} [ n_k(T) ]  & = \mathbb{E} [ \sum_{ t = 1 }^T \mathds{1}_{ \{ I_t=k \} } ]   \\
  & = \underbrace{  \mathbb{E} [ \sum_{ t = 1 }^T \mathds{1}_{ \{ I_t=k, \mathcal E_{t}  \} } ] }_{A_k} + \underbrace{  \mathbb{E} [ \sum_{ t = 1 }^T \mathds{1}_{\{ I_t=k,  \mathcal E^c_{t} \} } ]  }_{B_k} 
\end{align}
The regret can thus be expressed as:
\begin{equation}
    \mathbb{E}[ R(T) ] = \sum_{k=1}^{N} \delta_k A_k + \delta_k B_k 
\label{eq:last_step_proof}
\end{equation}

To obtain an upper bound on the regret of  \RandCBP{}, we need to upper bound the terms $A_k$ and $B_k$. The bound of $A_k$ is reported in Section \ref{subsection:boundingA}. The bound of $B_k$ is reported in Section \ref{subsection:boundingB}. 
The theorem that follows is obtained by combining Eq. \ref{eq:last_step_proof} and the analyses from Sections Section \ref{subsection:boundingA} and \ref{subsection:boundingB}.

\begin{theorem}

Consider the randomization over $K$ bins in the interval $[A,B]$, a probability $\epsilon$ on the tail and a standard deviation $\sigma$. Setting $\eta_a = W_k^{2/3}$, $ f(t) = \alpha^{ 1/3}t^{2/3} \log^{1/3}(t)$ and, with the notations $W = \max_{1\leq a \leq N} W_a$, $\mathcal V = \bigcup_{i,j \in \mathcal N} V_{i,j}$, and $N^+ =\bigcup_{i,j \in \mathcal N} N^+_{i,j}$, we obtain:
\begin{dmath}
        \mathbb{E}[ R_T ] \leq \quad \sum_{1 \leq k \leq N} \left[ 2 (1 + \frac{1}{2\alpha-2}) | \mathcal V | + 1 \right] +
                                \sum_{ k = 1 }^N \delta_k +  \\ \sum_{ k = 1, \delta_k>0 }^N  4 W_k^2 \frac{g_k^2}{\delta_k}\alpha \log(T) + \\
                                \sum_{k \in \mathcal V \backslash N^+} \delta_k \min (  4 W_k^2 \frac{g^2_{l(k)} }{ \delta^2_{l(k) } } \alpha \log(T),  \alpha^{1/3} W_k^{2/3} T^{2/3} \log^{1/3}(T)  ) + \\ 
                                \sum_{k \in \mathcal V \backslash N^+} \delta_k  \alpha^{1/3} W_k^{2/3} T^{2/3} \log^{1/3}(T)  + 2 g_k  \alpha^{1/3} W^{2/3} T^{2/3} \log^{1/3}(T)
\end{dmath}
\label{thm:non_contextual}
\end{theorem}

On easy games, we have $ \mathcal V \backslash N^+ = \emptyset$.
The theorem implies a bound on the individual regret of \RandCBP{} on easy games: 

\begin{corollary} 
Consider an easy game, and the same assumptions as in Theorem \ref{thm:non_contextual}. Then: 
$$ \mathbb{E}[ R_T ] \leq \quad N \left[ 2 (1 + \frac{1}{2\alpha-2})  | \mathcal V| + 1 \right] + \sum_{ k = 1 }^N \delta_k + \sum_{ k = 1, \delta_k>0 }^N 4 W_k^2 \frac{g_k^2}{\delta_k}\alpha \log(T). $$
\label{thm:easy_games_randcbp}
\end{corollary}
Corollary \ref{thm:easy_games_randcbp} matches the upper bound on the regret of \CBP{} on the time horizon \cite{bartokICML2012}. 
The first term corresponds to the confidence interval of the failure event. The second term comes from the initialization phase of the algorithm. The third term comes from the exploration-exploitation trade-off achievable on easy games.

\begin{corollary} 
Consider a hard game and the same assumptions as in Theorem \ref{thm:non_contextual}. Then, there exists a constant $C_1$ and $C_2$ such that the expected regret can be upper bounded independently of the choice of $p^\star$ as
$$ \mathbb{E}[ R_T ] \leq C_1 N + C_2 T^{2/3} \log^{1/3}(T)  $$
\end{corollary}

The regret bound of \RandCBP{} on hard games matches \CBP{}'s on hard games on the time horizon \cite{bartokICML2012}. 
Note that the bound on hard games is problem-independent unlike the bound on easy games.

\subsection{Bounding $A_k$}
\label{subsection:boundingA}

This part is quite similar to that of \citet{bartokICML2012}, except that the underlying Lemma \ref{lemma1} has been adapted for the randomized confidence bounds. 
We include the steps for completeness. 

The notation $I_t$ corresponds to the action that was effectively played at round $t$.
Define $k(t) = \argmax_{i \in \mathcal P(t) \cup V(t)} W_i^2/n_i(t)$. 
The event $k(t) \neq I_t$ happens when $k(t) \notin N^+(t)$ and $k(t) \notin \mathcal R(t)$, i.e. $k(t)$ is a purely information seeking (exploratory) action which has been sampled frequently. This corresponds to the event $\mathcal D_t = \{ k(t) \neq I_t \}$ = \text{"the decaying exploration rule is in effect at time t" }.

We can decompose:

\begin{dmath}
    \mathbb{E}[ \sum_{t=1}^{T} \mathds{1}_{ \{ I_t=k,\mathcal E_{t} \} } ] \delta_k  \leq  \delta_k + \\ 
    \underbrace{ \mathbb{E}[ \sum_{t=N+1}^{T} \mathds{1}_{ \{ \mathcal E_t, \mathcal D_t^c, k \in \mathcal{P}(t)\cup N^+(t), I_t=k \} } ]  \delta_k }_{A_1} + \\
    \underbrace{ \mathbb{E}[ \sum_{t=N+1}^{T} \mathds{1}_{ \{ \mathcal E_t, \mathcal D_t^c, k \notin \mathcal{P}(t)\cup N^+(t), I_t=k \} } ] \delta_k}_{A_2}  + \\
    \underbrace{ \mathbb{E}[ \sum_{t=N+1}^{T} \mathds{1}_{ \{ \mathcal E_t, \mathcal D_t, k \in \mathcal{P}(t)\cup N^+(t), I_t=k \} } ] \delta_k}_{A_3} + \\
    \underbrace{ \mathbb{E}[ \sum_{t=N+1}^{T} \mathds{1}_{ \{ \mathcal E_t, \mathcal D_t, k \notin \mathcal{P}(t)\cup N^+(t), I_t=k \} } ] \delta_k}_{A_4}
\end{dmath}
The first $\delta_k$ corresponds to the initialization phase of the algorithm when every action is chosen once. The next paragraphs are devoted to upper bounding the remaining four expressions $A_1, A_2, A_3$ and $A_4$, using the results from Lemma \ref{lemma1}. Note that, if action $k$ is optimal, then $\delta_k = 0$, so all the terms are zero. Thus, we can assume from now on that $\delta_k>0$.

\paragraph{Term $A_1$:} Consider the event $\mathcal E_t \cap \mathcal D_t^c \cap \{ k \in \mathcal P(t) \cup N^+(t) \}$. Using case 2 from Lemma \ref{lemma1} with the choice $k = i$. Thus, from $I_t = i$, we get that $I_t = i = k \in \mathcal{P}(t)\cup N^+(t)$. The result of the lemma gives:
$$ n_k(t) \leq A_k(t) = 4 W_k^2 \frac{g_k^2}{\delta_k^2}\alpha \log(t) $$
Therefore, we have
\begin{align}
    \sum_{t=N+1}^{T} \mathds{1}_{ \{ \mathcal E_t, \mathcal D_t^c, k \in \mathcal{P}(t)\cup N^+(t), I_t=k \} }  \\
    \leq &\quad \sum_{t=N+1}^{T} \mathds{1}_{ \{ I_t=k, n_k(t) \leq A_k(t) \} } + \sum_{t=N+1}^{T} \mathds{1}_{ \{ \mathcal E_t, \mathcal D_t, k \notin \mathcal{P}(t)\cup N^+(t), I_t=k, n_k(t)>A_k(t) \} }\\
    = &\quad \sum_{t=N+1}^{T} \mathds{1}_{ \{ I_t=k,n_k(t)\leq A_k(t) \} } \\
    \leq &\quad A_k(T) = 4 W_k^2 \frac{g_k^2}{\delta_k^2}\alpha \log(T) 
\end{align}

Consequently, 
\begin{align}
    \sum_{t=N+1}^{T} \mathds{1}_{ \{ \mathcal E_t, \mathcal D_t^c, k \in \mathcal{P}(t)\cup N^+(t), I_t=k \} } \delta_k 
    \leq &\quad  4 W_k^2 \frac{g_k^2}{\delta_k}\alpha \log(T)
\end{align}

\paragraph{Term $A_2$:} Consider the event $\mathcal E_t \cap \mathcal D_t^c \cap \{ k \notin \mathcal P(t) \cup N^+(t) \}$. From case 2 of Lemma \ref{lemma1}. The Lemma gives:
$$ n_k(t) \leq \min_{ j \in \mathcal P(t) \cup N^+(t)} 4 W_k^2 \frac{g_j^2}{\delta_j}\alpha \log(T)$$
We know that $k \in \mathcal V(t)= \bigcup_{ k,j \in \mathcal N(t)  } V_{i,j}$. 
Let $\Phi_t$ be the set of pairs $\{i,j\}$ in $\mathcal N(t) \subseteq \mathcal N$ such that $k \in V_{i,j}$. 
For any $\{i,j\} \in \Phi_t$, we also have that $i,j \in \mathcal P(t)$ and thus if $l'_{\{i,j\} } = \argmax_{l \in \{i,j \} } \delta_l$ then:

$$ n_k(t-1) \leq 4 W_k^2 \frac{g^2_{l'_{ \{i,j\} }} }{ \delta^2_{l'_{  \{i,j\} }} } \alpha \log(t)$$

If we define $l(k)$ as the action with
$$ \delta_{l(k) } = \min \{ \delta_{l'_{ \{i,j\} }}: \{i,j \} \in \mathcal N, k \in V_{ij} \}$$

Then, it follows that:
$$ n_k(t-1) \leq 4 W_k^2 \frac{g^2_{l(k)} }{ \delta^2_{l(k) } } \alpha \log(t) $$

Note that $\delta_{l(k)}$ can be zero and thus we use the convention $c/0 = \infty$. Also, since $k$ is not in $ \mathcal P(t) \cup N^+(t)$, we have that $n_k(t-1) \leq \eta_k f(t)$. 
Define $A_k(t)$ as:

$$ A_k(t) =  \min \Bigl\{  4 W_k^2 \frac{g^2_{l(k)} }{ \delta^2_{l(k) } } \alpha \log(t),  \eta_k f(t) \Bigl\}$$

Then, with the same argument as in the previous case (and recalling that $f(t)$ is increasing), we get:
$$ \mathbb{E}[ \sum_{t=N+1}^{T} \mathds{1}_{ \{ \mathcal E_t, \mathcal D_t^c, k \notin \mathcal{P}(t)\cup N^+(t), I_t=k \} } ] \leq \delta_k \min \Bigl\{  4 W_k^2 \frac{g^2_{l(k)} }{ \delta^2_{l(k) } } \alpha \log(t),  \eta_k f(t) \Bigl\}$$
 
\paragraph{Term $A_3$:} Consider the event $ \mathcal E_t \cap D_t \cap \{ k \in \mathcal P(t) \cup N^+(t) \} $. From Lemma \ref{lemma1} we have that:
$$ \delta_k \leq 2 g_k \sqrt{ \frac{\alpha \log(T) }{ f(t) } } \max_{1 \leq l \leq N } \frac{W_l}{\sqrt{\eta_l}}  $$
Thus, 
$$ \mathbb{E}[ \sum_{t=N+1}^{T} \mathds{1}_{ \{ \mathcal E_t, \mathcal D_t, k \in \mathcal{P}(t)\cup N^+(t), I_t=k \} } ]  \leq g_k \sqrt{ \frac{\alpha \log(T) }{ f(T) } } \max_{ 1 \leq l \leq N } \frac{W_l}{\sqrt{\eta_l}}$$

\paragraph{Term $A_4$:}  Consider the event $ \mathcal E_t \cap D_t \cap \{ k \notin \mathcal P(t) \cup N^+(t) \} $ we know that $k \in \mathcal V(t) \cap \mathcal R(t) \subseteq \mathcal R(t)$ and hence $n_k(t)\leq \eta_k f(t)$. With the same argument as in the first and second term, we get that:
$$ \mathbb{E}[ \sum_{t=N+1}^{T} \mathds{1}_{ \{ \mathcal E_t, \mathcal D_t, k \notin \mathcal{P}(t)\cup N^+(t), I_t=k \} } ] \leq \delta_k \eta_k f(T) $$ 

\subsection{Bounding term $B_k$: }
\label{subsection:boundingB}

In the analysis of $B_k$, the goal is to upper-bound the probability that the confidence interval fails.
For the deterministic \CBP{}, this corresponds to Lemma 1 in \citet{bartokICML2012}.
\RandCBP{} uses instead randomized confidence intervals.
Following the terminology in \citet{vaswani2017attention}, we use \textit{uncoupled randomized confidence intervals} because we sample a value for each action pair.

For a pair of actions $\{i,j\} \in \mathcal  N$, at a time $t$, note $ Q_{ij}( t )$ the probability that the confidence interval of pair $\{i,j\}$ fails: 
\begin{align} 
    Q_{i,j}(t)  & = \quad  \mathbb P_{ Z_{ijt} }( \{ \delta_{i,j} < \hat \delta_{i,j}(t) - c_{i,j}'(t) \}  \cup  \{ \delta_{i,j} > \hat \delta_{i,j}(t) + c_{i,j}'(t) \} ) \\
                & = \quad \mathbb P_{ Z_{ijt} }( | \hat \delta_{i,j}(t) - \delta_{i,j} | > c_{i,j}'(t) ) 
\end{align}
\noindent
The event $\mathcal E^c_t$ is unlikely to occur when $T$ is large; let 
$$\Upsilon_{k} = \{ t \in [T], \forall \{i,j\} \in \mathcal N, Q_{i,j}(t) > \frac{1}{T} \}$$ 
be the set of time steps where the probability of failure is non-negligible, i.e. is higher than $1/T$. Following \citet{kveton2019garbage}, the regret can be decomposed according to $\Upsilon_k$:
\begin{align}
  \mathbb{E} [ \sum_{ t = 1 }^T \mathds{1}_{\{ I_t=k,  \mathcal E^c_{t} \} } ] & =  \mathbb{E} [ \sum_{ t \in \Upsilon_{k} } \mathds{1}_{ \{ I_t=k \} } ] + \mathbb{E} [ \sum_{ t \notin \Upsilon_{k} } \mathds{1}_{ \{ \mathcal E^c_{t} \} } ]  \\
  & \leq    \mathbb{E} [ \sum_{ t = 0 }^{T} \sum_{ \{i,j\} \in \mathcal N } \mathds{1}_{ \{  Q_{i,j}(t) > \frac{1}{T}   \} } ] + \mathbb{E} [ \sum_{ t \notin \Upsilon_{k} } \frac{1}{T} ] \\
  & \leq  \sum_{ t = 0 }^{T}  \sum_{ \{i,j\} \in \mathcal N } \mathbb P_{\hat \delta_{i,j}(t)} (   Q_{i,j}(t) > \frac{1}{T}  ) + 1   
\end{align}
\noindent
For a given pair $\{i,j\} \in \mathcal N$, and for a specific time $t$, define:
\begin{align}
    b_{i,j}(t) & = \quad \mathbb P_{\hat \delta_{i,j}(t)} \left[ Q_{i,j}(t) > \frac{1}{T} \right] \\
            & = \quad \mathbb P_{\hat \delta_{i,j}(t)} \left[ \mathbb P_{ Z_{ijt} }(|\hat \delta_{i,j}(t) - \delta_{i,j}| \geq c_{i,j}'(t) ) > \frac{1}{T} \right] 
\end{align}
\noindent
By definition of $Z_{ijt}$ (that are sampled from a discrete probability distribution) we have:
\begin{align}
    b_{i,j}(t) & = \quad \mathbb P_{\hat \delta_{i,j}(t) } \left[  \mathbb P_{Z_{ijt}}(|\hat \delta_{i,j}(t) - \delta_{i,j}| \geq c'_{i,j}(t) ) > \frac{1}{T} \right] \\
            & = \quad \mathbb P_{\hat \delta_{i,j}(t) } \left[ \sum_{k=1}^K p_k \mathds{1}_{ \{ |\hat \delta_{i,j}(t) - \delta_{i,j}| \geq c_{i,j}^{k}(t)   \} } > \frac{1}{T} \right]
\end{align}
where $c_{i,j}^{k}(t)$ denotes the confidence interval associated to the sampled value $\rho_k$. Since $p_K > \frac{1}{T}$, we have:
\begin{align}
    b_{i,j}(t) & = \quad \mathbb P_{\hat \delta_{i,j}(t) }(  |\hat \delta_{i,j}(t) - \delta_{i,j}| \geq c_{i,j}^{K}(t)  ) \\
            & = \quad \mathbb P_{\hat \delta_{i,j}(t) }( |\hat \delta_{i,j}(t) - \delta_{i,j}| >  \sum_{a \in V_{i,j} }|| v_{ija} ||_{\infty} \frac{\rho_K}{ \sqrt{ n_a(t) } }   ) \\
            & \leq \quad \sum_{a \in V_{i,j} } \sum_{s = 1}^t  \mathbb P_{\hat \delta_{i,j}}( \hat \delta_{ij}(s) - \delta_{i,j} >  || v_{ija} ||_{\infty} \frac{\rho_K}{ \sqrt{ l } }  ) \mathds{1}_{ \{ n_a(t) = s \} }  \\
            & \leq \quad \sum_{a \in V_{i,j} } \sum_{s = 1}^t  2 \exp(- 2 s (|| v_{ija} ||_{\infty} \frac{\rho_K}{ \sqrt{ s } })^2   ) \mathds{1}_{ \{ n_a(t) = s \} } \label{eqn:hoefding} \\
            & \leq \quad \sum_{a \in V_{i,j} } \sum_{s = 1}^t  2 \exp(- 2 || v_{ija} ||_{\infty}^2 \rho_K^2   ) \mathds{1}_{ \{ n_a(t) = s \} } \\
            & \leq \quad \sum_{a \in V_{i,j} } 2  \exp(- 2\rho_K^2   ) \sum_{s = 1}^t \mathds{1}_{ \{ n_a(t) = s \} } \label{eqn:bound_norm} \\
            & \leq \quad \sum_{a \in V_{i,j} }  2 \exp(- 2\rho_K^2   )  \\
            & \leq \quad 2 |V_{i,j}| \exp(- 2\rho_K^2   )  \\
\end{align}
Where the Hoeffding's inequality was used in \ref{eqn:hoefding}. Therefore, 
\begin{align}
    B_{k}      & \leq \quad  \sum_{t=1}^T \sum_{ \{i,j\} \in \mathcal N } b_{i,j}(t)  + 1 \\
               & \leq \quad  \sum_{t=1}^T \sum_{ \{i,j\} \in \mathcal N } 2 |V_{i,j}| \exp(- 2\rho_K^2 )  + 1 \\
               & \leq \quad 2 | \mathcal V | \exp(- 2\rho_K^2 ) T  + 1 
\intertext{The linear dependency on T is cancelled with $\rho_K = \sqrt{ \alpha \log(T) }$ and for $\alpha>1$, we have:}
               & \leq \quad  2 (1 + \frac{1}{2\alpha-2}) | \mathcal V | + 1,
\end{align}
where $\mathcal V = \bigcup_{i,j \in \mathcal N} V_{i,j}$.

\subsection{Proofs of lemmas}

\begin{lemma}
Fix any $t\geq 1 $. \begin{enumerate}
    \item Take any action $i$. On the event $\mathcal E_t \cap \mathcal D_t $, from $i \in \mathcal P(t) \cap N^+(t)$ it follows that
    $$ \delta_i \leq 2 g_i \sqrt{ \frac{\alpha \log(t)}{f(t)} } \max_{1 \leq k \leq N} \frac{W_k}{\sqrt{\eta_k}} $$ 
    \item Take any action k. On the event $\mathcal E_t \cap \mathcal D_t^c $, from $I_t=i$ it follows that
    $$ n_k(t-1) \leq \min_{j \in \mathcal P(t) \cup N^+(t) } 4 W_k^2 \frac{g_j^2}{\delta_j^2}\alpha \log(t)$$ 
\end{enumerate}
\label{lemma1}
\end{lemma}

\begin{proof}

Observe that for any neighboring action pair $\{ i,j\} \in \mathcal N(t)$, on $\mathcal E_t$, it holds that $\delta_{i,j}(t) \leq 2 c'_{i,j}(t)$. Indeed, from ${i,j} \in \mathcal N(t)$ it follows by definition of the algorithm that $\tilde \delta_{i,j}(t) \leq c'_{i,j}(t) $. Now, from the definition of $\mathcal E_t$, we observe $\delta_{i,j}(t) \leq \tilde \delta_{i,j}(t) + c'_{ij}(t)$. Putting together the two inequalities, we get $\delta_{i,j}(t) \leq 2 c'_{i,j}(t) \leq 2 c_{i,j}(t)$.

Now, fix some action $i$ that is not dominated. We define the \textit{parent action} $i'$ of $i$ as follows: If $i$ is not degenerate then $i' = i$. If $i$ is degenerate then we define $i'$ to be the Pareto-optimal action such that $\delta_{i'} \geq \delta_i$ and $i$ is in the neighborhood action set of $i'$ and some other Pareto-optimal action. It follows from \citet{bartokICML2012} that $i'$ is well-defined.

\paragraph{Case 1}
Consider case 1. Recall that $k(t) = \argmax_{j \in \mathcal P(t) \cup \mathcal V(t)} \frac{W_j^2}{n_j(t)}$. Thus, $I_t \neq k(t)$. Consequently, $k(t) \notin \mathcal R(t)$, i.e.  $n_{k(t)}(t) > \eta_{k(t)} f(t)$. Assume now that $i \in \mathcal P(t) \cup N^+(t)$. If $i$ is degenerate, then $i'$ as defined in the previous paragraph is in $\mathcal P(t)$ (because the rejected regions in the algorithm are closed). In any case, we know from \citet{bartokICML2012} that there exists a path $(i_0,...,i_r) $ in $\mathcal N(t)$ that connects $i'$ to $i^*$ ($i^* \in \mathcal P(t)$ holds on $\mathcal E_t$). We have that:
\begin{align}
    \delta_i \leq &\quad \delta_{i'} = \sum_{s=1}^{r} \delta_{i_{s-1}, i_s}  \\
    \leq &\quad  2 \sum_{s=1}^{r} c'_{i_{s-1}, i_s}    \\
    \leq &\quad  2 \sum_{s=1}^{r} c_{i_{s-1}, i_s} \\
    \leq &\quad  2 \sum_{s=1}^{r} \sum_{a \in V_{i_{s-1}, i_s} } \| v_{i_{s-1},v_{i_{s}},a } \|_{\infty} \sqrt{ \frac{\alpha \log(t)}{n_a(t) } }   \\
    \leq &\quad  2 \sum_{s=1}^{r} \sum_{a \in V_{i_{s-1}, i_s} } W_a  \sqrt{ \frac{\alpha \log(t)}{n_a(t)} }   \\
    \leq &\quad  2 g_i W_{k(t)} \sqrt{ \frac{\alpha \log(t) }{n_{k(t)}(t) } }   \\
    \leq &\quad  2 g_i W_{k(t)} \sqrt{ \frac{\alpha \log(t) }{\eta_{k(t)}f(t) } }   
\end{align}
Upper bounding $W_{k(t)}/\sqrt{\eta_{k(t)} }$ by $\max_{1 \leq k \leq N} W_k/ \sqrt{\eta_k}$ we obtain the desired bound. \\

\paragraph{Case 2:}
Now, for case 2 take an action k, consider $\mathcal E_t \cap \mathcal D_t^c$, and assume that $I_t = k$. On the event $D_t^c$, we have that $I_t = k(t)$. Thus, from $I_t = k$ it follows that $W_k /  \sqrt{ n_k(t) } \geq W_j /  \sqrt{ n_j(t) }$ holds true for all $j \in \mathcal P(t)$. Let $J_t = \argmin_{ j \in \mathcal P (t) \cup N^+(t) } \frac{g_j^2}{\delta_j^2} $. Now, similarly to the previous case, there exists a path $(i_0,...,i_r)$ from the parent action $J_{t'} \in \mathcal P(t)$ of $J_t$ to $i^\star \in \mathcal N(t)$. Hence, 
\begin{align}
    \delta_{J_t} \leq &\quad \delta_{J'_t} = \sum_{s=1}^{r} \delta_{i_{s-1}, i_s}  \\
    \leq &\quad  2 \sum_{s=1}^{r} c'_{i_{s-1},  i_s}  \\
     \leq &\quad  2 \sum_{s=1}^{r} c_{i_{s-1}, i_s}   \\
    \leq &\quad  2 \sum_{s=1}^{r} \sum_{a \in V_{i_{s-1}, i_s} } W_a \sqrt{ \frac{\alpha \log(t)}{n_a(t)} } \\
    \leq &\quad  2 g_{J_t} W_k \sqrt{ \frac{\alpha \log(t)}{n_k(t)} }
\end{align}
This implies 
\begin{align}
    n_k(t-1) \leq &\quad 4 W_k^2 \frac{d_{J_t}^2}{\delta_{J_t}^2} \alpha \log(t) \\
    = &\quad  4 W_k^2 \min_{j \in \mathcal P(t) \cup N^+(t) } \frac{d_{j}^2}{\delta_{j}^2} \alpha \log(t)  
\end{align}

This concludes the proof of the Lemma.
\end{proof}

\section{Regret analysis of \RandCBPsidestar{} }
\label{appendix:regret_CBPside}

In this Section, we provide an upper bound on the expected regret of \RandCBPsidestar{}. Consider the problem of partial monitoring with linear side information \citep{bartok2012CBPside}. Let $\delta_{i}(x) = \max_{1\leq j \leq N} \delta_{i,j}(x)$ be the sub-optimality gap between the expected loss of action $i$ and the optimal action given the context $x$. Define $\Psi =  \max_{ 1 \leq a \leq N  }{ \sigma_a } $ as the maximum number of feedback symbols that can be induced by an action in the game.

Similarly to the proof in \citep{bartokICML2012}, consider the events $\mathcal D_t$ = \textit{"the decaying exploration rule is in effect at time t"} and $\mathcal E_t$ = \textit{"the confidence interval succeeds at time t"} =  $\{ | \hat \delta_{i,j}(x_t) - \delta_{i,j}(x_t) | \leq c_{i,j}(x_t) \}$\footnote{The notation is inversed in \citet{bartokICML2012}. }.

\subsection{Lemma: the confidence interval succeeds }
\label{proof_lemma1_contextual}

\begin{lemma}
Fix any $t\geq 1 $.  Take any action i. On the event $\mathcal E_t \cap \mathcal D_t $, from $i \in \mathcal P(t) \cap N^+(t)$ it follows that
\begin{equation}
    \delta_i(x_t) \leq  \frac{2 g_i \Psi  \left(\sqrt{ d \log(t )  } + \Psi \right) }{ \sqrt{ f(t) } }  \max_{1 \leq k \leq N}  \frac{  W_k }{ \sqrt{\eta_k }}
\end{equation}    

\label{lemma:lemmarandCBPside}
\end{lemma}

\begin{proof}
We start the proof with the following remarks:

\begin{remark}
    Observe that for any neighboring action pair $\{i, j\} \in \mathcal N(t)$, on $\mathcal E_t$, it holds that $\delta_{i,j}(x_t) \leq 2 c'_{i,j}(x_t)$. 
    Indeed, from ${i,j} \in \mathcal N(t)$ it follows by definition of the algorithm that $\tilde \delta_{i,j}(x_t) \leq c'_{i,j}(x_t) $. Furthermore, we have: $\delta_{i,j}(x_t) \leq \tilde \delta_{i,j}(x_t) + c'_{i,j}(x_t)$, by definition of $\mathcal E_t$. Putting together the two inequalities, and given that of $c'_{i,j}(x_t) \leq c_{i,j}(x_t)$, we obtain $\delta_{i,j}(x_t) \leq 2 c'_{i,j}(x_t) \leq 2 c_{i,j}(x_t)$. 
\label{lemma2:remark}
\end{remark}

\begin{remark}
Now, fix some action $i$ that is not \textit{dominated}\footnote{see definition \ref{def:cell}}. 
We define the \textit{parent action} $i'$ of $i$ as follows: If $i$ is not degenerate then $i' = i$. If $i$ is degenerate then $i'$ is the Pareto-optimal action such that $\delta_{i'}(x_t) \geq \delta_i(x_t)$ and $i$ is in the neighborhood action set of $i'$ and some other Pareto-optimal action. 
It follows from Lemma 5 in \citet{bartokICML2012} that $i'$ is well-defined.     
\end{remark}

Define the action $k(t) =  \argmax_{j \in \mathcal P(t) \cup \mathcal V(t)} W_j w_j(t)$. In other words, $k(t)$ represents the action that has the largest confidence width within the set $\mathcal P(t) \cup \mathcal V(t)$, which corresponds to the exploitation component of the strategy.

Consider $\mathcal E_t \cap \mathcal D_t$. 
Due to $\mathcal D_t$, the played action $I_t$ is such that $I_t \neq k(t) $. 
Therefore, $k(t) \notin \mathcal R(x_t)$ which implies  $ \| x_t\|_{G^{-1}_{k(t),t}  } \leq \frac{1}{ \sqrt{ \eta_{k(t)} f(t) } } $ from the definition of $\mathcal R(x_t)$ in the contextual setting. 
Assume now that $i \in \mathcal P(t) \cup N^+(t)$. 
If $i$ is degenerate, then $i'$ as defined in the previous paragraph is in $\mathcal P(t)$. 
In any case, there is a path $(i_0,...,i_r) $ in $\mathcal N(t)$ that connects $i'$ to $i^*$, with $i^* \in \mathcal P(t)$ that holds on $\mathcal E_t$. We have that:
\begin{align}
    \delta_i(x_t) \leq &\quad \delta_{i'}(x_t) = \sum_{s=1}^{r} \delta_{i_{s-1}, i_s}(x_t) \label{eqn:case1:parent_action} \\
    \leq &\quad  2 \sum_{s=1}^{r} c'_{i_{s-1}, i_s}(x_t) \label{eqn:case1:conf} \\
    \leq &\quad  2 \sum_{s=1}^{r}  c_{i_{s-1}, i_s}(x_t) \label{eqn:case1:deterministic_ub}  \\
    = &\quad  2 \sum_{s=1}^{r}  \sum_{a \in V_{i_{s-1}, i_s} } \| v_{ i_{s-1}i_{s}a } \|_{\infty} \sigma_a \left(\sqrt{ d \log(t) + 2 \log(1/\delta_t) } + \sigma_a \right) \| x_t \|_{G_{a,t}^{-1} }  \label{eqn:case1:expand_formula}  \\
    \leq &\quad  2 \sum_{s=1}^{r}  \sum_{a \in V_{i_{s-1}, i_s} } W_{k(t)} \Psi \left(\sqrt{ \Psi \log(t) + 2 \log(1/\delta_t)  } + \Psi \right) \| x_t \|_{G_{k(t),t}^{-1} }  \label{eqn:case1:upper_bound_formula}  \\
    \leq &\quad  2 g_i W_{k(t)} \Psi \left( \sqrt{ d \log(t ) + 2 \log(1/\delta_t)  } + \Psi \right)   \| x_t \|_{G_{ k(t), t}^{-1} } \label{eqn:case1:cardinality} \\
    \leq &\quad  2 g_i W_{k(t)} \Psi \left(\sqrt{ d \log(t ) + 2 \log(1/\delta_t)  } + \Psi \right) \frac{ 1 }{ \sqrt{  \eta_{k(t)} f(t)  } } \label{eqn:case1:forced_exploration} \\
    \leq &\quad  \frac{2 g_i \Psi  \left(\sqrt{ d \log(t ) + 2 \log(1/\delta_t)  } + \Psi \right) }{ \sqrt{ f(t) } }  \max_{ 1 \leq k \leq N}  \frac{  W_k }{ \sqrt{\eta_k }} \label{eqn:case1:finalize},
\end{align}

Equation \ref{eqn:case1:parent_action} was derived from the definition of a parent action.
Equations \ref{eqn:case1:conf} and \ref{eqn:case1:deterministic_ub} follow from remark \ref{lemma2:remark}. 
In Equation \ref{eqn:case1:expand_formula}, we expand the formula of the confidence bound, defined in Section \ref{sec:contextual_setting}.

In Equation \ref{eqn:case1:cardinality}, we simplify the double sum by using the fact that $\| v_{ i_{s-1},i_{s},a } \|_{\infty}$ is upper bounded by $W_{k(t)}$ and that the cardinality of the double sum is $g_i$.
In Equation \ref{eqn:case1:forced_exploration} we use the upper bound on the Gram matrix obtained from the events considered in the Lemma. In Equation \ref{eqn:case1:finalize}, we finalize the upper-bound by considering the action in $\{ 1, \dots, N\}$ that maximizes $\frac{  W_k }{ \sqrt{\eta_k }}$.

This concludes the proof of the Lemma.
\end{proof}

\subsection{Bounding the sum of sub-optimality gaps}
\label{proof_lemma2_contextual}

The goal of this section is to establish an upper bound for the sum of sub-optimality gaps, specifically under the event denoted as $\mathcal E_t$, which signifies the success of the confidence interval.

\CBPside{}, as presented by \citet{bartok2012CBPside}, utilizes confidence bounds that are tailored for easy games exclusively. 
\RandCBPsidestar{} adopts a broader definition of confidence bounds, as originally introduced by \citet{bartokICML2012} and \citet{lienert2013exploiting}. 
This broader definition makes \RandCBPsidestar{} applicable to both easy and hard games.

\begin{lemma}
    When $\mathcal E_t$ holds, the sum of the sub-optimality gaps can be upper-bounded by
$$ \sqrt{ \sum_{s=1}^{n_k(T)} \delta_k(x_{t_k(s)})^2 }   \leq  2 g_k W_k  \Psi d^{1/2} \left(\sqrt{ d \log( T ) + 2 \log(1/\delta_t) } + \Psi \right) \sqrt{  2 \log(T) }, $$
where $n_k(t)$ is the total number of times action $k$ was played up to time $t$ and $t_k(s)$ is the round index where action $k$ was played for the $s$-th time.
\label{lemma:delta_upper_bound}

\end{lemma} 

\begin{proof}

Recall that $\delta_{I_t}(x_t)$ corresponds to the gap between action $k$ and the optimal action given context $x_t$. 
There exist a path of $r$ neighboring actions $I_t = k_0, k_1, ..., k_r = i^\star(x_t)$ between the action played and the optimal action. 
This sequence always exists thanks to how the algorithm constructs the set of admissible actions \footnote{for a proof of this statement, refer to \citet{bartok2012CBPside}. }. 
The first step of the proof consists in upper-bounding the sub-optimality gap:

\begin{align}
  \delta_{k}(x_t)^2 \leq  & \quad \left( \sum_{s = 1}^{r} 2 c'_{k_{s-1},k_s}(x_t) \right)^2 \label{eqn:step1} \\
                    \leq  & \quad \left( \sum_{s = 1}^{r} 2 c_{k_{s-1},k_s}(x_t) \right)^2 \\
                    \leq & \quad 4 \left(  \sum_{s=1}^{r}  \sum_{a \in V_{k_{s-1}, k_s} } \| v_{ k_{s-1},k_{s},a } \|_{\infty} \sigma_a  \left(\sqrt{ d \log(t) + 2 \log(1/\delta_t) } + \sigma_a \right) \| x_t \|_{G_{a,t}^{-1} } \right)^2  \\
                    \leq & \quad 4 \left(  \sum_{s=1}^{r}  \sum_{a \in V_{k_{s-1}, k_s} } \| v_{ k_{s-1},k_{s},a } \|_{\infty} \sigma_a  \left(\sqrt{ d \log(t) + 2 \log(1/\delta_t)  } + \sigma_a \right) \| x_t \|_{G_{a,t}^{-1} } \right)^2  \\
                    \leq & \quad 4 \left(  \sum_{s=1}^{r}  \sum_{a \in V_{k_{s-1}, k_s} } W_k \Psi  \left(\sqrt{ d \log(t ) + 2 \log(1/\delta_t) } + \Psi \right) \max_{1 \leq l \leq N}\| x_t \|_{G_{l,t}^{-1} } \right)^2  \label{eqn:final_step} \\
                    \leq & \quad  4 \left(  g_{k} W_k \Psi  \left(\sqrt{ d \log(t ) + 2 \log(1/\delta_t) } + \Psi \right) \max_{1 \leq l \leq N} \| x_t \|_{G_{l,t}^{-1} } \right)^2  \label{eqn:upper_max} \\
                    \leq & \quad  4 g_{k}^2 W_k^2 \Psi^2  \left(\sqrt{ d \log(t ) + 2 \log(1/\delta_t) } + \Psi \right)^2 \max_{1 \leq l \leq N} \| x_t \|^2_{G_{l,t}^{-1} } \label{eqn:final_upper_bound} 
\end{align}
For the detail between Equation \ref{eqn:step1} and Equation \ref{eqn:final_step}, we refer the reader to the steps described in the proof of Lemma \ref{lemma:lemmarandCBPside}. In Equation \ref{eqn:upper_max} we consider the greatest weighted norm over the action space $\{1, \dots, N\}$ to be able to remove it from the double sum.

We now analyse the square root of the sum of the sub-optimality gaps over the time horizon of the action $k$. We start with the result obtained in Equation \ref{eqn:final_upper_bound}:

\begin{align}
 \sqrt{ \sum_{t=1}^{n_k(T)} \delta_k(x_{t_k(s)})^2 }   \leq  & \quad  \sqrt{ \sum_{s=1}^{n_k(T)}  \min( 4 g_k^2 W_k^2 \Psi^{2}  \left(\sqrt{ \log( t_k(s) ) + 2 \log(1/\delta_s)  } + \Psi \right)^2  \max_{1 \leq l \leq N} \| x_{t_k(s)} \|^2_{G_{l, t_k(s)}^{-1} }, 1) } \\
                                \leq & \quad  \sqrt{   4 g_k^2 W_k^2 \Psi^2 \left(\sqrt{ d \log( T ) + 2 \log(1/\delta_T)  } + \Psi \right)^2 \sum_{s=1}^{n_k(T)}  \min(  \max_{1 \leq l \leq N} \| x_t \|_{G_{l, t_k(s) }^{-1} }^2, 1) } \\
                                \leq & \quad 2 g_k W_k  \Psi \left(\sqrt{ d \log( T ) + 2 \log(1/\delta_T) } + \Psi \right) \sqrt{  2 d \log( 1+n_k(T) E^2) } \label{eqn:abbassi_upperbound} \\
                                \in & \quad O( 2 g_k W_k  \Psi d^{1/2} \left(\sqrt{ d \log( T ) + 2 \log(1/\delta_T)  } + \Psi \right) \sqrt{  2  \log(T) } )\label{eqn:CBPside_final_result}  
\end{align}
In Equation \ref{eqn:abbassi_upperbound} we have used the upper bound on the sum of weighted norms presented in Lemma 10 of \citet{abbasi2011improved}, with the assumption $\| x_t\|_2 \leq E$. 
The difference between Line \ref{eqn:CBPside_final_result} and Equation 6 in \citet{bartok2012CBPside} is that a $\sqrt{T}$ term is not appearing. We will see that the $\sqrt{T}$ term appears appears later in the analysis from the Cauchy-Schwartz inequality. 
\end{proof}

\subsection{Regret analysis of \RandCBPside{} using Lemma \ref{lemma:lemmarandCBPside} }
\label{proof_contextual_regret}

In this Section, we analyse the regret rate of \RandCBPsidestar{} on easy and hard games. 
The initial strategy \CBPside{} \citep{bartok2012CBPside} has a guarantee restricted to easy games. 
The key component to obtain the guarantee of \RandCBPside{} to hard games is to define underplayed actions in a suitable way for the contextual setting. 

\begin{proof}

First, we decompose the regret around the event $\mathcal E_t$ and its complimentary: 

\begin{align}
\mathbb{E}[R_T] &= \sum_{t=1}^T \mathbb{E}[\textbf{L}[I_t,J_t]] - \sum_{t=1}^T \mathbb{E}[\textbf{L}[i^\star(x_t),J_t]] \\
&= \sum_{t=1}^{T} \mathbb{E}[\delta_{I_t}(x_t)] \nonumber \\
&= \sum_{t=1}^T \mathbb{E}[\mathds{1}_{\{\mathcal{E}_t\}} \delta_{I_t}(x_t)] + \mathbb{E}[\mathds{1}_{\{\mathcal{E}_t^c\}} \delta_{I_t}(x_t)] \\
&= \underbrace{\sum_{t=1}^T \mathbb{E}[\mathds{1}_{\{\mathcal{E}_t^c\}} \delta_{I_t}(x_t)]}_{A} + \underbrace{\sum_{t=1}^T \mathbb{E}[\mathds{1}_{\{\mathcal{E}_t\}} \delta_{I_t}(x_t)]}_{B}      
\end{align}

\subsection{Term A}
In this Section, we will study component A. 
Assume for each action $k$ at time $t$, there exist a number such that $p( \mathcal E_t^c, I_k = k ) \leq \beta_{k,t}$. Therefore, there exist a sequence of numbers $\beta_{k,1}, \beta_{k,2}, \dots, \beta_{k,T} \in [0,1]$.
These numbers can be seen as some probabilities that $\mathcal E^c$ occurs. 
In the previous analysis \citep{bartok2012CBPside} the numbers where action independent. 
In this work, the numbers are action dependent i.e. we add a dependency on $k$ because the strategy \RandCBPsidestar{} generates a sample $Z_{k,t}$ for each action which influences the value of $\beta_{k,t}$.
Define $a(t) = \argmax_{1 \leq k \leq N} \beta_{k,t}$:

\begin{dmath}
\sum_{ t = 1 }^T \mathbb{E} [ \mathds{1}_{ \{ \mathcal E^c_t  \} }  \delta_{I_t}(x_t)  ]  = \quad  \sum_{ t = 1 }^T \sum_{ k = 1 }^N \mathbb{E} [ \mathds{1}_{ \{ \mathcal E^c_t, k = I_t  \} }  \delta_{k}(x_t)  ]   \\
             \leq \quad \sum_{ t = 1 }^T \sum_{ k = 1 }^N \mathbb{E} [ \mathds{1}_{ \{ \mathcal E^c_t, k = I_t  \} }  ], \text{ because $\delta_k(x_t)\leq1$}  \\
            = \quad \sum_{ t = 1 }^T \sum_{k=1}^N \beta_{k,t}  \\
            \leq \quad \sum_{ t = 1 }^T N \beta_{a(t),t}  
\end{dmath}

\subsection{Term B:}

Consider a specific action $k$. The regret decomposition is decomposed into multiple components. depending whether $\mathcal D_t$ occurs or not. This decomposition was initially presented in \citet{bartokICML2012} in the non-contextual case. Here, we adapt the decomposition to the contextual case, as demonstrated by the presence of contextual sub-optimality gaps $\delta_k(x_t)$.

\begin{dmath}
\mathbb{E} [  \sum_{ t = 1 }^T   \mathds{1}_{ \{ \mathcal E_t, I_t = k  \} } \delta_{k}(x_t) ]   =   \quad   
   \delta_k(x_t) \label{eq:elemA} +  \\
   \underbrace{ \sum_{ t = N+1 }^T \mathbb{E} [ \mathds{1}_{ \{ \mathcal E_t, D_t, k \in P(t) \cup N^+(t), I_t=k \} }    ]  \delta_k(x_t)  }_{B_1} + \\
   \underbrace{ \sum_{ t = N+1 }^T \mathbb{E} [ \mathds{1}_{ \{ \mathcal E_t, D_t, k \notin P(t) \cup N^+(t), I_t=k \} }   ] \delta_k(x_t)  }_{B_2} +\\
   \underbrace{ \sum_{ t = N+1 }^T \mathbb{E} [ \mathds{1}_{ \{ \mathcal E_t, D_t^c, k \in P(t) \cup N^+(t), I_t=k \} }   ]  \delta_k(x_t) }_{B_3} +\\
   \underbrace{ \sum_{ t = N+1 }^T \mathbb{E} [ \mathds{1}_{ \{ \mathcal E_t, D_t^c, k \notin P(t) \cup N^+(t), I_t=k \} }  ] \delta_k(x_t) }_{B_4}  \\
\end{dmath}

The first term corresponds to the regret suffered at the initialization of the algorithm, where each action is played once.  We will now focus on bounding the terms $B_1$, $B_2$, $B_3$, and $B_4$.

\paragraph{Term $B_1$:} 
Consider the case $\mathcal E_t \cap D_t \cap \{ k \in P(t) \cup N^+(t) \} $.  From case 1 of Lemma \ref{lemma:lemmarandCBPside}, we have the relation: 
\begin{equation}
    \delta_k(x_t) \leq \frac{2 g_k \Psi  \left(\sqrt{ d \log(t ) + 2 \log(1/\delta_t) } + \Psi \right) }{ \sqrt{ f(t) } }  \max_{1 \leq j \leq N}  \frac{  W_j }{ \sqrt{\eta_j }}
\end{equation}

\begin{dmath}
\sum_{ t = N+1 }^T \mathbb{E} [ \mathds{1}_{ \{ \mathcal E_t, D_t, k \in \mathcal P(t) \cup N^{+}(t), I_t = k \} }  \delta_{k}(x_t)  ]  \leq  \quad  \\ 
T  \frac{2 g_k \Psi  \left(\sqrt{ d \log(T) + 2 \log(1/\delta_T) } + \Psi \right) }{ \sqrt{ f(T) } }  \max_{ 1 \leq j \leq N}  \frac{  W_j }{ \sqrt{\eta_j } }
\end{dmath}

\paragraph{Term $B_2$:} 
Consider the case $\mathcal E_t \cap D_t \cap \{ k \notin P(t) \cup N^+(t) \} $. It follows that $k \in \mathcal V(t) \cap \mathcal R(x) \subseteq \mathcal R(x)$. Hence, we know by definition of the exploration rule that $ 1/\|x\|^2_{G_{k,t}^{-1}} < \eta_k f(t) $.

\begin{dmath}
\sum_{ t = N+1 }^T \mathds{1}_{ \{ \mathcal E_t, D_t, k \notin \mathcal P(t) \cup N^{+}(t), I_t = k \} }      \leq \quad \sum_{ t = N+1 }^T \mathds{1}_{ \{ I_t = k, 1/\|x_t\|^2_{G_{k,t}^{-1}} < \eta_k f(t) \} } + \sum_{ t = N+1 }^T \mathds{1}_{ \{ \mathcal E_t, D_t, k \notin \mathcal P(t) \cup N^{+}(t), I_t = k, 1/\|x_t\|^2_{ G_{k,t}^{-1} } \geq \eta_k f(t)} \}   \label{eqn:B2_indicator}\\
 \leq \quad \eta_{k} f(T) 
\end{dmath}
In Equation \ref{eqn:B2_indicator} there are two antagonist indicators. The second one simplifies to 0 because the inequality is never verified due to $D_t$. We now apply the Cauchy-Schwartz inequality:

\begin{align}
\sum_{ t = N+1 }^T \mathds{1}_{ \{ \mathcal E_t, D_t, k \notin \mathcal P(t) \cup N^{+}(t), I_t = k \} } \delta_k(x_t) \leq & \quad \sqrt{ \left( \sum_{ t = N+1 }^T \mathds{1}_{ \{ \mathcal E_t, D_t, k \notin \mathcal P(t) \cup N^{+}(t), I_t = k \} }^2 \right) \left( \sum_{ t = N+1 }^T  \delta_k(x_t)^2 \right) } \label{eqn:b2_start}\\
\leq & \quad  \left( \sqrt{ \sum_{ t = N+1 }^T \mathds{1}_{ \{ \mathcal E_t, D_t, k \notin \mathcal P(t) \cup N^{+}(t), I_t = k \} } } \right) \left( \sqrt{ \sum_{ t = N+1 }^T  \delta_k(x_t)^2 } \right) \label{eqn:b2_sqrt} \\
 \leq & \quad \sqrt{ \eta_{k} f(T) } \left( \sqrt{ \sum_{ t = N+1 }^T  \delta_k(x_t)^2 } \right) \label{eqn:caseB2}
\end{align}

In Equation \ref{eqn:b2_start} we use the relation $ \left(\sum_i a_i b_i \right)^2 \leq \left( \sum_i a_i^2 \right) \left( \sum_i b_i^2 \right) $. In Equation \ref{eqn:b2_sqrt} we notice that the square of an indicator is equal to the indicator. We also use the relation $\sqrt{ab}\leq \sqrt{a} \sqrt{b}$.

\paragraph{Term $B_3$:} 
Consider the event $\mathcal E_t \cap \mathcal D_t^c \cap \{ k \in \mathcal P(t) \cup N^+(t) \}$. We will use the Cauchy-Schwartz inequality to simplify the expression.

\begin{align}
\sum_{ t = N+1 }^T \mathds{1}_{ \{ \mathcal E_t, \mathcal D_t^c, k \in \mathcal{P}(t)\cup N^+(t), I_t=k \} } \delta_k(x_t) \leq & \quad \sqrt{ \left( \sum_{ t = N+1 }^T \mathds{1}_{ \{ \mathcal E_t, \mathcal D_t^c, k \in \mathcal{P}(t)\cup N^+(t), I_t=k \} }^2 \right) \left( \sum_{ t = N+1 }^T  \delta_k(x_t)^2 \right) } \\
\leq & \quad  \left( \sqrt{ \sum_{ t = N+1 }^T \mathds{1}_{ \{ \mathcal E_t, \mathcal D_t^c, k \in \mathcal{P}(t)\cup N^+(t), I_t=k \} } } \right) \left( \sqrt{ \sum_{ t = N+1 }^T  \delta_k(x_t)^2 } \right)  \\
 \leq & \quad \sqrt{ T } 2 g_k W_k \Psi d^{1/2} \left(\sqrt{ d \log( T ) + 2 \log(1/\delta_T)  } + \Psi \right) \sqrt{  2  \log(T) }  \\
 \leq & \quad \sqrt{ T } 2 g_k W_k \Psi d^{1/2} \left(\sqrt{ d \log( T ) + 2 \log(1/\delta_T)  } + \Psi \right) \sqrt{  2  \log(T) } 
\end{align}

\paragraph{Term $B_4$:}
Consider the event $\{ \mathcal E_t, \mathcal D^c_t, k \notin \mathcal P(t) \cup \N^+(t) \}$.

Since $k$ is not in $ \mathcal P(t) \cup N^+(t)$, we also have that $ \|x \|^2_{G_{k,t}^{-1}}   \leq \frac{1}{  \eta_k f(t) } \iff \frac{1}{  \|x \|^2_{G_{k,t}^{-1}} }    \geq  \eta_k f(t) $. 

We get:
\begin{equation}
    \mathbb{E}[ \sum_{t=N+1}^{T} \mathds{1}_{ \{ \mathcal E_t, \mathcal D_t^c, k \notin \mathcal{P}(t)\cup N^+(t), I_t=k \} } ] \leq \min \Bigl\{ T,   \eta_k f(T)  \Bigl\}    
\end{equation}
We now use Cauchy-Schwartz,
\begin{align}
\sum_{ t = N+1 }^T \mathds{1}_{ \{ \mathcal E_t, \mathcal D_t^c, k \notin \mathcal{P}(t)\cup N^+(t), I_t=k \} } \delta_k(x_t) \leq & \quad \sqrt{ \left( \sum_{ t = N+1 }^T \mathds{1}_{ \{ \mathcal E_t, \mathcal D_t^c, k \notin \mathcal{P}(t)\cup N^+(t), I_t=k \} }^2 \right) \left( \sum_{ t = N+1 }^T  \delta_k(x_t)^2 \right) } \\
\leq & \quad  \left( \sqrt{ \sum_{ t = N+1 }^T \mathds{1}_{ \{ \mathcal E_t, \mathcal D_t^c, k \notin \mathcal{P}(t)\cup N^+(t), I_t=k \} } } \right) \left( \sqrt{ \sum_{ t = N+1 }^T  \delta_k(x_t)^2 } \right)  \\
 \leq & \quad \sqrt{ \min \Bigl\{ T,   \eta_k f(T)  \Bigl\}  } \left( \sqrt{ \sum_{ t = N+1 }^T  \delta_k(x_t)^2 } \right)\label{eqn:caseB4}
\end{align}

\end{proof}

\subsection{Conclusion:}

The following theorem is an individual upper bound on the regret of \RandCBPside{}.

\begin{theorem}
Consider the interval $[A,B]$, with $B = \sqrt{ d \log(t) + 2 \log(1/t^2) } $ and $A\leq0$.
Set the randomization over $K$ bins with a probability $\epsilon$ on the tail and a standard deviation $\sigma$. 
Let $f(t) = \alpha^{1/3} t^{2/3} \log(t)^{1/3}$, $\eta_a = W_a^{2/3}$ and $\alpha>1$.
Assume $\| x_t\|_2 \leq E$ and positive constants $C_1$, $C_2$, $C_3$, and $C_4$. Note $W = \max_{1 \leq k \leq N} W_k$.  

\begin{dmath}
  \mathbb{E} [ R(T) ]   \leq  \quad     \sum_{ t = 1 }^T N \beta_{a(t),t} + N + \\   
                                        \sum_{ 1 \leq k \leq N } \left( \sqrt{ \sum_{ s = 1 }^{n_k(T)}  \delta_k(x_{t_k(s)} )^2 } \right) \sqrt{T} + \\  
                                        \sum_{ k \in \mathcal V \backslash N^+ } \left( \sqrt{ \sum_{ s = 1 }^{n_k(T)}  \delta_k(x_{t_k(s)} )^2 } \right) \left( \sqrt{ \eta_k f(T) } + \sqrt{ \min \Bigl\{ T,   \eta_k f(T)  \Bigl\} } \right) + \\ 
                                        \sum_{ k \in \mathcal V \backslash N^+ } T  \frac{2 g_k \Psi  \left(\sqrt{ d \log(T) + 2 \log(1/\delta_T) } + \Psi \right) }{ \sqrt{ f(T) } }  W^{2/3} ,\label{eqn:regret_RandCBPside}
\end{dmath}
\label{thm:contextual_result}
\end{theorem}

where $\mathcal V = \bigcup_{i,j \in \mathcal N} V_{i,j}$, and $N^+ =\bigcup_{i,j \in \mathcal N} N^+_{i,j}$.

\paragraph{Result on easy games:}
On easy games, the set $k \in \mathcal V \backslash N^+$ is empty which simplifies the expression in Equation \ref{eqn:regret_RandCBPside}. The regret rate can be expressed as:
\begin{align}
   \mathbb{E} [ R(T) ]   \leq & \quad \sum_{ t = 1 }^T N \beta_{a(t),t} + N +  N \sqrt{ T } 2 g_k W_k  \Psi d^{1/2} \left(\sqrt{ d \log( T ) + 2 \log(1/\delta_T)  } + \Psi \right) \sqrt{  2  \log(T) } 
\end{align}

\begin{corollary}
Consider an easy game and $\delta_t = 1/t^2$, and the same assumptions as in theorem \ref{thm:contextual_result}, there exist constants $C_1$ and $C_2$ such that the expected regret of \RandCBPside{} on this game can be upper bounded independently of the choice of $p^\star$ as:
$$ \mathbb{E}[ R_T ]\leq C_1 N + C_2 N d \sqrt{T} \log(T) $$
\end{corollary}

The guarantee of \CBPside{} on easy games proposed in \citet{bartok2012CBPside} is $C_1 N + C_2 N^{3/2} d^2 \sqrt{T} \log{T}$.  Here, the dependency drops from $d^2$ to $d$ simply because we corrected the confidence bound formula, but this result should also apply to \CBPside{}.

\paragraph{Result on hard games:}
On hard games, the set $\mathcal V \backslash N^+$ is not empty.

We need to study the terms of the regret expression to identify which one dominates. The regret expression is:

\begin{align}
  \mathbb{E} [ R(T) ] & \leq \sum_{ t = 1 }^T N \beta_{a(t),t} + N + \nonumber \\
  & \sum_{ k \in \underbar N } \sqrt{ T } 2 g_k W_k  d^{3/2} \left(\sqrt{ d \log( T ) + 2 \log(1/\delta_T)  } + \Psi \right) \sqrt{  2  \log(T) }  + \nonumber \\
  & \sum_{ k \in \mathcal V \backslash N^+ } \left( \sqrt{\eta_k f(T) } + \min \Bigl\{ \sqrt{T}, \sqrt{\eta_k f(T) } \Bigl\} \right) 2 g_k W_k  \Psi d^{1/2} \left(\sqrt{ d \log( T ) + 2 \log(1/\delta_T)  } + \Psi \right) \sqrt{  2  \log(T) } + \nonumber \\
  & \sum_{ k \in \mathcal V \backslash N^+ } T  \frac{2 g_k \Psi  \left(\sqrt{ \Psi \log(T) + 2 \log(1/\delta_T) } + \Psi \right) }{ \sqrt{ f(T) } }  W^{2/3}
\end{align}

We will now study the last term in the regret expression. 
If we choose $\delta_t = 1/t^2$, we can set $f(t) = t^{2/3} \log(t)^{1/3}$ and $\eta_k = W_k^{2/3}$, we have $\sqrt{\eta_k f(T) } + \min \Bigl\{ \sqrt{T}, \sqrt{\eta_k f(T) } \Bigl\} \in O( \sqrt{\eta_k f(T) } )$

\begin{align}
   \sum_{ k \in \mathcal V \backslash N^+ } T  \frac{2 g_k \Psi  \left(\sqrt{ d \log(T) + 2 \log(1/\delta_T) } + \Psi \right) }{ \sqrt{ f(T) } }  W^{2/3} & = \sum_{ k \in \mathcal V \backslash N^+ } T  \frac{2 g_k \Psi \left(\sqrt{ (d+4) \log(T) } + \Psi \right) }{ \sqrt{ f(T) } }  W^{2/3} \\
                                    & \in O(  2 g_k \Psi d^{1/2} T  \frac{ \sqrt{ \log(T) }  }{ \sqrt{ f(T) } }  W^{2/3} ) \\
                                    & \in O(  2 g_k \Psi d^{1/2} T^{2/3} \frac{ \sqrt{ \log(T) }   }{ \sqrt{  \log(T)^{1/3} } }  W^{2/3} ) \\
                                    & \in O( 2 g_k \Psi d^{1/2} T^{2/3}   \log(T)^{1/2 - 1/6} W^{2/3}  ) \\
                                    & \in O( 2 g_k \Psi d^{1/2} T^{2/3} \log(T)^{1/3} W^{2/3} ) 
\end{align}

We will now study the penultimate term in the regret expression. If we choose $\delta_t = 1/t^2$, we can set $f(t) = t^{2/3} \log(t)^{1/3}$ and $\eta_k = W_k^{2/3}$, we have:

\begin{align}
    \sqrt{f(T) \eta_k} \times (\dots) & = W_k^{1/3} T^{1/3} \log(T)^{1/6} 2 g_k W_k  \Psi d^{1/2} \left( \sqrt{ d \log( T ) + 2 \log(1/\delta_T)  } + \Psi \right) \sqrt{  2  \log(T) } \\
    & = 2 g_k W_k^{3/2} \Psi d^{1/2} T^{1/3} \log(T)^{1/6} \left( \sqrt{ d \log( T ) + 2 \log(1/\delta_T)  } + \Psi \right) \\
                                    & \in O( 2 g_k W_k^{3/2} \Psi d T^{1/3} \log(T)^{2/3}  ) 
\end{align}

The conclusion is that the last term dominates the penultimate term over time. Therefore, we can conclude:

\begin{corollary}
Consider a hard game and $\delta_t = 1/t^2$, and the same assumptions as in theorem \ref{thm:contextual_result}. Then, there exist constants $C_3$ and $C_4$ such that the expected regret of \RandCBPside{} on this game can be upper bounded independently of the choice of $p^\star$ as:
$$ \mathbb{E}[ R_T ]\leq C_3 N + C_4  \sqrt{d} \log(T)^{1/3} T^{2/3} $$
\end{corollary}

\section{Additional Experiments}
\label{appendix:additional_exps}
\subsection{Implementation details and hyper-parameters}

Contextual and non-contextual experiments are run on machines with $48$ CPUs which justifies why we consider $96$ runs rather than $100$ ($48\times2 = 96$ is the optimal allocation).

\paragraph{Non-contextual baselines}

The stochastic strategies \BPMLeast{}, \TSPM{} and\TSPMgaussian{} are initialized with priors $p^\star = [1/M, \dots, 1/M]$ as this is the common choice reported in their respective original papers \citep{vanchinathan2014efficient,tsuchiya2020analysis}.
The number of samples for \BPMLeast{}, \TSPM{} and \TSPMgaussian{} is set to $100$. We found that higher values increase drastically the computational complexity of the approaches.
The strategies \TSPM{} and \TSPMgaussian{} are set with $\lambda=0.01$ as reported to be the most competitive value in the original paper \citep{tsuchiya2020analysis}. The deterministic strategy \PMDMED{} is initialized with $c=1$ following the value presented in the original paper \citep{komiyama2015regret}.

To compare \CBP{} and \RandCBP{} fairly, both strategies are set with $\alpha = 1.01$. Sampling in \RandCBP{} is performed according to the procedure described in Section \ref{subsec:randcbp} over $K=5$ bins, with probability $\varepsilon=10^{-7}$ on the tail and standard deviation $\sigma=1$. Although this choice is not necessarily the most optimal (see Figures \ref{fig:bench_without_side_info} and \ref{fig:bench_with_linear_side_info}), we find it is the most robust across the different settings considered.

\paragraph{Contextual baselines}

We run \PGTS{} and \PGIDS{} over a horizon $T=7.5$k because both strategies scale in cubic time with the number of verifications. For a horizon $20$k, on a time budget of $5$ hours and a 48-cores machine, less than $10$ realizations succeed out of the $96$ considered.
Note that \PGTS{} and \PGIDS{} assume a logistic setting while in our experiments we consider a linear setting.
The logistic regression still performs well because we consider binary outcome games.
For both strategies, we consider $10$ Gibbs samples: higher values increase the computational complexity of the approaches.
\STAP{} and \CESA{} are hyper-parameter free. 

We compare \CBPsidestar{} to its counterpart \RandCBPsidestar{} fairly by setting both with $\alpha = 1.01$. 
Sampling in \RandCBPsidestar{} is performed according to the randomization procedure described in Section~\ref{subsec:randcbp} with $K=5$ bins, a probability $\varepsilon=10^{-7}$ on the tail, and standard deviation $\sigma=1$. Although this choice is not always the most optimal (see Figures \ref{fig:bench_without_side_info} and \ref{fig:bench_with_linear_side_info}), we find it is the most robust across the various settings considered.

All contextual approaches use a regularization $\lambda = 0.05$.

\subsection{Detailed results}

Table \ref{tab:numericdetail_noncontextual_AT} and \ref{tab:numericdetail_noncontextual_LE} provide numeric details to support the non-contextual experiments in the main paper.
Table \ref{tab:numericdetail_contextual} provides numeric details for the contextual experiment presented in the main paper.

\begin{table}
\centering
\resizebox{\textwidth}{!}{%
\begin{tabular}{|c|cccccccc|}
\hline
Game &
  \multicolumn{8}{c|}{Apple Tasting (AT)} \\ \hline
Case &
  \multicolumn{4}{c|}{imbalanced} &
  \multicolumn{4}{c|}{balanced} \\ \hline
Metric &
  \multicolumn{1}{c|}{mean} &
  \multicolumn{1}{c|}{std} &
  \multicolumn{1}{c|}{pvalue} &
  \multicolumn{1}{c|}{win count} &
  \multicolumn{1}{c|}{mean} &
  \multicolumn{1}{c|}{std} &
  \multicolumn{1}{c|}{pvalue} &
  win count \\ \hline
RandCBP &
  \multicolumn{1}{c|}{\cellcolor[HTML]{9AFF99}4.689} &
  \multicolumn{1}{c|}{4.07} &
  \multicolumn{1}{c|}{1.0} &
  \multicolumn{1}{c|}{\cellcolor[HTML]{9AFF99}82} &
  \multicolumn{1}{c|}{\cellcolor[HTML]{9AFF99}41.417} &
  \multicolumn{1}{c|}{78.311} &
  \multicolumn{1}{c|}{1.0} &
  \cellcolor[HTML]{FFFC9E}37 \\ \hline
CBP &
  \multicolumn{1}{c|}{\cellcolor[HTML]{FFFC9E}8.672} &
  \multicolumn{1}{c|}{8.532} &
  \multicolumn{1}{c|}{0.0} &
  \multicolumn{1}{c|}{\cellcolor[HTML]{FFFC9E}47} &
  \multicolumn{1}{c|}{\cellcolor[HTML]{FFCE93}72.748} &
  \multicolumn{1}{c|}{138.279} &
  \multicolumn{1}{c|}{0.055} &
  \cellcolor[HTML]{9AFF99}48 \\ \hline
PM-DMED &
  \multicolumn{1}{c|}{13.915} &
  \multicolumn{1}{c|}{13.155} &
  \multicolumn{1}{c|}{0.0} &
  \multicolumn{1}{c|}{2} &
  \multicolumn{1}{c|}{113.5} &
  \multicolumn{1}{c|}{138.047} &
  \multicolumn{1}{c|}{0.0} &
  3 \\ \hline
TSPM &
  \multicolumn{1}{c|}{\cellcolor[HTML]{FFCE93}9.359} &
  \multicolumn{1}{c|}{9.007} &
  \multicolumn{1}{c|}{0.0} &
  \multicolumn{1}{c|}{\cellcolor[HTML]{FFCE93}25} &
  \multicolumn{1}{c|}{\cellcolor[HTML]{FFFC9E}43.117} &
  \multicolumn{1}{c|}{45.53} &
  \multicolumn{1}{c|}{0.854} &
  \cellcolor[HTML]{FFCE93}13 \\ \hline
TSPM Gaussian &
  \multicolumn{1}{c|}{19.417} &
  \multicolumn{1}{c|}{15.925} &
  \multicolumn{1}{c|}{0.0} &
  \multicolumn{1}{c|}{7} &
  \multicolumn{1}{c|}{67.658} &
  \multicolumn{1}{c|}{56.203} &
  \multicolumn{1}{c|}{0.008} &
  3 \\ \hline
BPM-Least &
  \multicolumn{1}{c|}{15.125} &
  \multicolumn{1}{c|}{8.063} &
  \multicolumn{1}{c|}{0.0} &
  \multicolumn{1}{c|}{0} &
  \multicolumn{1}{c|}{165.04} &
  \multicolumn{1}{c|}{111.969} &
  \multicolumn{1}{c|}{0.0} &
  3 \\ \hline
\end{tabular}
}
\caption{Supplement for the non-contextual experiment presented in the main paper (see Figure \ref{fig:without_side_info}). 
Imbalanced instances: $p \sim \mathcal U_{[0, 0.2] \cup [0.8, 1]}$. Balanced instances: $p \sim \mathcal U_{[0.4,0.6]}$. Mean: average regret at the last step ($T = 20$k). Std: standard deviation at the last step. P-value: Welch's t-test on the distribution of regrets at the last step, with \RandCBP{} as reference (p-value $>0.05$ means no statistical difference). Win count: number of times a given strategy achieved the lowest final regret (ties included). 
Color \raisebox{0.75ex}{ \tikz[baseline=(char.base)]{\node[shape=circle, fill=green, inner sep=3pt] (char) {};} } indicates the best; \raisebox{0.75ex}{ \tikz[baseline=(char.base)]{\node[shape=circle, fill=yellow, inner sep=3pt] (char) {};} } indicates second best; \raisebox{0.75ex}{ \tikz[baseline=(char.base)]{\node[shape=circle, fill=orange, inner sep=3pt] (char) {};} } is the third best. }
\label{tab:numericdetail_noncontextual_AT}
\end{table}

\begin{table}
\centering
\resizebox{\textwidth}{!}{%
\begin{tabular}{|c|cccccccc|}
\hline
Game & \multicolumn{8}{c|}{Label Efficient (LE)}                       \\ \hline
Case & \multicolumn{4}{c|}{imbalanced} & \multicolumn{4}{c|}{balanced} \\ \hline
Metric &
  \multicolumn{1}{c|}{mean} &
  \multicolumn{1}{c|}{std} &
  \multicolumn{1}{c|}{pvalue} &
  \multicolumn{1}{c|}{win count} &
  \multicolumn{1}{c|}{mean} &
  \multicolumn{1}{c|}{std} &
  \multicolumn{1}{c|}{pvalue} &
  win count \\ \hline
RandCBP &
  \multicolumn{1}{c|}{\cellcolor[HTML]{9AFF99}11.887} &
  \multicolumn{1}{c|}{15.004} &
  \multicolumn{1}{c|}{1.0} &
  \multicolumn{1}{c|}{\cellcolor[HTML]{9AFF99}81.0} &
  \multicolumn{1}{c|}{\cellcolor[HTML]{9AFF99}321.023} &
  \multicolumn{1}{c|}{353.111} &
  \multicolumn{1}{c|}{1.0} &
  \cellcolor[HTML]{9AFF99}60.0 \\ \hline
CBP &
  \multicolumn{1}{c|}{\cellcolor[HTML]{FFFC9E}16.47} &
  \multicolumn{1}{c|}{8.173} &
  \multicolumn{1}{c|}{0.009} &
  \multicolumn{1}{c|}{\cellcolor[HTML]{FFFC9E}15.0} &
  \multicolumn{1}{c|}{\cellcolor[HTML]{FFFC9E}726.877} &
  \multicolumn{1}{c|}{643.233} &
  \multicolumn{1}{c|}{0.0} &
  \cellcolor[HTML]{FFFC9E}18.0 \\ \hline
PM-DMED &
  \multicolumn{1}{c|}{\cellcolor[HTML]{FFCE93}1489.217} &
  \multicolumn{1}{c|}{2887.675} &
  \multicolumn{1}{c|}{0.0} &
  \multicolumn{1}{c|}{\cellcolor[HTML]{FFCE93}0.0} &
  \multicolumn{1}{c|}{\cellcolor[HTML]{FFCE93}1253.432} &
  \multicolumn{1}{c|}{1048.542} &
  \multicolumn{1}{c|}{0.0} &
  \cellcolor[HTML]{FFFC9E}18.0 \\ \hline
\end{tabular}
}
\caption{Supplement for the non-contextual experiment presented in the main paper (see Figure \ref{fig:without_side_info}). 
Imbalanced instances: $p \sim \mathcal U_{[0, 0.2] \cup [0.8, 1]}$. Balanced instances: $p \sim \mathcal U_{[0.4,0.6]}$. Mean: average regret at the last step ($T = 20$k). Std: standard deviation at the last step. P-value: Welch's t-test on the distribution of regrets at the last step, with \RandCBP{} as reference (p-value $>0.05$ means no statistical difference). Win count: number of times a given strategy achieved the lowest final regret (ties included). 
Color \raisebox{0.75ex}{ \tikz[baseline=(char.base)]{\node[shape=circle, fill=green, inner sep=3pt] (char) {};} } indicates the best; \raisebox{0.75ex}{ \tikz[baseline=(char.base)]{\node[shape=circle, fill=yellow, inner sep=3pt] (char) {};} } indicates second best; \raisebox{0.75ex}{ \tikz[baseline=(char.base)]{\node[shape=circle, fill=orange, inner sep=3pt] (char) {};} } is the third best. }
\label{tab:numericdetail_noncontextual_LE}
\end{table}

\begin{table}
\centering
\resizebox{\textwidth}{!}{%
\begin{tabular}{|c|cccc|cccc|}
\hline
Game &
  \multicolumn{4}{c|}{Apple Tasting (AT)} &
  \multicolumn{4}{c|}{Label Efficient (LE)} \\ \hline
Metric &
  \multicolumn{1}{c|}{mean} &
  \multicolumn{1}{c|}{std} &
  \multicolumn{1}{c|}{pvalue} &
  win count &
  \multicolumn{1}{c|}{mean} &
  \multicolumn{1}{c|}{std} &
  \multicolumn{1}{c|}{pvalue} &
  win count \\ \hline
RandCBPside &
  \multicolumn{1}{c|}{\cellcolor[HTML]{9AFF99}1016.312} &
  \multicolumn{1}{c|}{82.151} &
  \multicolumn{1}{c|}{1.0} &
  \cellcolor[HTML]{9AFF99}96 &
  \multicolumn{1}{c|}{\cellcolor[HTML]{9AFF99}2026.604} &
  \multicolumn{1}{c|}{70.161} &
  \multicolumn{1}{c|}{1.0} &
  \cellcolor[HTML]{9AFF99}96.0 \\ \hline
CBPside &
  \multicolumn{1}{c|}{\cellcolor[HTML]{FFCE93}6109.521} &
  \multicolumn{1}{c|}{86.325} &
  \multicolumn{1}{c|}{0.0} &
  \cellcolor[HTML]{FFCE93}0 &
  \multicolumn{1}{c|}{\cellcolor[HTML]{FFCE93}11071.333} &
  \multicolumn{1}{c|}{86.779} &
  \multicolumn{1}{c|}{0.0} &
  \cellcolor[HTML]{FFCE93}0 \\ \hline
\cellcolor[HTML]{C0C0C0}PGIDSratio &
  \multicolumn{1}{c|}{\cellcolor[HTML]{C0C0C0}129.5} &
  \multicolumn{1}{c|}{\cellcolor[HTML]{C0C0C0}12.758} &
  \multicolumn{1}{c|}{\cellcolor[HTML]{C0C0C0}0.0} &
  \cellcolor[HTML]{C0C0C0}0 &
  \multicolumn{4}{c|}{} \\ \hline
\cellcolor[HTML]{C0C0C0}PGTS &
  \multicolumn{1}{c|}{\cellcolor[HTML]{C0C0C0}179.156} &
  \multicolumn{1}{c|}{\cellcolor[HTML]{C0C0C0}15.318} &
  \multicolumn{1}{c|}{\cellcolor[HTML]{C0C0C0}0.0} &
  \cellcolor[HTML]{C0C0C0}0 &
  \multicolumn{4}{c|}{} \\ \hline
STAP &
  \multicolumn{1}{c|}{\cellcolor[HTML]{FFFC9E}1565.917} &
  \multicolumn{1}{c|}{127.488} &
  \multicolumn{1}{c|}{0.0} &
  \cellcolor[HTML]{FFFC9E}0 &
  \multicolumn{4}{c|}{} \\ \hline
CESA &
  \multicolumn{4}{c|}{} &
  \multicolumn{1}{c|}{\cellcolor[HTML]{FFFC9E}5792.052} &
  \multicolumn{1}{c|}{1386.179} &
  \multicolumn{1}{c|}{0.0} &
  \cellcolor[HTML]{FFFC9E}0 \\ \hline
\end{tabular}%
}
\caption{Numeric detail of the contextual experiment presented in the main paper (see Figure \ref{fig:with_linear_info}).  
Mean: average regret at the last step ($T = 20$k). Std: standard deviation at the last step. P-value: Welch's t-test on the distribution of regrets at the last step with RandCBPside as reference (p-value$>0.05$ means no statistical difference). Win count: number of times a given strategy achieved the lowest final regret (ties included). 
Color \raisebox{0.75ex}{ \tikz[baseline=(char.base)]{\node[shape=circle, fill=green, inner sep=3pt] (char) {};} } indicates the best; \raisebox{0.75ex}{ \tikz[baseline=(char.base)]{\node[shape=circle, fill=yellow, inner sep=3pt] (char) {};} } indicates second best; \raisebox{0.75ex}{ \tikz[baseline=(char.base)]{\node[shape=circle, fill=orange, inner sep=3pt] (char) {};} } is the third best. Color \raisebox{0.75ex}{ \tikz[baseline=(char.base)]{\node[shape=circle, fill=gray, inner sep=3pt] (char) {};} } indicates an evaluation on the truncated horizon $T=7.5$k. }
\label{tab:numericdetail_contextual}
\end{table}

\subsection{Sensitivity to hyper-parameters}

The goal of this experiment is to illustrate the sensitivity to hyper-parameters of \RandCBP{} and \RandCBPsidestar{}. 
We conducted the evaluation for $\varepsilon = 10^{-7}$. Higher values of $\epsilon$ imply a higher probability of sampling on the value $B$ in the discretized interval $[A,B]$.
We consider the standard deviation values $\sigma \in \{ 1/2, 1, 2, 10\}$.
We consider bin values of $K$ in $\{ 5, 10, 20 \}$. 
We report averaged regret and upper $99\%$ confidence interval, measured over a $20$k horizon and $96$ random runs. 

\paragraph{Results (non-contextual case):} The experimental setting is described in the main paper. We find the hyper-parameter $\sigma$ to be the most influential in the performance of \RandCBP{}. Too small values of $\sigma$ result in a more exploitation, and expose the strategy to a risk of catastrophic failures. 

\begin{figure}[H]
     \centering
     \begin{subfigure}[b]{0.23\textwidth}
         \centering
         \includegraphics[width=\textwidth]{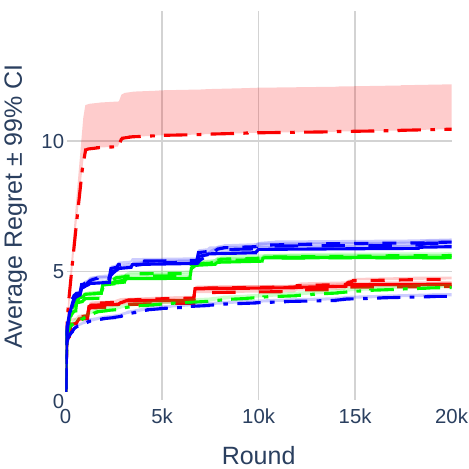}
         \caption{Imbalanced AT}
         \label{fig:bench_easyFF}
     \end{subfigure}
     \hfill
     \begin{subfigure}[b]{0.23\textwidth}
         \centering
         \includegraphics[width=\textwidth]{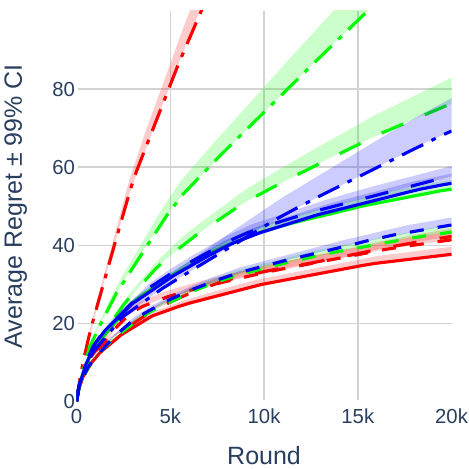}
         \caption{Balanced AT}
         \label{fig:bench_hardAT}
     \end{subfigure}
     \hfill
     \begin{subfigure}[b]{0.23\textwidth}
         \centering
         \includegraphics[width=\textwidth]{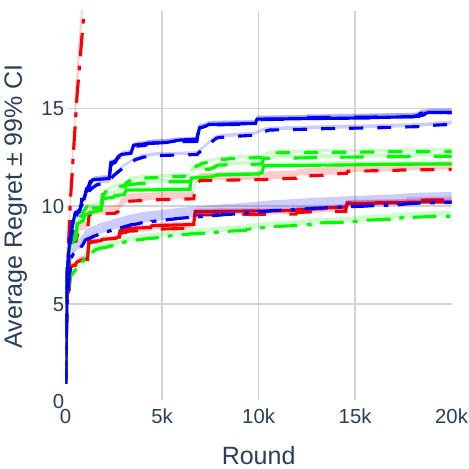}
         \caption{Imbalanced LE}
         \label{fig:bench_easyAT}
     \end{subfigure}
     \hfill
     \begin{subfigure}[b]{0.23\textwidth}
         \centering
         \includegraphics[width=\textwidth]{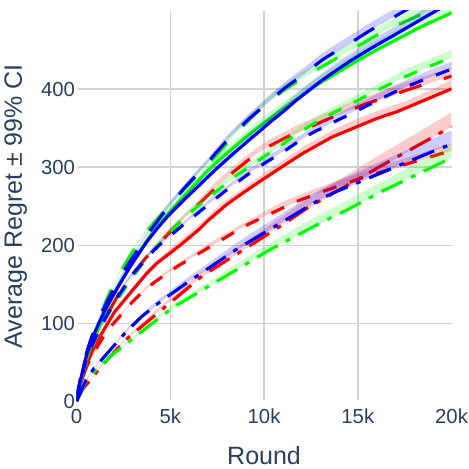}
         \caption{Balanced LE}
         \label{fig:bench_hardLE}
     \end{subfigure}
        \hfill
     \begin{subfigure}[b]{0.33\textwidth}
         \centering
         \includegraphics[width=\textwidth]{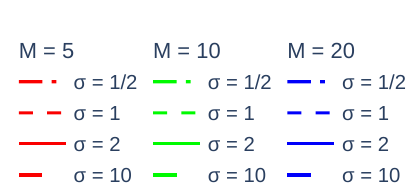}
         \caption{Legend}
         \label{fig:bench_easyLE}
     \end{subfigure}
    \caption{Benchmark of \RandCBP{} on Apple Tasting (AT) and Label Efficient (LE) games, non-contextual case.}
    \label{fig:bench_without_side_info}
\end{figure}

\paragraph{Results (contextual case):}

The experimental setting is described in the main paper. 
Figure \ref{fig:bench_with_linear_side_info} reports multiple hyper-parameter combinations over the Apple Tasting (AT) and Label Efficient (LE) games on linear contexts. 
Similarly to the non-contextual case, the hyper-parameter $\sigma$ greatly influences the performance of in \RandCBPsidestar{}, as suggested in Figures \ref{fig:bench_balanced_linear_AT} and \ref{fig:bench_balanced_linear_LE}.

\begin{figure}[H]
    \centering
    \begin{subfigure}[b]{0.23\textwidth}
        \centering
        \includegraphics[width=\textwidth]{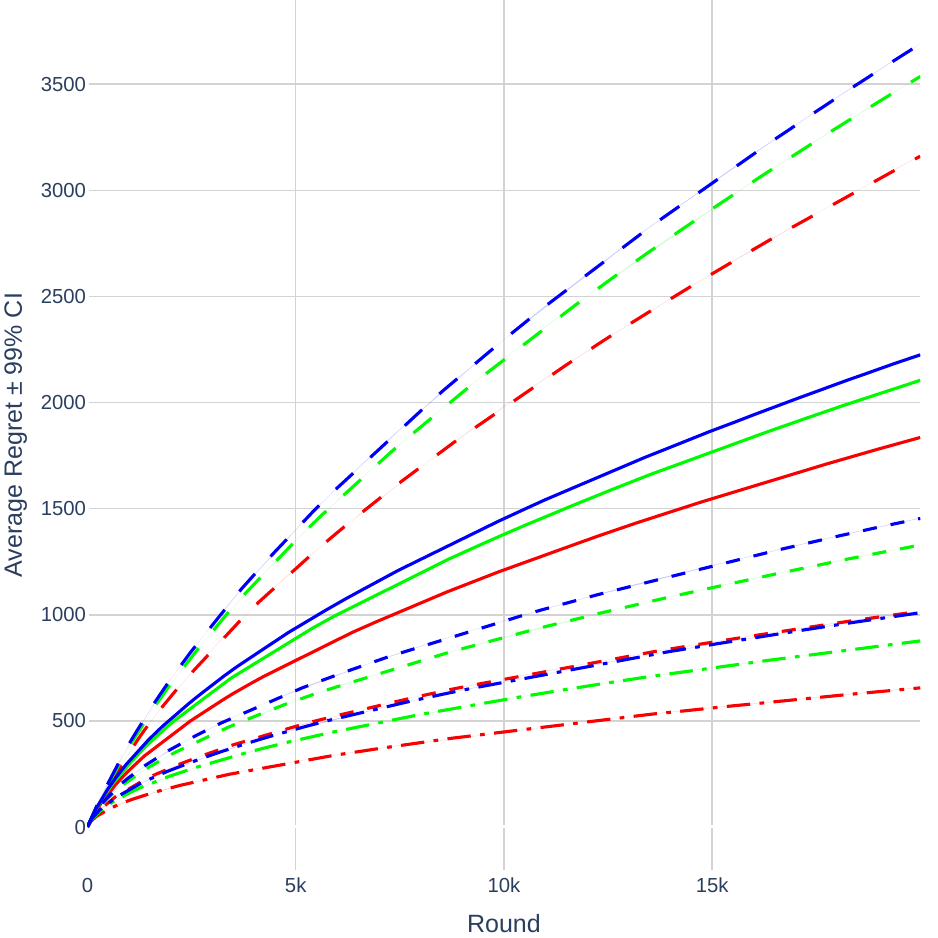}
        \caption{Balanced AT}
        \label{fig:bench_balanced_linear_AT}
    \end{subfigure}
    \hspace*{0.02\textwidth} 
    \begin{subfigure}[b]{0.23\textwidth}
        \centering
        \includegraphics[width=\textwidth]{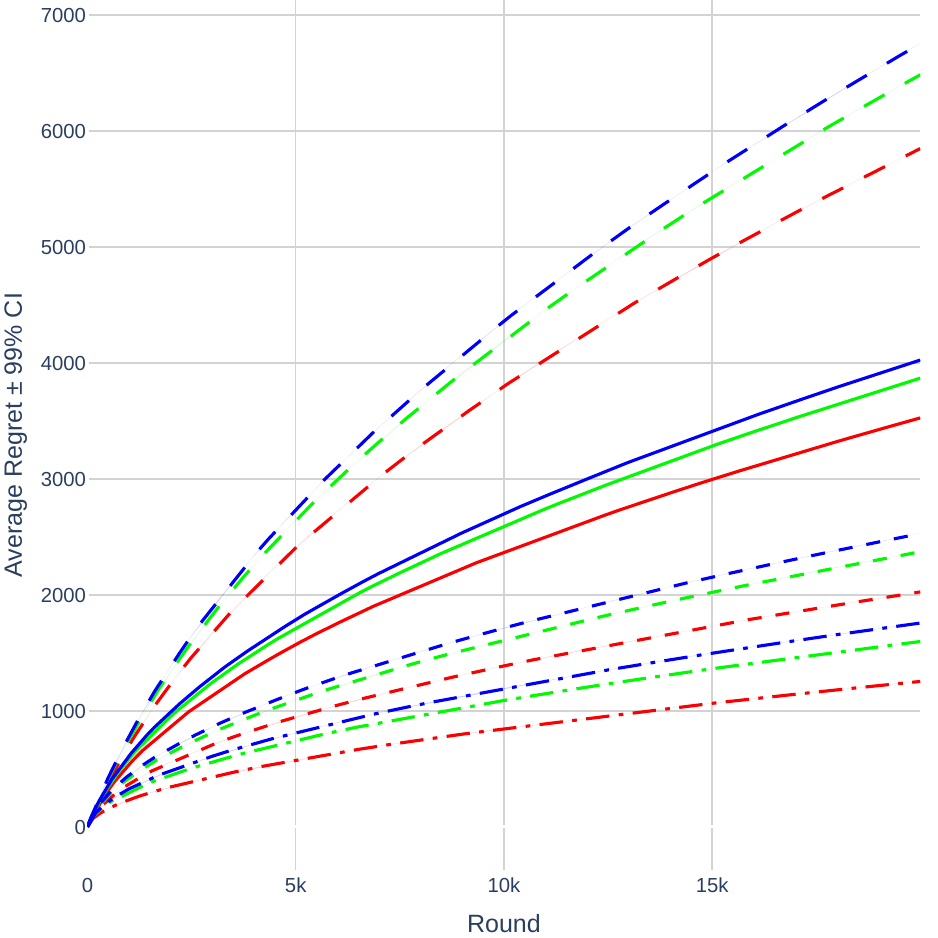}
        \caption{Balanced LE}
        \label{fig:bench_balanced_linear_LE}
    \end{subfigure}
    \hspace*{0.02\textwidth} 
    \begin{subfigure}[b]{0.23\textwidth}
        \centering
        \includegraphics[width=\textwidth]{figures/legend_benchmark_nocontext.pdf}
        \caption{Legend}
        \label{fig:legend_benchmark_nocontext}
    \end{subfigure}
    \caption{Benchmark of \RandCBPside{} on Apple Tasting (AT) and Label Efficient (LE) games, contextual case.}
    \label{fig:bench_with_linear_side_info}
\end{figure}

\subsection{Additional results on the case-study}
\label{appendix:case_study}

In this Section, we report additional results for the use-case presented in Section \ref{sec:case_study}. 
The distribution of true classes in the stream of observations is represented by the vector $B \in [0,1]^N$ and the black-box classifier is represented by its confusion matrix $\textbf{C} \in [0,1]^{N\times N}$. 
In each run, $B$ and $\textbf{C}$ are generated randomly such that the global error rate remains below $10\%$ (the black-box would probably not be deployed otherwise).
Each instance of \RandCBP{} and \CBP{} in the approaches \ccbp{} and \crandcbp{} is parameterized similarly as in previous experiments. 

\paragraph{Maximum verification number of verifications}
Since the outcomes are binary, the Wald's confidence interval can be used to determine the maximum number of verifications needed to obtain estimates of  $p_c$ with a specified level of confidence.
We set probability that the confidence interval fails is set to $\zeta = 0.01$ and the acceptable margin length of the Wald interval to $E = \tau / 10$.
Assuming that the classifier is deployed with a global error rate of at most $10\%$, a reasonable prior belief per class (noted $\bar p_c$) is that the error rate is distributed uniformly across classes $\bar p_c=10/C\%$. 
For a detection threshold $\tau$, the maximum verification budget for a class $c$ is 
$n_{\tau} = \frac{ z(1-\zeta/2)^2   \bar p (1-\bar p) }{ E^2 } $, where $z(\cdot)$ is the quantile of the standard normal distribution.
In practice, the value $C \times n_{\tau}$ corresponds to the maximum number of verifications  one is willing to use to identify with confidence which of the $C$ predicted classes errors exceed the threshold $\tau$. 
The goal is to obtain a strategy that performs the task while consuming less verifications than this maximum amount.

\paragraph{Results}
In the main paper, we reported results for cases:
i) the true classes are balanced and the black-box yields uniform mispredictions (case 1),
ii) the true classes are imbalanced and the black-box yields non-uniform mispredictions (case 2).

Results for the two cases are reported in Tables \ref{tab:case1} and  \ref{tab:case2}. 
In all the considered cases, the mean and median F1-score performance of \crandcbp{}, \ccbp{} and \explorecommit{} is very comparable, as the mean F1-score values overlap when considering the standard-errors. 
Although all strategies achieve equivalent F1-score performance, \crandcbp{} consumes on average less verifications than \explorecommit{}. 
The variance on the number of verifications is the highest when the stream of observations is balanced and the mispredictions are uniforms (case 1, Table \ref{tab:case1}).

\begin{table}[H]
\centering
\resizebox{\textwidth}{!}{%
\begin{tabular}{|c|c|c|c|c|c|c|c|}
\hline
Threshold &
  Strategy &
  F1-score (mean) &
  F1-score (median) &
  F1-score (std) &
  Nb. verifs (mean) &
  Nb. verifs (median) &
  Nb. verifs (std) \\ \hline
 &
  \cellcolor[HTML]{C0C0C0}Full-exploration &
  \cellcolor[HTML]{C0C0C0}0.962 &
  \cellcolor[HTML]{C0C0C0}1.0 &
  \cellcolor[HTML]{C0C0C0}0.15 &
  \cellcolor[HTML]{C0C0C0}105120.0 &
  \cellcolor[HTML]{C0C0C0}105120.0 &
  \cellcolor[HTML]{C0C0C0}0.0 \\ \cline{2-8} 
 &
  N-CBP &
  0.962 &
  1.0 &
  0.15 &
  105110.0 &
  105110.0 &
  0.0 \\ \cline{2-8} 
\multirow{-3}{*}{0.025} &
  \cellcolor[HTML]{C0C0C0}C-RandCBP &
  \cellcolor[HTML]{C0C0C0}0.955 &
  \cellcolor[HTML]{C0C0C0}1.0 &
  \cellcolor[HTML]{C0C0C0}0.163 &
  \cellcolor[HTML]{C0C0C0}83818.0 &
  \cellcolor[HTML]{C0C0C0}104846.0 &
  \cellcolor[HTML]{C0C0C0}32297.0 \\ \hline
 &
  Full-exploration &
  0.927 &
  1.0 &
  0.195 &
  26300.0 &
  26300.0 &
  0.0 \\ \cline{2-8} 
 &
  \cellcolor[HTML]{C0C0C0}C-CBP &
  \cellcolor[HTML]{C0C0C0}0.927 &
  \cellcolor[HTML]{C0C0C0}1.0 &
  \cellcolor[HTML]{C0C0C0}0.195 &
  \cellcolor[HTML]{C0C0C0}26290.0 &
  \cellcolor[HTML]{C0C0C0}26290.0 &
  \cellcolor[HTML]{C0C0C0}0.0 \\ \cline{2-8} 
\multirow{-3}{*}{0.05} &
  C-RandCBP &
  0.915 &
  1.0 &
  0.208 &
  15976.0 &
  15091.0 &
  9071.0 \\ \hline
 &
  \cellcolor[HTML]{C0C0C0}Full-exploration &
  \cellcolor[HTML]{C0C0C0}0.907 &
  \cellcolor[HTML]{C0C0C0}1.0 &
  \cellcolor[HTML]{C0C0C0}0.219 &
  \cellcolor[HTML]{C0C0C0}6590.0 &
  \cellcolor[HTML]{C0C0C0}6590.0 &
  \cellcolor[HTML]{C0C0C0}0.0 \\ \cline{2-8} 
 &
  C-CBP &
  0.908 &
  1.0 &
  0.216 &
  6491.0 &
  6580.0 &
  251.0 \\ \cline{2-8} 
\multirow{-3}{*}{0.1} &
  \cellcolor[HTML]{C0C0C0}C-RandCBP &
  \cellcolor[HTML]{C0C0C0}0.91 &
  \cellcolor[HTML]{C0C0C0}1.0 &
  \cellcolor[HTML]{C0C0C0}0.211 &
  \cellcolor[HTML]{C0C0C0}2000.0 &
  \cellcolor[HTML]{C0C0C0}1666.0 &
  \cellcolor[HTML]{C0C0C0}1094.0 \\ \hline
 &
  Full-exploration &
  1.0 &
  1.0 &
  0.0 &
  1670.0 &
  1670.0 &
  0.0 \\ \cline{2-8} 
 &
  \cellcolor[HTML]{C0C0C0}C-CBP &
  \cellcolor[HTML]{C0C0C0}1.0 &
  \cellcolor[HTML]{C0C0C0}1.0 &
  \cellcolor[HTML]{C0C0C0}0.0 &
  \cellcolor[HTML]{C0C0C0}1289.0 &
  \cellcolor[HTML]{C0C0C0}1326.0 &
  \cellcolor[HTML]{C0C0C0}273.0 \\ \cline{2-8} 
\multirow{-3}{*}{0.2} &
  C-RandCBP &
  1.0 &
  1.0 &
  0.0 &
  272.0 &
  255.0 &
  70.0 \\ \hline
\end{tabular}%
}
\caption{Case 1 - balanced stream and uniform mispredictions}
\label{tab:case1}
\end{table}

\begin{table}[H]
\centering
\resizebox{\textwidth}{!}{%
\begin{tabular}{|c|c|c|c|c|c|c|c|}
\hline
Threshold &
  Strategy &
  F1-score (mean) &
  F1-score (median) &
  F1-score (std) &
  Nb. verifs (mean) &
  Nb. verifs (median) &
  Nb. verifs (std) \\ \hline
 &
  \cellcolor[HTML]{C0C0C0}Full-exploration &
  \cellcolor[HTML]{C0C0C0}0.978 &
  \cellcolor[HTML]{C0C0C0}1.0 &
  \cellcolor[HTML]{C0C0C0}0.07 &
  \cellcolor[HTML]{C0C0C0}103703.0 &
  \cellcolor[HTML]{C0C0C0}105120.0 &
  \cellcolor[HTML]{C0C0C0}2853.0 \\ \cline{2-8} 
 &
  N-CBP &
  0.978 &
  1.0 &
  0.07 &
  103693.0 &
  105110.0 &
  2853.0 \\ \cline{2-8} 
\multirow{-3}{*}{0.025} &
  \cellcolor[HTML]{C0C0C0}C-RandCBP &
  \cellcolor[HTML]{C0C0C0}0.976 &
  \cellcolor[HTML]{C0C0C0}1.0 &
  \cellcolor[HTML]{C0C0C0}0.07 &
  \cellcolor[HTML]{C0C0C0}81741.0 &
  \cellcolor[HTML]{C0C0C0}81330.0 &
  \cellcolor[HTML]{C0C0C0}10921.0 \\ \hline
 &
  Full-exploration &
  0.976 &
  1.0 &
  0.054 &
  26231.0 &
  26300.0 &
  287.0 \\ \cline{2-8} 
 &
  \cellcolor[HTML]{C0C0C0}C-CBP &
  \cellcolor[HTML]{C0C0C0}0.975 &
  \cellcolor[HTML]{C0C0C0}1.0 &
  \cellcolor[HTML]{C0C0C0}0.054 &
  \cellcolor[HTML]{C0C0C0}26221.0 &
  \cellcolor[HTML]{C0C0C0}26290.0 &
  \cellcolor[HTML]{C0C0C0}287.0 \\ \cline{2-8} 
\multirow{-3}{*}{0.05} &
  C-RandCBP &
  0.965 &
  1.0 &
  0.063 &
  16901.0 &
  16673.0 &
  2580.0 \\ \hline
 &
  \cellcolor[HTML]{C0C0C0}Full-exploration &
  \cellcolor[HTML]{C0C0C0}0.953 &
  \cellcolor[HTML]{C0C0C0}1.0 &
  \cellcolor[HTML]{C0C0C0}0.073 &
  \cellcolor[HTML]{C0C0C0}6590.0 &
  \cellcolor[HTML]{C0C0C0}6590.0 &
  \cellcolor[HTML]{C0C0C0}0.0 \\ \cline{2-8} 
 &
  C-CBP &
  0.953 &
  1.0 &
  0.073 &
  6453.0 &
  6457.0 &
  111.0 \\ \cline{2-8} 
\multirow{-3}{*}{0.1} &
  \cellcolor[HTML]{C0C0C0}C-RandCBP &
  \cellcolor[HTML]{C0C0C0}0.959 &
  \cellcolor[HTML]{C0C0C0}1.0 &
  \cellcolor[HTML]{C0C0C0}0.073 &
  \cellcolor[HTML]{C0C0C0}2965.0 &
  \cellcolor[HTML]{C0C0C0}2961.0 &
  \cellcolor[HTML]{C0C0C0}506.0 \\ \hline
 &
  Full-exploration &
  0.927 &
  1.0 &
  0.156 &
  1670.0 &
  1670.0 &
  0.0 \\ \cline{2-8} 
 &
  \cellcolor[HTML]{C0C0C0}C-CBP &
  \cellcolor[HTML]{C0C0C0}0.927 &
  \cellcolor[HTML]{C0C0C0}1.0 &
  \cellcolor[HTML]{C0C0C0}0.156 &
  \cellcolor[HTML]{C0C0C0}1335.0 &
  \cellcolor[HTML]{C0C0C0}1338.0 &
  \cellcolor[HTML]{C0C0C0}90.0 \\ \cline{2-8} 
\multirow{-3}{*}{0.2} &
  C-RandCBP &
  0.923 &
  1.0 &
  0.163 &
  447.0 &
  436.0 &
  84.0 \\ \hline
\end{tabular}%
}
\caption{Case 2 - Imbalanced stream and non-uniform mispredictions}
\label{tab:case2}
\end{table}

\end{document}